\documentclass[a4paper]{article}
\usepackage{a4wide}

\input{packages.sty}
\input{commands.sty}

\begin{document}

\author{Jannik Irmai$^{1,2}$, Shengxian Zhao$^1$, Jannik Presberger$^1$, Bjoern Andres$^{1,2,*}$}
\title{\textbf{A Graph Multi-separator Problem for Image Segmentation}}
\date{$^1$\textit{TU Dresden} \hspace{2ex} $^2$\textit{Center for Scalable Data Analytics and AI Dresden/Leipzig}}
\maketitle
\footnotetext[1]{Correspondence: \texttt{bjoern.andres@tu-dresden.de}}

\begin{abstract}
    We propose a novel abstraction of the image segmentation task in the form of a combinatorial optimization problem that we call the \emph{multi-separator problem}.
    Feasible solutions indicate for every pixel whether it belongs to a segment or a segment separator, and indicate for pairs of pixels whether or not the pixels belong to the same segment.
    This is in contrast to the closely related lifted multicut problem where every pixel is associated to a segment and no pixel explicitly represents a separating structure.
    While the multi-separator problem is \textsc{np}-hard, we identify two special cases for which it can be solved efficiently.
    Moreover, we define two local search algorithms for the general case and demonstrate their effectiveness in segmenting simulated volume images of foam cells and filaments.
    \\[1ex]
\end{abstract}

\tableofcontents

\section{Introduction}\label{sec:introduction}
Fundamental in the field of image analysis is the task of decomposing an image into distinct objects.
Instances of this task differ with regard to the risk of making specific mistakes:
False cuts are the dominant risk e.g.~for volume images of intrinsically one-dimensional filaments (\Cref{fig:filament-example}).
False joins are the dominant risk e.g.~for volume images of intrinsically three-dimensional foam cells (\Cref{fig:cell-example}).
Mathematical abstractions of the image segmentation task in the form of optimization problems, as well as algorithms for solving these problems, exactly or approximately, typically have parameters for balancing the risk of false cuts and false joins.
Outstanding from these abstractions is the lifted multicut problem (cf.~\Cref{sec:related-work}) in that it treats cuts and joins symmetrically and allows the practitioner to bias solutions toward cuts or joins explicitly.
In the lifted multicut problem, every pixel is associated to an object and no pixel explicitly represents a structure separating these objects.
This is suitable for applications in the field of computer vision where, typically, every pixel can be associated to an object and no pixel explicitly represents a separating structure.

As the first and major contribution of this work, we analyze theoretically a novel abstraction of the image segmentation task in the form of a combinatorial optimization problem that we call the \emph{multi-separator problem}.
Like in the lifted multicut problem, feasible solutions make explicit for pairs of pixels whether these pixels belong to the same or distinct objects, and these two cases are treated symmetrically.
Also like in the lifted multicut problem, the number and size of segments is not constrained by the problem but is instead determined by its solutions.
Unlike for the lifted multicut problem, feasible solutions distinguish between separated and separating pixels in the image.
As the second and minor contribution of this work, we examine empirically the accuracy of the multi-separator problem as a model for segmenting objects that are separated by other objects, specifically, of foam cells or filaments that are separated by foam membranes or void. 
To this end, we define algorithms for finding locally optimal feasible solutions to the multi-separator problem efficiently, apply these to simulated volume images of foams and filaments, and report the accuracy of reconstructed foams or filaments defined by the output, along with absolute computation times.

Throughout this article, we abstract images as graphs whose nodes relate one-to-one to image pixels, whose edges connect adjacent pixels as depicted in \Cref{fig:img-seg-example}, and whose component-inducing node subsets model feasible segments of the image. 
Although these graphs are grid graphs, we define and discuss the multi-separator problem for arbitrary graphs, and we do not exploit the grid structure algorithmically.

The remainder of the article is organized as follows.
In \Cref{sec:related-work}, we discuss related work. 
In \Cref{sec:model}, we define the multi-separator problem, analyze its complexity and establish a connection to the lifted multicut problem.
In \Cref{sec:algorithms}, we define algorithms for finding locally optimal feasible solutions efficiently.
In \Cref{sec:experiments}, we examine the accuracy of the multi-separator problem as a model for segmenting synthetic volume images of simulated foams and filaments.
In \Cref{sec:conclusion}, we draw conclusions and discuss perspectives for future work.

\begin{figure}[t]
    \centering
    \begin{subfigure}[t]{0.19\textwidth}
        \includegraphics[height=5.5cm]{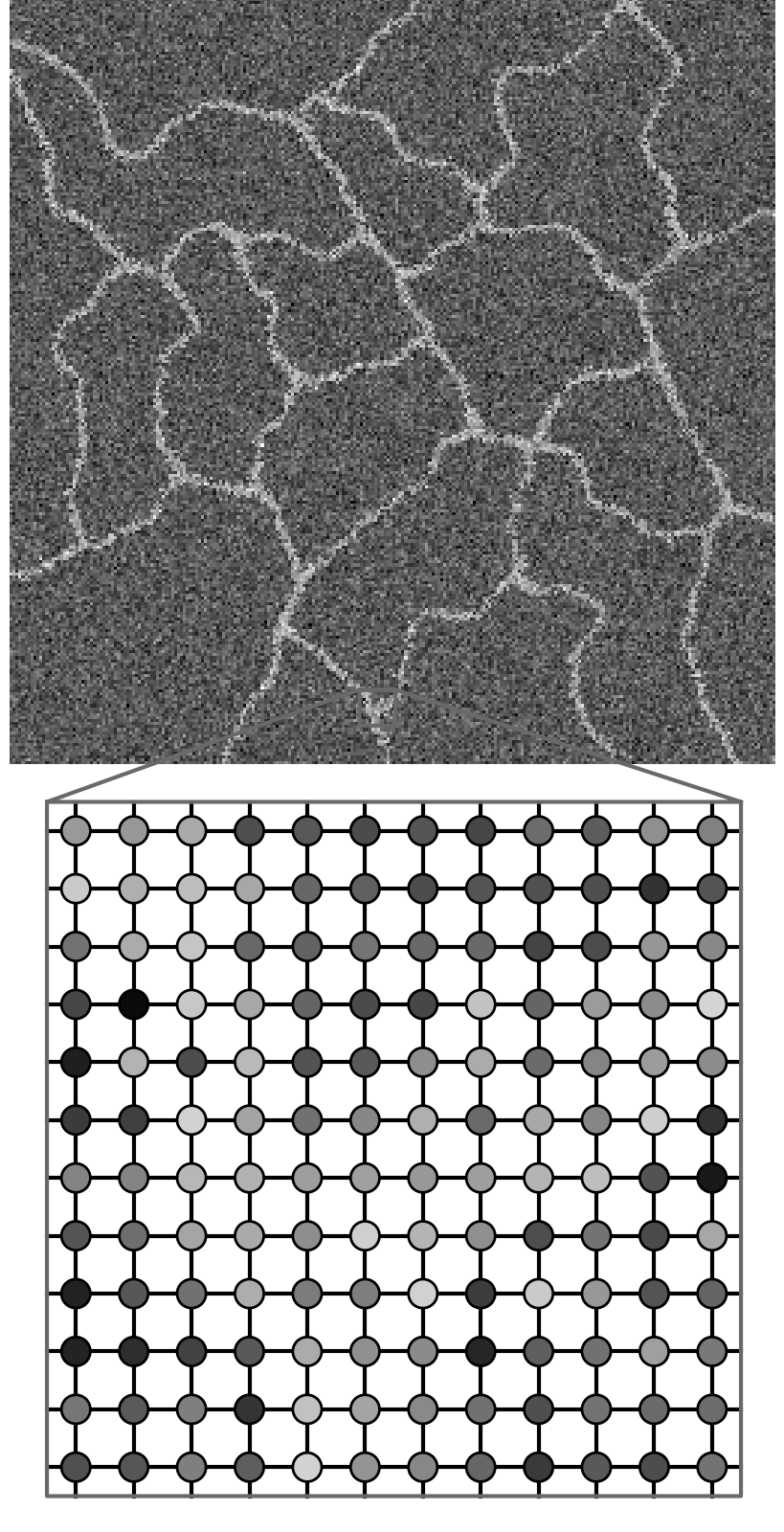}
        \caption{}
        \label{fig:img-seg-example-a}
    \end{subfigure}
    \begin{subfigure}[t]{0.19\textwidth}
        \includegraphics[height=5.5cm]{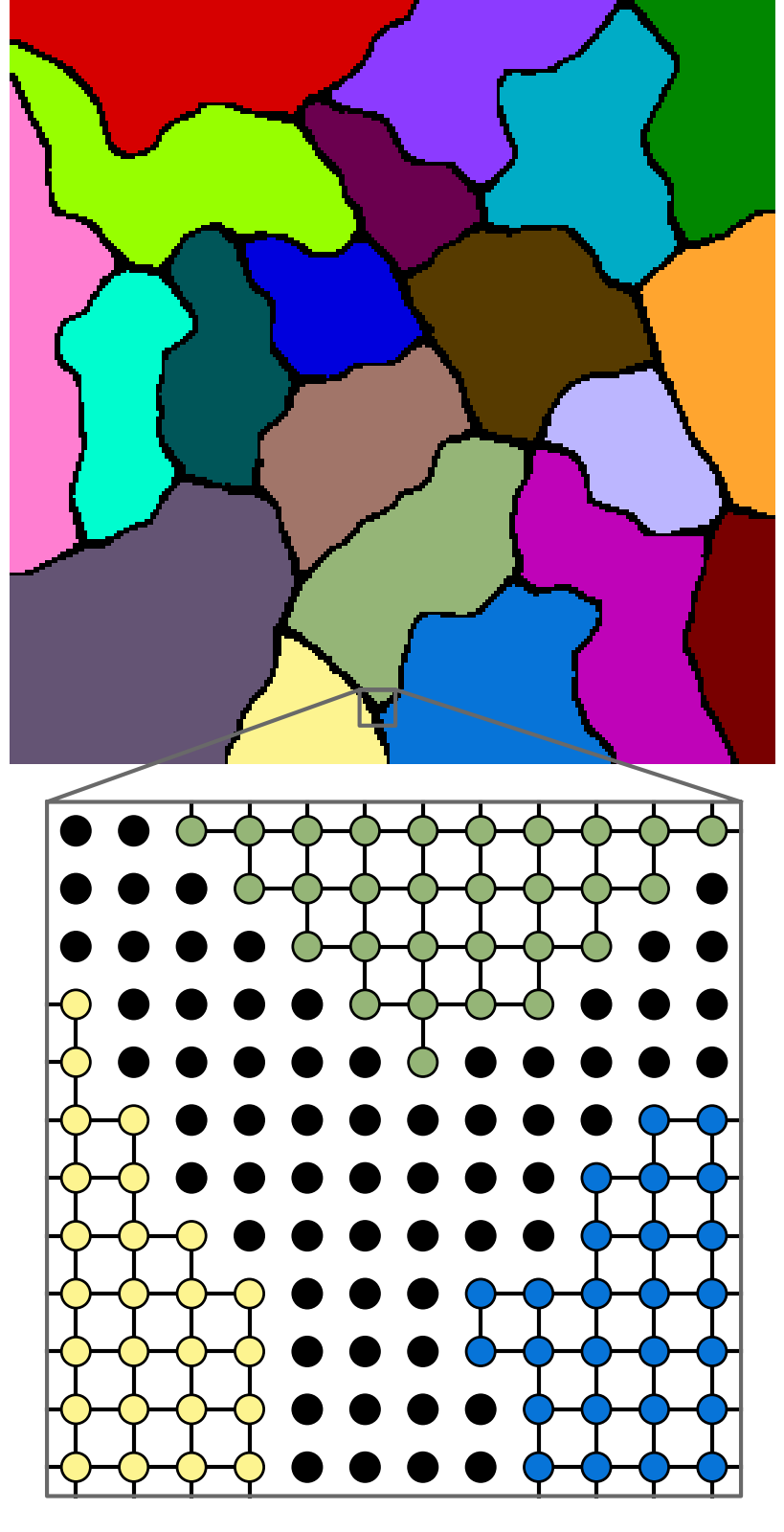}
        \caption{}
        \label{fig:img-seg-example-b}
    \end{subfigure}
    \begin{subfigure}[t]{0.19\textwidth}
        \includegraphics[height=5.5cm]{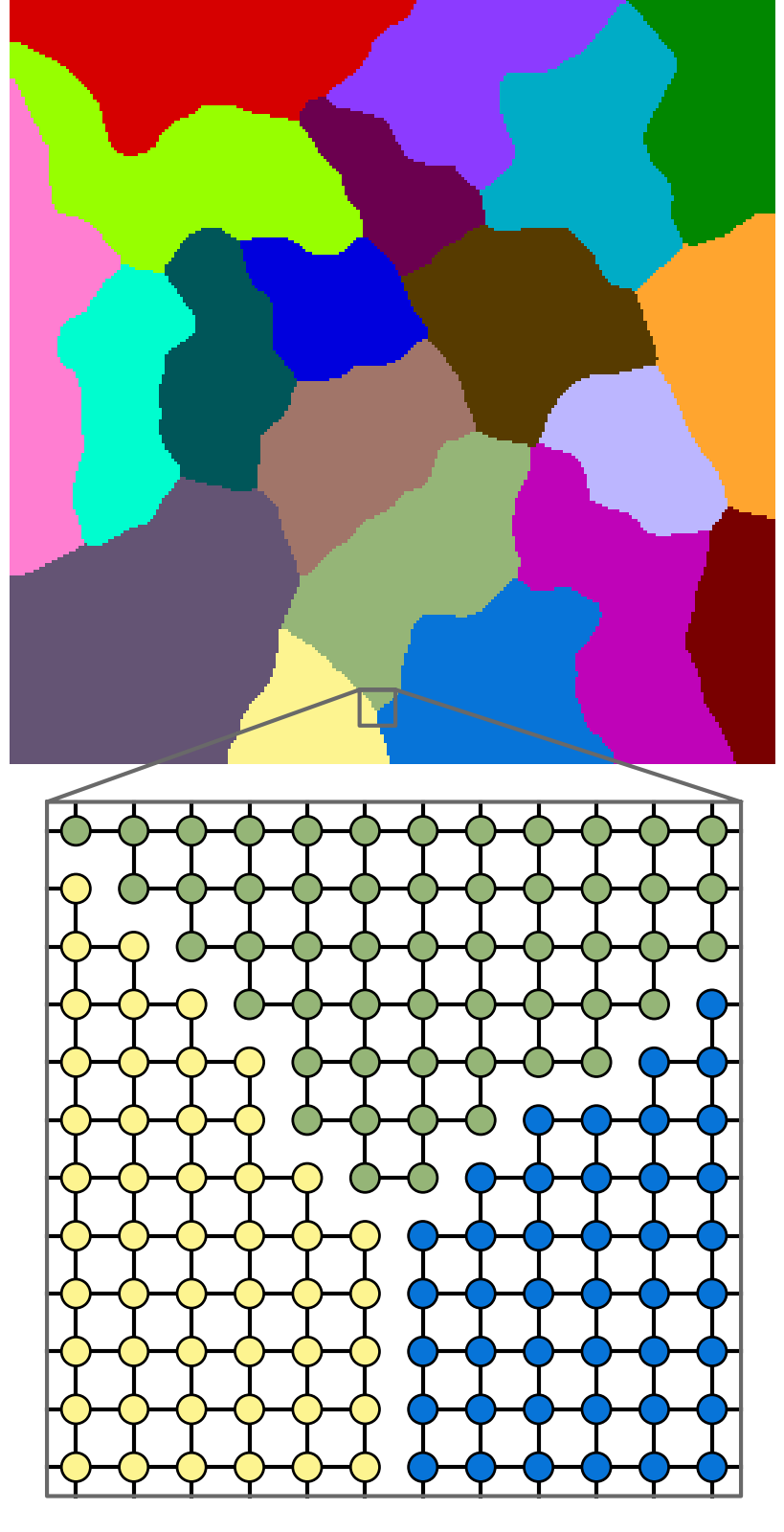}
        \caption{}
        \label{fig:img-seg-example-c}
    \end{subfigure}
    \unskip \hfill \vrule \hfill
    \begin{subfigure}[t]{0.19\textwidth}
        \includegraphics[height=5.5cm]{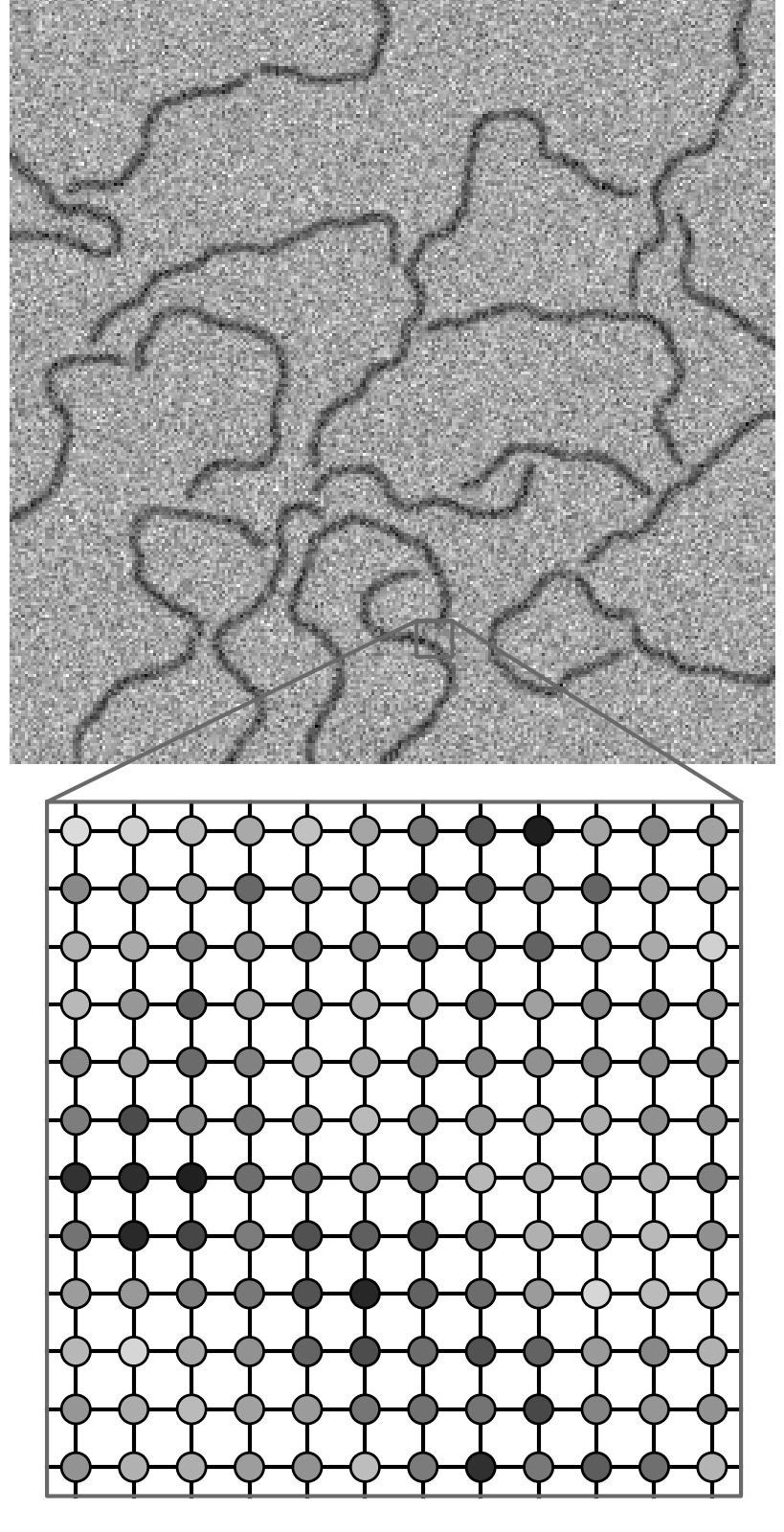}
        \caption{}
        \label{fig:img-seg-example-d}
    \end{subfigure}
    \begin{subfigure}[t]{0.19\textwidth}
        \includegraphics[height=5.5cm]{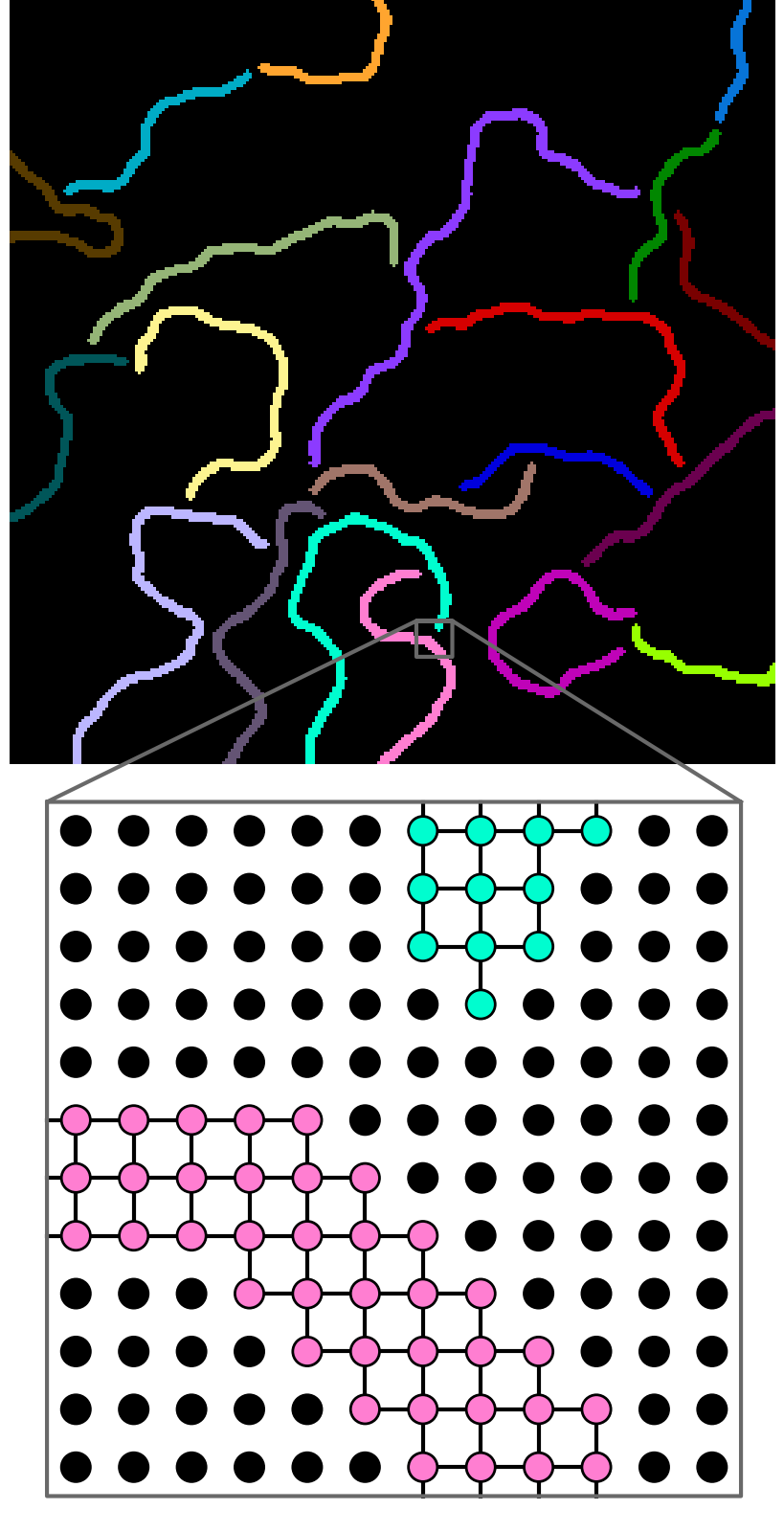}
        \caption{}
        \label{fig:img-seg-example-e}
    \end{subfigure}
    \caption{Depicted above are synthetic gray-scale images of foam cells, in \textbf{(a)}, and filaments, in \textbf{(d)}.
    Depicted in \textbf{(b)} and \textbf{(e)} in color are components of the pixel grid graphs of these image obtained by removing nodes (depicted in black). 
    These nodes represent a separating structure that we call a \emph{multi-separator} of the graph, and by which we treat the segmentation of foam cells and filaments analogously.
	This is in contrast to \textbf{(c)} where components of the pixel grid graph of the image \textbf{(a)} are obtained by removing edges, more specifically, a \emph{multicut} of the graph, and no pixel represents a separating structure.}
    \label{fig:img-seg-example}
\end{figure}

\section{Related work}\label{sec:related-work}
A close connection exists, as we show by \Cref{thm:lmp-msp-reduction,thm:msp-lmp-reduction}, between the multi-separator problem we define and the lifted multicut problem defined by 
\citet{keuper2015efficient} and discussed by \citet{horvnakova2017analysis,kardoost-2018,andres2023polyhedral}.
Both problems are defined with respect to a graph $G = (V,E)$ and a set $F$ of arbitrary node pairs.
The feasible solutions to the lifted multicut problem relate one-to-one to the decompositions of the graph, i.e.~all partitions of the node set into component-inducing subsets.
These are precisely all ways of decomposing the graph by removing edges.
In addition, the feasible solutions to the lifted multicut problem make explicit for all nodes pairs $\{u, v\} \in E \cup F$ whether $u$ and $v$ are in distinct components.
By assigning a positive or negative cost to this decision, the objective function penalizes or rewards decompositions that have this property.
No costs or constraints are imposed on the number or size of components.
Instead, these properties are determined by the solutions.
The lifted multicut problem has applications in the field of image analysis, notably to the tasks of image segmentation \citep{beier2017multicut,wolf2020mutex,lee2021learning}, video segmentation \citep{keuper-2017} and multiple object tracking \citep{tang-2017-multiple}.
The multi-separator problem is similar in that the objective function assigns positive or negative costs to pairs of nodes being in distinct components.
It is similar also in that no costs or constraints are imposed on the number or size of components and, instead, these properties are determined by the solutions.
The multi-separator problem is different, however, in that its feasible solutions do not relate to all ways of decomposing the graph by removing edges but to all ways of separating the graph by removing nodes.

More fundamental is the multicut problem, i.e.~the specialization of the lifted multicut problem with $F = \emptyset$, in which costs are assigned only to pairs of neighboring nodes \citep{chopra1993partition}.
The complexity and approximability of this problem and the closely related correlation clustering and coalition structure generation problems have been studied for signed graphs \citep{bansal2004correlation}, weighted graphs \citep{charikar2005clustering,demaine2006correlation} and planar graphs \citep{voice2012coalition,bachrach2013optimal,klein2023correlation}.
Connections to the task of image segmentation and algorithms are explored e.g.~by \citet{kappes-2011-globally,andres2011probabilistic,yarkony-2012,beier2014cut,kim-2014,zhang-2014,beier-2015-fusion,alush-2016,kappes-2016,reinelt-2016,kirillov2017instancecut,kardoost-2021}.

Of particular interest for applications in image analysis are special cases of the multicut and lifted multicut problem that can be solved efficiently.
As \citet{wolf2020mutex} show, the multicut problem can be solved efficiently for a cost pattern that we refer to here as absolute dominant costs.
Their result implies that also the lifted multicut problem with absolute dominant costs and non-positive (i.e.~cut-rewarding) costs for all non-neighboring node pairs can be solved efficiently.
The efficient algorithm by \citet{wolf2020mutex} is used e.g.~by \citet{lee2021learning} for segmenting volume images.
In contrast, the lifted multicut problem with a non-negative (i.e.~cut-penalizing) cost for even a single pair of non-neighboring nodes is \textsc{np}-hard \citep{horvnakova2017analysis}.
This hampers applications of the lifted multicut problem to the task of segmenting volume images of filaments in which one would like to attribute positive costs to some pairs of non-neighboring nodes in order to prevent false cuts.
Here, in this article, we show by \Cref{thm:ms-abs-dom-efficient}: Unlike the lifted multicut problem, the multi-separator problem can be solved efficiently also for absolute dominant costs and non-negative costs for non-neighboring nodes.

An efficient technique for image segmentation by nodes separators is the computation of watersheds \citep{meyer-1991,soille-1991}; see \citet{roerdink2000watershed} for a survey.
Watersheds depend on weights attributed to individual nodes.
Each component separated by watersheds corresponds to a local minimum of a node-weighted graph, or to a connected node set provided as additional input.
In contrast, the multi-separator problem associates costs also with node pairs, and its solutions are not constrained to local optima.
Watershed segmentation is canonical for images where a good initial estimate of components exists, like for images of foam cells.
It is less canonical for images where such estimates are difficult, like for images of filaments.
So far, fundamentally different models are used for reconstructing filaments, including 
\citep{rempfler-2015,shit-2022,turetken2016reconstructing}.
In our experiments, we empirically compare feasible solutions to a multi-separator problem to watershed segmentations.

Toward more complex models for image segmentation by node separators, an \textsc{np}-hard problem introduced and analyzed by \citet{hornakova-2020} is similar to the multi-separator problem in that it attributes unconstrained costs to pairs of nodes and in that its feasible solutions define components that are node-disjoint.
It is different from the multi-separator problem in that the components defined by its feasible solutions are necessarily paths, and in that these components are not necessarily node-separated.
An \textsc{np}-hard problem introduced and analyzed by \citet{nowozin-2010} is similar to the multi-separator problem in that costs can penalize disconnectedness.
It differs from the multi-separator problem in that feasible solutions define at most one component.

Graph separators have been studied also from a theoretical perspective. 
Structural results include Menger's Theorem \citep{menger1927allgemeinen}, the planar separator theorem \citep{lipton1979separator}, and the observation that all minimal $st$-separators form a lattice \citep{escalante1972schnittverbande}.
An efficient algorithm for enumerating all minimal separators of a graph is by \citet{berry2000generating}.
The problem of finding a partition of the node set of a graph into three sets $A$, $B$, $C$ such that $A$ and $B$ are separated by $C$ and such that $|C|$ is minimal subject to some constraints on $|A|$ and $|B|$ is studied by \citet{balas2005vertex,souza2005vertex,didi2011exact}.
This problem is \textsc{np}-hard even for planar graphs \citep{fukuyama2006np}.
The vertex $k$-cut problem asks for a minimum cardinality subset of nodes whose removal disconnects the graph into at least $k$ components.
\citet{cornaz2019vertex} show that this problem is \textsc{np}-hard for $k \geq 3$.
Exact algorithms for this problem are studied by \citet{furini2020integer}.
For a given set of terminal nodes, the multi-terminal vertex separator problem consists in finding a subset of nodes whose removal disconnects the graph such that no two terminals are in the same component.
This problem is studied by \citet{garg2004multiway,cornaz2019multi,magnouche2021multi}.
More loosely connected variants of the graph separator problem can be found in the referenced articles as well as in the articles referenced there.
The multi-separator problem we propose here is different in that no costs or constraints are imposed on the number of size of components and, instead, these properties are determined by the solutions.

\section{Multi-separator problem}\label{sec:model}
In this section, we define the multi-separator problem as a combinatorial optimization problem over graphs and discuss its connection to the lifted multicut problem.
\subsection{Problem statement}\label{sec:problem} % note change here
Let $G = (V, E)$ be a connected graph, let $S \subseteq V$ be a subset of nodes, and let $u, v \in V$ with $u \neq v$. 
We say $u$ and $v$ are \emph{separated by $S$} if $u \in S$ or $v \in S$ or every $uv$-path in $G$ passes through at least one node in $S$.
Conversely, $u$ and $v$ are \emph{not} separated by $S$ if there exists a component in the subgraph of $G$ induced by $V \setminus S$ that contains both $u$ and $v$.
In this article, we call every node subset $S \subseteq V$ a \emph{multi-separator} or, abbreviating, just a \emph{separator} of $G$.
Examples are depicted in \Cref{fig:img-seg-example-b,fig:img-seg-example-e} where the nodes of different colors are separated by the set of nodes depicted in black.
Given an arbitrary set $F \subseteq \tbinom{V}{2}$ of node pairs, we let $F(S) \subseteq F$ denote the subset of those pairs that are separated by $S$.

\begin{definition}\label{def:multi-separator}
    Let $G=(V, E)$ be a connected graph, let $F \subseteq \binom{V}{2}$ be a set of node pairs called \emph{interactions}, and let $c: V \cup F \to \R$ be called a \emph{cost vector}.
    The \emph{min-cost multi-separator} problem with respect to the graph $G$, the interactions $F$, and the cost vector $c$ consists in finding a node subset $S \subseteq V$ called a \emph{separator} so as to minimize the sum of the costs of the nodes in $S$ plus the sum of the costs of those interactions $F(S) \subseteq F$ that are separated by $S$, i.e.
    \begin{align}\label{eq:msp}
        \min_{S \subseteq V} \quad \sum_{v \in S} c_v + \sum_{f \in F(S)} c_f \enspace . \tag{MSP}
    \end{align}
    We say an interaction $f \in F$ is \emph{repulsive} if its associated costs $c_f$ is negative, for in this case, the objective of \eqref{eq:msp} is decreased whenever $f$ is separated.
    Analogously, we say $f$ is \emph{attractive} if the associated cost if positive.
    Likewise, we call a node $v \in V$ repulsive (attractive) whenever its associated cost $c_v$ is negative (positive).

    For any separator $S$, we define the characteristic vector $x^S: V \cup F \to \{0, 1\}$ 
    such that $x^S_v = 1$ for $v \in S$, such that $x^S_v = 0$ for $v \in V \setminus S$, such that $x^S_f = 1$ for $f \in F (S)$, and such that $x^S_f = 0$ for $f \in F \setminus F(S)$.
    For the set of all such characteristic vectors, we write $\ms(G,F) := \{x^S \mid S \subseteq V\}$.
    With this, \eqref{eq:msp} is written equivalently as
    \begin{align}
        \min_{x \in \ms(G, F)} \quad \sum_{g \in V \cup F} x_g \, c_g \enspace . \tag{MSP'}
    \end{align}
\end{definition}

\subsection{Connection to the lifted multicut problem}\label{sec:connection-to-lmc} % note change here
The multi-separator problem \eqref{eq:msp} is closely related to the \emph{lifted multicut problem} \citep[Definition 5]{andres2023polyhedral}.
In this section, we show that these problems are equally hard in the sense that solving either can be reduced in linear time to solving the other.
In subsequent sections, we establish differences between these problems.

The \emph{multicut problem} asks for a partition of the node set of a graph into component-inducing subsets such that the sum of the costs of the edges that straddle distinct components is minimized. 
The \emph{lifted multicut problem} generalizes the multicut problem by assigning costs not only to edges of the graph but also to pairs of nodes that are not necessarily connected by an edge.
Formally, the lifted multicut problem is defined as follows.

\begin{definition}\label{def:lmc}
    Let $G=(V,E)$ be a connected graph and let $\widehat{G} = (V, E \cup F)$ be an augmentation of $G$ where $F \subseteq \binom{V}{2} \setminus E$ is called a set of \emph{(additional) long-range edges}.
    A subset $M \subseteq E \cup F$ is called a \emph{multicut of $\widehat{G}$ lifted from $G$} if there exists a partition $\{\pi_1,\dots,\pi_n\}$ of $V$ such that each $\pi_i$ induces a component in $G$, and $M = \left\{ \{u,v\} \in E \cup F \mid \forall i \in \{1,\dots,n\} \colon \{u, v\} \not\subseteq \pi_i \right\}$.
    Let $\lmc(G,\widehat{G})$ denote the set of all multicuts of $\widehat{G}$ lifted from $G$.
    
    The \emph{lifted multicut problem} with respect to the graph $G$, the augmented graph $\widehat{G}$ and costs $c: E \cup F \to \R$ consists in finding a multicut of $\widehat{G}$ lifted from $G$ with minimal edge costs, i.e.
    \begin{align}\label{eq:lmp}
        \min_{M \in \lmc(G,\widehat{G})} \sum_{e \in M} c_e \enspace. \tag{LMP}
    \end{align}
\end{definition}

The following two theorems state that the problems \eqref{eq:msp} and \eqref{eq:lmp} can be reduced in linear time to one another.
Examples of these reductions are depicted in \Cref{fig:msp-lmc-reduction}.

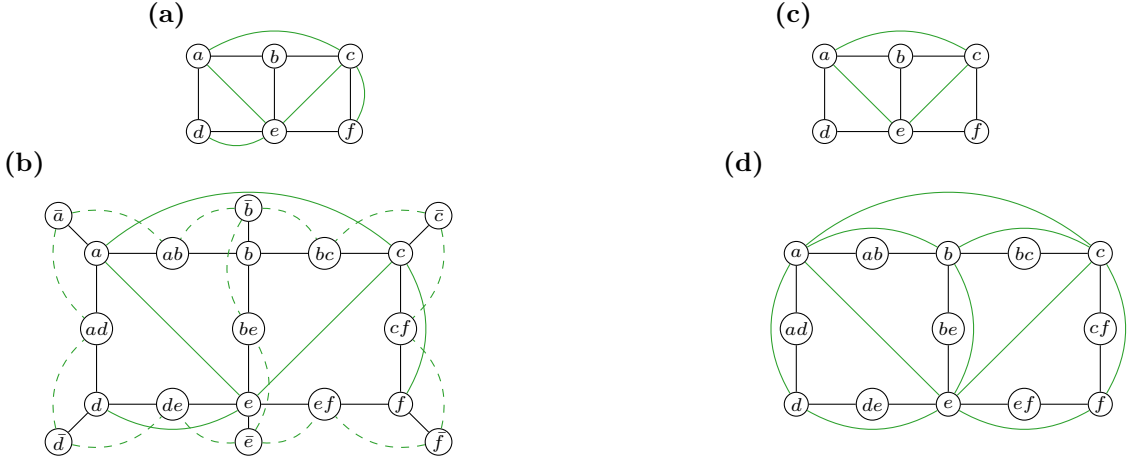
\begin{figure}
    \centering
    \textbf{(a)}\imagetop{\begin{tikzpicture}
    \node[vertex, minimum width=9pt] (a) at (0, 1) {\scriptsize $a$};
    \node[vertex, minimum width=9pt] (b) at (1, 1) {\scriptsize $b$};
    \node[vertex, minimum width=9pt] (c) at (2, 1) {\scriptsize $c$};
    \node[vertex, minimum width=9pt] (d) at (0, 0) {\scriptsize $d$};
    \node[vertex, minimum width=9pt] (e) at (1, 0) {\scriptsize $e$};
    \node[vertex, minimum width=9pt] (f) at (2, 0) {\scriptsize $f$};

    \draw (a) -- (b) -- (c) -- (f) -- (e) -- (d) -- (a);
    \draw (b) -- (e);

    \draw[mygreen] (a) to[bend left] (c);
    \draw[mygreen] (a) -- (e);
    \draw[mygreen] (c) -- (e);
    \draw[mygreen] (c) to[bend left] (f);
    \draw[mygreen] (d) to[bend right] (e);
\end{tikzpicture}}
    \hspace{5cm}
    \textbf{(c)}\imagetop{\begin{tikzpicture}
    \node[vertex, minimum width=9pt] (a) at (0, 1) {\scriptsize $a$};
    \node[vertex, minimum width=9pt] (b) at (1, 1) {\scriptsize $b$};
    \node[vertex, minimum width=9pt] (c) at (2, 1) {\scriptsize $c$};
    \node[vertex, minimum width=9pt] (d) at (0, 0) {\scriptsize $d$};
    \node[vertex, minimum width=9pt] (e) at (1, 0) {\scriptsize $e$};
    \node[vertex, minimum width=9pt] (f) at (2, 0) {\scriptsize $f$};

    \draw (a) -- (b) -- (c) -- (f) -- (e) -- (d) -- (a);
    \draw (b) -- (e);

    \draw[mygreen] (a) to[bend left] (c);
    \draw[mygreen] (a) -- (e);
    \draw[mygreen] (c) -- (e);
\end{tikzpicture}}
    \\
    \textbf{(b)}\imagetop{\begin{tikzpicture}
    \node[vertex, minimum width=9pt] (a) at (0, 2) {\scriptsize $a$};
    \node[vertex, minimum width=9pt] (b) at (2, 2) {\scriptsize $b$};
    \node[vertex, minimum width=9pt] (c) at (4, 2) {\scriptsize $c$};
    \node[vertex, minimum width=9pt] (d) at (0, 0) {\scriptsize $d$};
    \node[vertex, minimum width=9pt] (e) at (2, 0) {\scriptsize $e$};
    \node[vertex, minimum width=9pt] (f) at (4, 0) {\scriptsize $f$};

    \node[vertex, minimum width=10pt] (aa) at (-0.5, 2.5) {\scriptsize $\bar a$};
    \node[vertex, minimum width=10pt] (bb) at (2, 2.6) {\scriptsize $\bar b$};
    \node[vertex, minimum width=10pt] (cc) at (4.5, 2.5) {\scriptsize $\bar c$};
    \node[vertex, minimum width=10pt] (dd) at (-0.5, -0.5) {\scriptsize $\bar d$};
    \node[vertex, minimum width=10pt] (ee) at (2, -0.5) {\scriptsize $\bar e$};
    \node[vertex, minimum width=10pt] (ff) at (4.5, -0.5) {\scriptsize $\bar f$};

    \node[vertex, minimum width=12pt] (ab) at (1, 2) {\scriptsize $ab$};
    \node[vertex, minimum width=12pt] (ad) at (0, 1) {\scriptsize $ad$};
    \node[vertex, minimum width=12pt] (bc) at (3, 2) {\scriptsize $bc$};
    \node[vertex, minimum width=12pt] (be) at (2, 1) {\scriptsize $be$};
    \node[vertex, minimum width=12pt] (cf) at (4, 1) {\scriptsize $cf$};
    \node[vertex, minimum width=12pt] (de) at (1, 0) {\scriptsize $de$};
    \node[vertex, minimum width=12pt] (ef) at (3, 0) {\scriptsize $ef$};

    \draw (a) -- (aa);
    \draw (b) -- (bb);
    \draw (c) -- (cc);
    \draw (d) -- (dd);
    \draw (e) -- (ee);
    \draw (f) -- (ff);

    \draw (a) -- (ab) -- (b) -- (bc) -- (c) -- (cf) -- (f) -- (ef) -- (e) -- (de) -- (d) -- (ad) -- (a);
    \draw (b) -- (be) -- (e);

    \draw[mygreen] (a) to[bend left=40] (c);
    \draw[mygreen] (a) -- (e);
    \draw[mygreen] (c) -- (e);
    \draw[mygreen] (c) to[bend left] (f);
    \draw[mygreen] (d) to[bend right] (e);

    \draw[mygreen, dashed] (aa) to[bend left] (ab);
    \draw[mygreen, dashed] (aa) to[bend right] (ad);

    \draw[mygreen, dashed] (bb) to[bend right] (ab);
    \draw[mygreen, dashed] (bb) to[bend left] (bc);
    \draw[mygreen, dashed] (bb) to[bend right] (be);

    \draw[mygreen, dashed] (cc) to[bend right] (bc);
    \draw[mygreen, dashed] (cc) to[bend left] (cf);

    \draw[mygreen, dashed] (dd) to[bend left] (ad);
    \draw[mygreen, dashed] (dd) to[bend right] (de);

    \draw[mygreen, dashed] (ee) to[bend right] (be);
    \draw[mygreen, dashed] (ee) to[bend left] (de);
    \draw[mygreen, dashed] (ee) to[bend right] (ef);

    \draw[mygreen, dashed] (ff) to[bend right] (cf);
    \draw[mygreen, dashed] (ff) to[bend left] (ef);

\end{tikzpicture}}
    \hfill
    \textbf{(d)}\imagetop{\begin{tikzpicture}
    \node[vertex, minimum width=9pt] (a) at (0, 2) {\scriptsize $a$};
    \node[vertex, minimum width=9pt] (b) at (2, 2) {\scriptsize $b$};
    \node[vertex, minimum width=9pt] (c) at (4, 2) {\scriptsize $c$};
    \node[vertex, minimum width=9pt] (d) at (0, 0) {\scriptsize $d$};
    \node[vertex, minimum width=9pt] (e) at (2, 0) {\scriptsize $e$};
    \node[vertex, minimum width=9pt] (f) at (4, 0) {\scriptsize $f$};

    \node[vertex, minimum width=12pt] (ab) at (1, 2) {\scriptsize $ab$};
    \node[vertex, minimum width=12pt] (ad) at (0, 1) {\scriptsize $ad$};
    \node[vertex, minimum width=12pt] (bc) at (3, 2) {\scriptsize $bc$};
    \node[vertex, minimum width=12pt] (be) at (2, 1) {\scriptsize $be$};
    \node[vertex, minimum width=12pt] (cf) at (4, 1) {\scriptsize $cf$};
    \node[vertex, minimum width=12pt] (de) at (1, 0) {\scriptsize $de$};
    \node[vertex, minimum width=12pt] (ef) at (3, 0) {\scriptsize $ef$};

    \draw (a) -- (ab) -- (b) -- (bc) -- (c) -- (cf) -- (f) -- (ef) -- (e) -- (de) -- (d) -- (ad) -- (a);
    \draw (b) -- (be) -- (e);

    \draw[mygreen] (a) to[bend left=40] (c);
    \draw[mygreen] (a) -- (e);
    \draw[mygreen] (c) -- (e);
    \draw[mygreen] (a) to[bend left] (b);
    \draw[mygreen] (b) to[bend left] (c);
    \draw[mygreen] (c) to[bend left] (f);
    \draw[mygreen] (f) to[bend left] (e);
    \draw[mygreen] (e) to[bend left] (d);
    \draw[mygreen] (d) to[bend left] (a);
    \draw[mygreen] (b) to[bend left] (e);
\end{tikzpicture}}
    \caption{Depicted in \textbf{(a)} are a graph $G$ (in black) and a set of interactions $F$ (in green). 
    Depicted in \textbf{(b)} are the corresponding auxiliary graph $\bar{G}$ (in black) and the set of long-range edges $\bar{F}$ (in green) that occur in the reduction of \eqref{eq:msp} to \eqref{eq:lmp} in \Cref{thm:msp-lmp-reduction}. 
    The dashed edges have a high negative cost.
    This ensures that they are cut in any optimal solution.
    Depicted in \textbf{(c)} are a graph $G$ (black edges) and an augmentation $\widehat{G}$ of $G$ that contains the additional green edges.
    Depicted in \textbf{(d)} are the corresponding auxiliary graph $\bar{G}$ and the set of interactions $\bar{F}$ that occur in the reduction of \eqref{eq:lmp} to \eqref{eq:msp} in \Cref{thm:lmp-msp-reduction}.
    }
    \label{fig:msp-lmc-reduction}
\end{figure}

\begin{theorem}\label{thm:msp-lmp-reduction}
    The multi-separator problem \eqref{eq:msp} can be reduced to the lifted multicut problem \eqref{eq:lmp} in linear time.
\end{theorem}

\begin{proof}
    To begin with, observe that a solution to the multi-separator problem can be represented as a solution to the lifted multicut problem by interpreting each node in the separator as a singleton component consisting of just this node.
    In the following, we construct an instance of \eqref{eq:lmp} such that all relevant feasible solutions have this characteristic.

    Let $G=(V,E)$ be a connected graph, let $F \subseteq \binom{V}{2}$ be a set of interactions, and let $c: V \cup F \to \R$ be a cost function that define an instance of \eqref{eq:msp}.
    For $|V| = 1$, the multi-separator problem is trivial, so, from now on, we may assume $|V| \geq 2$.
    We construct an auxiliary graph $\bar{G}=(\bar{V}, \bar{E})$, an augmentation $\widehat{G} = (\bar{V},\bar{E} \cup \bar{F})$ of $\bar{G}$, and costs $\bar{c}: \bar{E} \cup \bar{F} \to \R$ such that the optimal solutions of the instance of \eqref{eq:msp} with respect to $G$, $F$, and $c$ correspond one-to-one to the optimal solutions of the instance of \eqref{eq:lmp} with respect to $\bar{G}$, $\widehat{G}$, and $\bar{c}$.
    For each node $v \in V$, we consider the node $v$ itself and a copy that we label $\bar{v}$. 
    For each edge $\{v,w\} \in E$, we introduce an additional node that we label $vw$. 
    For each $v \in V$, let $N_G(v) = \left\{w \in V \mid \{v,w\} \in E\right\}$ denote the set of neighbors of $v$ in $G$, and let $\deg_G(v) = |N_G(v)|$ denote the degree of $v$ in $G$.
    Altogether, we define nodes and edges as written below and as depicted in \Cref{fig:msp-lmc-reduction} (a) and (b).
    \begin{align*}
        \bar{V} &= 
            V \cup \{\bar{v} \mid v \in V\} \cup 
            \left\{
                vw \mid \{v,w\} \in E
            \right\} \\
        \bar{E} &= 
            \left\{
                \{v,\bar{v}\} \mid v \in V 
            \right\} \cup \left\{
                \{v,vw\} \mid v \in V, w \in N_G(v) 
            \right\} \\
        \bar{F} &= 
            F \cup
            \left\{
                \{\bar{v},vw\} \mid v \in V, w \in N_G(v)
            \right\}
    \end{align*}
    With regard to the costs $\bar{c}$, we define $C = 1 + \sum_{g \in V \cup F} |c_g|$ and
    \begin{align*}
        \bar{c}_{\{v, \bar{v}\}} &= C \cdot \deg_G(v)
            && \forall v \in V \\
        \bar{c}_{\{v,vw\}} &= C + \frac{c_v}{\deg_G(v)}
            && \forall v \in V\ \forall w \in N_G(v) \\
        \bar{c}_{\{v, w\}} &= c_{\{v, w\}} 
            && \forall \{v, w\} \in F \\
        \bar{c}_{\{\bar{v},vw\}} &= - C \cdot |V|
            && \forall v \in V \ \forall w \in N_G(v) \enspace.
    \end{align*}

    We show that there is a bijection $\phi$ from the feasible solutions of the instance of \eqref{eq:msp} with respect to $G$, $F$, and $c$ onto those feasible solutions of the instance of \eqref{eq:lmp} with respect to $\bar{G}$, $\widehat{G}$, and $\bar{c}$ that have cost less than $D = -2C|E|(|V|-1) + C$.
    We remark that $-2C|E|(|V|-1)$ is the cost of the lifted multicut that corresponds to the decomposition of $\bar{G}$ where the nodes $\bar{v}$ for $v \in V$ are in singleton components and all other nodes form one large component (the degree sum formula yields $\sum_{v \in V} \deg_G(v) = 2|E|$).

    Firstly, let $S \subseteq V$ be a feasible solution of \eqref{eq:msp} and let $C_S = \sum_{v \in S} c_v + \sum_{f \in F(S)} c_f$ be its cost.
    Then, by construction of $\bar{G}$ and $\widehat{G}$, the set
    \begin{align*}
        M = \phi(S) = \quad & 
        \left\{
            \{v,\bar{v}\} \mid v \in V \setminus S
        \right\} \\ \cup &\left\{
            \{v,vw\} \mid v \in S, w \in N_G(v)
        \right\} \\ 
        \cup & \; F(S) \\
        \cup & \left\{
            \{\bar{v},vw\} \mid v \in V, w \in N_G(v)
        \right\} 
    \end{align*}
    is a feasible multicut of $\widehat{G}$ lifted from $\bar{G}$.
    By definition of the costs $\bar{c}$, the cost of $M$ is $C_M = -2C|E|(|V|-1) + C_S$. 
    By definition of $C$ follows $C_S < C$, and thus, $C_M < D$.

    Secondly, let $M \subseteq \bar{E} \cup \bar{F}$ be any multicut of $\widehat{G}$ lifted from $\bar{G}$ with cost $C_M < D$.
    Since $C_M < D$ it must hold that $\{\bar{v},vw\} \in M$ for all $v \in V$ and $w \in N_G(v)$.
    In the following, let $v \in V$ arbitrary but fixed.
    We show that either $\{v,\bar{v}\} \in M$ and $\{v,w\} \notin M$ for all $w \in N_G(v)$, or $\{v,\bar{v}\} \notin M$ and $\{v,w\} \in M$ for all $w \in N_G(v)$. 
    (The two cases relate to the decisions of $v$ being part or not of the separator in the solution to the corresponding multi-separator problem).
    If $\{v,\bar{v}\} \notin M$, then $\{v,vw\} \in M$ for all $w \in N_G(v)$, by the definition of lifted multicuts and $\{\bar{v},vw\} \in M$.
    Otherwise, i.e.~if $\{v,\bar{v}\} \in M$, then $\{v,vw\} \notin M$ for all $w \in N_G(v)$ since $C_M < D$.
    Together, we have that either $\{v,\bar{v}\}$ is cut and all other outgoing edges of $v$ are not cut, or that $\{v,\bar{v}\}$ is not cut an all other outgoing edges of $v$ are cut.
    Therefore, the separator $S = \left\{v \in V \mid \{v,\bar{v}\} \notin M \right\}$ is such that $M = \phi(S)$.

    Together, we have shown that $\phi$ is a bijection from the set of feasible solutions of the instance of \eqref{eq:msp} to the set of feasible solutions of the instance of \eqref{eq:lmp} with cost less than $D$. 
    Moreover, the costs of feasible \eqref{eq:msp} solutions differ by the additive constant $-2C|E|(|V|-1)$ from the costs of their images under $\phi$.
    This concludes the proof that \eqref{eq:msp} can be reduced to \eqref{eq:lmp}.
    The time complexity of the reduction is linear in the size of the instance of \eqref{eq:msp}, by construction.
\end{proof}

\begin{theorem}\label{thm:lmp-msp-reduction}
    The lifted multicut problem \eqref{eq:lmp} can be reduced to the lifted multi-separator problem \eqref{eq:msp} in linear time.
\end{theorem}

\begin{proof}
    To begin with, observe that any solution to the lifted multicut problem can be represented as a solution to the multi-separator problem with respect to an auxiliary graph in which each edge of the original graph is replaced by two edges incident to an additional auxiliary node representative of the edge, namely as the separator consisting of precisely those auxiliary nodes whose corresponding edge is part of the lifted multicut.
    For an illustration, see \Cref{fig:msp-lmc-reduction} (c) and (d).
	
    More formally, let $G=(V, E)$ be a connected graph, let $\widehat{G}=(V, E \cup F)$ with $F \subseteq \binom{V}{2} \setminus E$ be an augmentation of $G$, and let $c: E \cup F \to \R$ be a cost function that define an instance of \eqref{eq:lmp}.
    Below, we construct an auxiliary graph $\bar{G} = (\bar{V}, \bar{E})$, a set of interactions $\bar{F} \subseteq \binom{\bar{V}}{2}$, and a cost function $\bar{c}: \bar{V} \cup \bar{F}$ such that the optimal solutions of the instance of \eqref{eq:lmp} with respect to $G$, $\widehat{G}$, and $c$ correspond one-to-one to the optimal solutions of the instance of \eqref{eq:msp} with respect to $\bar{G}$, $\bar{F}$, and $\bar{c}$.
    Specifically, we define
    \begin{align*}
        \bar{V} &= V \cup \left\{vw \mid \{v,w\} \in E\right\} \\
        \bar{E} &= \left\{ \{v, vw\} \mid v \in V, \{v,w\} \in E \right\} \\
        \bar{F} &= E \cup F \enspace .
    \end{align*}
    With regard to the costs $\bar c$, we define $C = 1 + \sum_{e \in E \cup F} |c_e|$ and
    \begin{align*}
        \bar{c}_v &= C \cdot |E|
        && \forall v \in V \\
        \bar{c}_{vw} &= C
        && \forall \{v,w\} \in E \\
        \bar{c}_{\{v,vw\}} &= 0 
        && \forall v \in V \ \forall \{v,w\} \in E \\
        \bar{c}_e &= c_e - C
        && \forall e \in E \\
        \bar{c}_f &= c_f
        && \forall f \in F
        \enspace .
    \end{align*}
    
    We establish the existence of a bijection $\phi$ from the feasible solutions of the instance of \eqref{eq:lmp} with respect to $G$, $\widehat{G}$, and $c$ onto those feasible solutions of the instance of \eqref{eq:msp} with respect to $\bar{G}$, $\bar{F}$, and $\bar{c}$ that have cost less than $C$.

    Firstly, let $M \subseteq E \cup F$ be a multicut of $\widehat{G}$ lifted from $G$, and let $C_M = \sum_{e \in M} c_e$ be its cost.
    Then, the separator $S = \phi(M) = \left\{vw \mid \{v,w\} \in E \cap M \right\} \subseteq \bar{V}$ is a feasible solution of \eqref{eq:msp} with 
    \[
        \bar{F}(S) = M \cup \left\{\{v,vw\} \mid v \in V \wedge \{v,w\} \in E \cap M \right\} \enspace .
    \]
    By the definition of $\bar{c}$, it has cost
    \begin{align}\label{eq:c_s-eq-c_m}
        C_S = \sum_{v \in S} \bar{c}_v + \sum_{f \in \bar{F}(S)} \bar{c}_f = \sum_{e \in M} \bar{c}_{e} = C_M \enspace.
    \end{align}
    Moreover, by definition of $C$, we have $C_S < C$.

    Secondly, let $S \subseteq \bar{V}$ be a feasible solution to \eqref{eq:msp} with cost $C_S = \sum_{v \in S} \bar{c}_v + \sum_{f \in \bar{F}(S)} \bar{c}_f < C$.
    By the assumption that $C_S < C$ and the definition of $\bar{c}$, it follows that $v \notin S$ for all $v \in V$.
    Every $\{v,w\} \in E$ with $\{v,w\} \in \bar{F}(S)$ implies $vw \in S$, as otherwise, $v$ and $w$ would be connected by the path along the nodes $v$, $vw$, $w$, and thus, $\{v,w\} \notin \bar{F}(S)$, by definition of $\bar{F}(S)$.
    Conversely, by the assumption $C_S < C$, we have $\{v,w\} \in \bar{F}$ whenever $vw \in S$.
    Moreover, by the definitions of the set $\bar{F}(S)$ and lifted multicuts, $M = \bar{F}(S) \cap (E \cup F)$ is a multicut of $\widehat{G}$ lifted from $G$.
    By construction: $\phi(M) = S$.

    Together, we have shown that $\phi$ is a bijection from the set of feasible solution of the instance of \eqref{eq:lmp} to the set of feasible solutions of the instance of \eqref{eq:msp} with cost less than $C$. 
    Furthermore, by \eqref{eq:c_s-eq-c_m}, the costs of feasible solutions related by $\phi$ are equal.
    Thus, finding an optimal solution to any of the two problems yields an optimal solution to the other.
    Moreover, the time complexity of the reduction is linear in the size of the instance of \eqref{eq:lmp}, by construction.
\end{proof}

In spite of their mutual linear reducibility, there are differences between the lifted multicut problem, on the one hand, and the multi-separator problem, on the other hand:
For $F = \emptyset$, the lifted multicut problem specializes to the multicut problem that is \textsc{np}-hard, while the multi-separator problem specializes to the linear unconstrained binary optimization problem that can be solved in linear time.
For $F \neq \emptyset$, we will extend this observation to specific cost functions in \Cref{thm:ms-abs-dom-efficient} and \Cref{thm:lmc-abs-dom-hard} below.
Also for these specific cost functions, the lifted multicut problem is \textsc{np}-hard, while the multi-separator problem can be solved efficiently.

%\section{Complexity}\label{sec:complexity}
\subsection{Hardness}

The reduction of the lifted multicut problem to the multi-separator problem in \Cref{thm:lmp-msp-reduction} implies several hardness results that we summarize below.

\begin{corollary}\label{cor:msp-apx-hard}
    The multi-separator problem \eqref{eq:msp} is \textsc{apx}-hard.
\end{corollary}

\begin{proof}
    The reduction of \eqref{eq:lmp} to \eqref{eq:msp} from \Cref{thm:lmp-msp-reduction} is approximation-preserving. 
    This follows directly from the fact that the solutions that are mapped onto one another by the bijection $\phi$ have the same cost.
    Therefore, approximating \eqref{eq:msp} is at least as hard as approximating \eqref{eq:lmp}.
    This implies that \eqref{eq:msp} is \textsc{apx}-hard, as \eqref{eq:lmp} is \textsc{apx}-hard \cite{andres2023polyhedral}.
\end{proof}

\begin{remark}
    In contrast to \Cref{cor:msp-apx-hard}, the reduction of \eqref{eq:msp} to \eqref{eq:lmp} from \Cref{thm:msp-lmp-reduction} is not approximation-preserving since the costs of two solutions that are mapped to one another differ by a large constant.
    Whether there exists an approximation-preserving reduction of \eqref{eq:msp} to \eqref{eq:lmp} or whether approximating \eqref{eq:msp} is harder than approximating \eqref{eq:lmp} is an open problem.
\end{remark}

For planar graphs, \citet{voice2012coalition} and \citet{bachrach2013optimal} have shown independently that edge sum coalition structure generation is \textsc{np}-hard.
The edge sum graph coalition structure generation problem asks for a decomposition of a graph that maximizes the sum of the costs of those edges whose nodes are in the same component.
Clearly, a decomposition maximizes the sum of the costs of those edges whose nodes are in the same component if and only if it minimizes the sum of the costs of those edges whose nodes are in distinct components.
Thus, for any graph, finding an optimal edge sum graph coalition structure is equivalent to finding an optimal multicut.
As the multicut problem is the special case of the lifted multicut problem without long-range edges, i.e. $F =  \emptyset$, the reduction from \Cref{thm:lmp-msp-reduction} implies the following hardness result.

\begin{corollary}
    The multi-separator problem \eqref{eq:msp} is \textsc{np}-hard even if the graph $(V, E \cup F)$ is planar.
\end{corollary}

\begin{proof}
    If $G=(V,E)$ is a planar graph, then the graph $(\bar{V}, \bar{E} \cup \bar{F})$ with $\bar{V}$, $\bar{E}$, and $\bar{F}$ as in the reduction of \eqref{eq:lmp} to \eqref{eq:msp} in \Cref{thm:lmp-msp-reduction} is also planar.
\end{proof}

% In the remainder of this section, we discuss the complexity of the multi-separator problem in more detail.

\subsection{Attraction or repulsion only}\label{sec:attraction-repulsion-only}

In the following, we show: The multi-separator problem remains \textsc{np}-hard for the special cases where all interactions are attractive or all interactions are repulsive.
More specifically, we show that the multi-separator problem generalizes the quadratic unconstrained binary optimization (QUBO) problem, the node-weighted Steiner tree problem \citep{moss-2007} and the multi-terminal vertex separator problem \citep{cornaz2019multi}. 

\begin{theorem}\label{thm:qubo-msp-reduction}
    Quadratic unconstrained binary optimization \eqref{eq:qubo} is equivalent to the special case of the multi-separator problem \eqref{eq:msp} with $E = F$.
\end{theorem}

\begin{proof}
    Let $n \in \mathbb{N}$, and $q_{ij} \in \R$ for all integers $i$ and $j$ such that $1 \leq i \leq j \leq n$.
    Then the quadratic unconstrained binary optimization problem with respect to coefficients $q$ is defined as
    \begin{align}\label{eq:qubo}
        \max \quad & \sum_{1 \leq i \leq j \leq n} q_{ij} \, x_{ij}  \tag{QUBO} \\
        \text{s.t.} \quad 
        & x_{ii} \in \{0, 1\} && \forall i \in \{1, \dots, n\} \notag \\
        & x_{ij} = x_{ii} \, x_{jj} && \forall i \in \{1, \dots, n\} \ \forall j \in \{i + 1, \dots, n\}  \enspace . \notag
    \end{align}
    Let $G=(V, E)$ be the graph with $V=\{1,\dots,n\}$ and $E = \left\{\{i,j\} \mid q_{ij} \neq 0,\; 1 \leq i < j \leq n \right\}$, let $F = E$ be the set of interactions and let $c: V \cup F \to \R$ with $c_i = q_{ii}$ for $i \in V$ and $c_{\{i, j\}} = q_{ij}$ for $1 \leq i < j \leq n$ with $\{i, j\} \in F$.
    It is easy to see that the instance of \eqref{eq:qubo} with respect to $q$ is equivalent to the instance of \eqref{eq:msp} with respect to $G$, $F$ and $c$:

    For all $x \in \ms(G, F)$ and all $\{i, j\} \in F$, we have $x_{\{i, j\}} = 0 \Leftrightarrow x_i = x_j = 0$, which is equivalent to $1 - x_{\{i,j\}} = (1 - x_i) (1 - x_j)$.
    Therefore, the affine transformation $x \mapsto 1-x$ is a bijection between the feasible solutions of the instance of \eqref{eq:msp} and the feasible solutions of the instance of \eqref{eq:qubo}.
    Moreover, if $x$ is a feasible solution to the instance of \eqref{eq:qubo} with cost $\sum_{1 \leq i \leq j \leq n} q_{ij} x_i x_j$, then $y=1-x$ is a feasible solution to the instance of \eqref{eq:msp} with cost
    \[
        \sum_{g \in V \cup F} c_g y_g = \sum_{1 \leq i \leq n} c_i (1-x_i) + \sum_{1 \leq i < j \leq n} c_{ij} (1- x_i x_j) = \sum_{1 \leq i \leq j \leq n} q_{ij} - \sum_{1 \leq i \leq j \leq n} q_{ij} x_i x_j \enspace.
    \]
    The two costs differ by sign and the additive constant term $\sum_{1 \leq i \leq j \leq n} q_{ij}$. 
    Thus, the maximizers of the instance of \eqref{eq:qubo} are precisely the minimizers of the instance of \eqref{eq:msp}, which concludes the proof.
\end{proof}

\begin{theorem}\label{thm:steiner-tree-msp-reduction}
    The multi-separator problem \eqref{eq:msp} generalizes the node-weighted Steiner tree problem.
\end{theorem}

\begin{proof}
    Let $G=(V, E)$ be a connected graph, let $U \subseteq V$ be called a set of terminals, and let $w: V \to \R_{\geq 0}$ assign a non-negative weight to each node in $G$. 
    The node-weighted Steiner tree problem with respect to $G$, $U$, and $w$ consists in finding a node set $T \subseteq V$ with minimum cost $\sum_{v \in T} w_v$ such that the subgraph of $G$ induced by $T$ contains a component that contains $U$.

    Let $U = \{u_1,\dots,u_k\}$ with $k = |U|$, and define interactions $F = \left\{\{u_1,u_2\},\dots,\{u_1,u_k\}\right\}$.
    Furthermore, let $W := \sum_{v \in V} w_v$, and define costs $c: V \cup F \to \R$ with $c_v = -w_v$ for $v \in V$, and with $c_f = W + 1$ for $f \in F$.
    Now, consider the multi-separator problem \eqref{eq:msp} with respect to $G$, $F$, and $c$. 
    By construction, the solutions to the Steiner tree problem relate one-to-one to the solutions of the multi-separator problem with non-positive cost:
    Every feasible solution $x \in \ms(G,F)$ with non-positive cost, i.e.~$\sum_{g \in V \cup F} c_g x_g \leq 0$, satisfies $x_f = 0$ for all $f \in F$.
    This implies that the subgraph of $G$ induced by $x^{-1}(0)$ contains a component that contains $U$.
    Conversely, if $T$ is a feasible solution to the Steiner tree problem then $x \in \{0,1\}^{V \cup F}$ defined such that $x_v = 1$ for $v \notin T$ and such that $x_g = 0$ for $g \in T \cup F$ is a feasible solution to the multi-separator problem that has a non-positive cost.

    Moreover,
    \[
        \sum_{g \in V \cup F} c_g x_g = \sum_{v \in V \setminus T} c_v = - \sum_{v \in V \setminus T} w_v = \sum_{v \in T} w_v - W \enspace .
    \]  
    Consequently, the optimal solutions of the Steiner tree problem relate one-to-one to the optimal solutions of the multi-separator problem, which concludes the proof.
\end{proof}

\begin{theorem}\label{thm:mwvc-msp-reduction}
    The multi-separator problem \eqref{eq:msp} generalizes the multi-terminal vertex separator problem.
\end{theorem}

\begin{proof}
    Let $G=(V, E)$ be a connected graph, let $U \subseteq V$ such that $\{u,v\} \notin E$ for $u,v \in U$ be called a set of terminals, and let $w: V \to \R_{\geq 0}$ assign a non-negative weight to each node in $G$. 
    The multi-terminal vertex separator problem with respect to $G$, $U$ and $w$ consists in finding a node set $S \subseteq V \setminus U$ such that each pair $u, v \in U$, $u \neq v$ of terminals is separated by $S$ and such that $\sum_{v \in S} w_v$ is minimal.
    The condition $\{u,v\} \notin E$ for $u,v \in U$ ensures that a feasible solution exists.

    We define the set of interactions $F = \{\{u,v\} \mid u, v \in U, u \neq v\}$ as the set of all pairs of terminals.
    Furthermore, we define $W := \sum_{v \in V} w_v$ and costs $c: V \cup F$ such that $c_v = w_v$ for $v \in V \setminus U$, such that $c_v = W+1$ for $v \in U$ and such that $c_f = - (W + 1)$ for $f \in F$.
    Now, consider the multi-separator problem \eqref{eq:msp} with respect to $G$, $F$ and $c$.
    By construction, the solutions to the multi-terminal vertex separator problem correspond one-to-one to the solutions of the multi-separator problem with cost at most $-|F| \cdot (W + 1)$:
    Every feasible solution $x \in \ms(G,F)$ with $\sum_{g \in V \cup F} c_g x_g \leq -|F| (W + 1)$ satisfies $x_f = 1$ for all $f \in F$ and satisfies $x_v = 0$ for all $v \in U$. 
    This implies that any pair of terminals is separated by the node set $x^{-1}(1) \subseteq V \setminus U$, i.e.~that $x^{-1}(1)$ is a feasible solution to the multi-terminal vertex separator problem.
    Conversely, if $S$ is a feasible solution to the multi-terminal vertex separator problem, then $x \in \{0,1\}^{V \cup F}$ defined such that $x_v = 0$ for $v \in V \setminus S$ and such that $x_g = 1$ for $g \in S \cup F$ is a feasible solution to the multi-separator problem with cost at most $-|F| \cdot (W + 1)$.

    Moreover,
    \[
        \sum_{g \in V \cup F} c_g x_g = - |F| \cdot (W + 1) + \sum_{v \in V \setminus U} c_v = - |F| \cdot (W + 1) + \sum_{v \in V \setminus U} w_v 
        \enspace .
    \]
    Consequently, the optimal solutions of the multi-terminal vertex separator problem relate one-to-one to the optimal solutions of the multi-separator problem, which concludes the proof.
\end{proof}

It is well known that \eqref{eq:qubo} is \textsc{np}-hard even if the costs of all quadratic terms are $-1$ \cite{barahona1982computational}.
Together with \Cref{thm:qubo-msp-reduction}, this immediately implies the following hardness results for \eqref{eq:msp} with all repulsive interactions.
The same result is obtained from the fact that the multi-terminal vertex separator problem is \textsc{np}-hard \citep{garg1994multiway}, together with \Cref{thm:mwvc-msp-reduction}.

\begin{corollary}
    The multi-separator problem \eqref{eq:msp} is \textsc{np}-hard even if $c_f = -1$ for all $f \in F$.
\end{corollary}

Conversely, if the costs of all quadratic terms in the \eqref{eq:qubo} objective are positive, then the objective is supermodular and, thus, \eqref{eq:qubo} can be solved efficiently \cite{grotschel1981ellipsoid}.
In contrast, the multi-separator problem remains \textsc{np}-hard even for all attractive interactions, due to \Cref{thm:steiner-tree-msp-reduction}:

\begin{corollary}
    The multi-separator problem \eqref{eq:msp} is \textsc{np}-hard even if $c_f \geq 0$ for all $f \in F$ and $c_v \leq 0$ for all $v \in V$.
\end{corollary}

\begin{proof}
    In the proof of \Cref{thm:steiner-tree-msp-reduction}, the node-weighted Steiner tree problem is reduced to an instance that satisfies the conditions.
    Moreover, the node-weighted Steiner tree problem is a generalization of the edge-weighted Steiner tree problem (the edge-weighted version can be reduced to the node-weighted version by subdividing each edge and giving the new node its weight), which is well-known to be \textsc{np}-hard \cite{karp1972reducibility}.
\end{proof}

\subsection{Feasibility of partial assignments}\label{sec:deciding-consistent}

The following theorem shows that the multi-separator problem is harder than quadratic unconstrained binary optimization, Steiner tree, and multi-terminal vertex separator, in the sense that deciding consistency is hard for the multi-separator problem and easy for the other problems.
Here, deciding consistency refers to the problem of deciding whether a partial assignment to the variables of a problem has an extension to all variables of the problem that is a feasible solution.
For the multi-separator problem, this problem is defined as follows.

\begin{definition}\label{def:msp-consistency}
    Let $G=(V,E)$ be a connected graph, let $F \subseteq \binom{V}{2}$ be a set of interactions.
    A partial variable assignment $x: V \cup F \to \{0, 1, *\}$, where $*$ indicates no assignment, is called \emph{consistent} with respect to $G$ and $F$ if and only if there exists $y \in \ms(G,F)$ such that $x_g = y_g$ for all $g \in V \cup F$ with $x_g \in \{0,1\}$.
    In that case, $y$ is called a \emph{feasible extension} of $x$.
\end{definition}

\begin{theorem}\label{thm:ms-consistency}
    Deciding consistency for the multi-separator problem \eqref{eq:msp} is \textsc{np}-complete.
\end{theorem}

\begin{proof}
    Deciding consistency is in \textsc{np} because any feasible extension serves as a certificate that can be verified in linear time.
    \textsc{np}-hardness can be shown by reducing \textsc{3-sat} to deciding consistency for the multi-separator problem, exactly analogous to the proof of Theorem 1 of \citet{horvnakova2017analysis}.
    An example of this reduction is depicted in \Cref{fig:3-sat-to-msp-reduction}.
\end{proof}

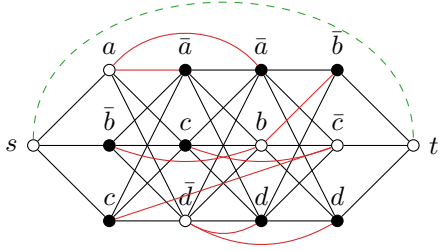
\begin{figure}
    \centering
    \begin{tikzpicture}
    \node[vertex, label=left:$s$] (s) at (0, 0) {};
    \node[vertex, label=right:$t$] (t) at (5, 0) {};
    
    \node[vertex, label=above:$a$] (00) at (1, 1) {};
    \node[separator, label=above:$\bar{b}$] (01) at (1, 0) {};
    \node[separator, label=above:$c$] (02) at (1, -1) {};

    \node[separator, label=above:$\bar{a}$] (10) at (2, 1) {};
    \node[separator, label=above:$c$] (11) at (2, 0) {};
    \node[vertex, label=above:$\bar{d}$] (12) at (2, -1) {};

    \node[separator, label=above:$\bar{a}$] (20) at (3, 1) {};
    \node[vertex, label=above:$b$] (21) at (3, 0) {};
    \node[separator, label=above:$d$] (22) at (3, -1) {};

    \node[separator, label=above:$\bar{b}$] (30) at (4, 1) {};
    \node[vertex, label=above:$\bar{c}$] (31) at (4, 0) {};
    \node[separator, label=above:$d$] (32) at (4, -1) {};

    % s
    \draw (s) -- (00);
    \draw (s) -- (01);
    \draw (s) -- (02);
    
    % 01
    \draw[myred] (00) -- (10);
    \draw (00) -- (11);
    \draw (00) -- (12);

    \draw (01) -- (10);
    \draw (01) -- (11);
    \draw (01) -- (12);

    \draw (02) -- (10);
    \draw (02) -- (11);
    \draw (02) -- (12);

    % 12
    \draw (10) -- (20);
    \draw (10) -- (21);
    \draw (10) -- (22);

    \draw (11) -- (20);
    \draw (11) -- (21);
    \draw (11) -- (22);

    \draw (12) -- (20);
    \draw (12) -- (21);
    \draw (12) -- (22);

    % 23
    \draw (20) -- (30);
    \draw (20) -- (31);
    \draw (20) -- (32);

    \draw[myred] (21) -- (30);
    \draw (21) -- (31);
    \draw (21) -- (32);

    \draw (22) -- (30);
    \draw (22) -- (31);
    \draw (22) -- (32);
    
    % t
    \draw (30) -- (t);
    \draw (31) -- (t);
    \draw (32) -- (t);

    \draw[myred] (00) to[bend left=50] (20);
    \draw[myred] (01) to[bend right=20] (21);
    \draw[myred] (11) to[bend right=20] (31);
    \draw[myred] (02) -- (31);
    \draw[myred] (12) to[bend right=30] (22);
    \draw[myred] (12) to[bend right=30] (32);

    \draw[mygreen, dashed] (s) to[out=90, in=180] (2.5, 1.9) to[out=0, in=90] (t);
    
\end{tikzpicture}
    \caption{The 3-SAT problem can be reduced to deciding consistency for \eqref{eq:msp}.
    For instance, the 3-SAT formula $(a \lor \bar{b} \lor c) \land (\bar{a} \lor c \lor \bar{d}) \land (\bar{a} \lor b \lor d) \land (\bar{b} \lor \bar{c} \lor d)$ is satisfiable if and only if there exists a separator $S$ in the graph depicted above (solid edges) such that $\{s,t\}$ (green, dashed) is not separated by $S$ and such that every pair $\{x,\bar{x}\}$ (red) for $x \in \{a, b, c, d\}$ is separated by $S$.
    This particular formula is satisfied, e.g., by the assignment $\{a, b, \bar{c}, \bar{d}\}$.
    The corresponding separator is depicted by the black nodes.}
    \label{fig:3-sat-to-msp-reduction}
\end{figure}

The following lemma identifies two special cases in which consistency can be decided efficiently.

\begin{lemma}\label{lem:msp-consistency-special-cases}
    For the multi-separator problem \eqref{eq:msp}, consistency can be decided in linear time if either $x_f \in \{0,*\}$ for all $f \in F$, or $x_f \in \{1, *\}$ for all $f \in F \setminus E$.
\end{lemma}

\begin{proof}
    Let $G = (V, E)$ be a connected graph, let $F \subseteq \binom{V}{2}$ be a set of interactions, and let $x: V \cup F \to \{0, 1, *\}$ be a partial assignment to the variables of the multi-separator problem with respect to $G$ and $F$.

    Firstly, suppose that $x_f \in \{0, *\}$ for all $f \in F$.
    Let $S = \{v \in V \mid x_v = 1\}$ be the set of nodes that is assigned to the separator by $x$.
    If there exists an interaction $f \in F$ with $x_f = 0$ such that $f$ is separated by $S$, then $x$ is clearly not consistent.
    Otherwise, $\{f \in F \mid x_f = 0\} \cap F(S) = \emptyset$.
    Now, let $y: V \cup F \to \{0, 1\}$ be the feasible solution that is induced by the separator $S$, i.e.
    \[
        y_g = \begin{cases}
            1 & \text{if } g \in S \cup F(S) \\
            0 & \text{ otherwise}
        \end{cases}  \enspace. 
    \]
    By construction, $y \in \ms(G, F)$ and $x_g = y_g$ for all $g \in V \cup F$ with $x_g \in \{0, 1\}$, which implies that $x$ is consistent.
    Therefore, consistency can be decided by the following algorithm: 
    1.~Delete the nodes that are labeled $1$ by $x$ from the graph, 
    2.~compute the connected components of the obtained graph, 
    3.~check whether all interactions that are labeled $0$ by $x$ are contained in a connected component.
    Computing the connected component can be done in time $\mathcal{O}(|V| + |E|)$ by breadth-first search.
    The overall algorithm is linear in the size $|V| + |E| + |F|$ of the multi-separator problem instance.

    Secondly, suppose that $x_f \in \{1, *\}$ for all $f \in F \setminus E$.
    Note: For any interaction that is an edge of $G$, i.e.~$\{u, v\} \in F \cap E$, with $x_{\{u,v\}} = 0$, we need to have $x_u \neq 1 \neq x_v$ in order for $x$ to be consistent.
    For the remainder of this proof, we may assume $x_u = x_v = 0$ for all $\{u,v\} \in F \cap E$ with $x_{\{u,v\}} = 0$.
    Let $S = \{v \in V \mid x_v \in \{1, *\}\}$ be the set of nodes that are assigned to the separator or do not have an assignment.
    If there exists an interaction $f \in F$ with $x_f = 1$ such that $f$ is not separated by $S$, then $x$ is clearly not consistent.
    Otherwise, it holds that $\{f \in F \mid x_f = 1\} \subseteq F(S)$.
    As before, let $y: V \cup F \to \{0,1\}$ be the feasible solution induced by the separator $S$.
    Again, by construction, $y \in \ms(G,F)$ and $x_g = y_g$ for all $g \in V \cup F$ with $x_g \in \{0,1\}$, which implies that $x$ is consistent.
    Similarly to above, consistency can be checked by the following algorithm: 
    0.~Check if for all $\{u,v\} \in F \cap E$ with $x_{\{u,v\}} = 0$, it holds that $x_u \neq 1 \neq x_v$; otherwise, $x$ is not consistent,
    1.~delete all nodes in $\{v \in V \mid x_v = 1 \text{ or } \nexists \{u,v\} \in F \cap E: x_{\{u,v\}} = 0\}$ from $G$, 
    2.~compute the connected components of the obtained graph,
    3.~check whether all interaction that are labeled $1$ by $x$ are not contained in a connected component.
    This algorithm is also linear in the size of the instance of the multi-separator problem.
\end{proof}

As shown in \Cref{thm:qubo-msp-reduction,thm:steiner-tree-msp-reduction,thm:mwvc-msp-reduction}, quadratic unconstrained binary optimization, Steiner tree, and multi-terminal vertex separator can be modeled as special cases of the multi-separator problem.
The three special cases all satisfy one of the two conditions of \Cref{lem:msp-consistency-special-cases}:
The special case of \eqref{eq:msp} that is equivalent to quadratic unconstrained binary optimization (\Cref{thm:qubo-msp-reduction}) satisfies $F \setminus E = \emptyset$, i.e. $x_f \in \{1,*\}$ for all $f \in F \setminus E$. 
The special cases of \eqref{eq:msp} used to model Steiner tree (\Cref{thm:steiner-tree-msp-reduction}) or multi-terminal vertex separator (\Cref{thm:mwvc-msp-reduction}) are such that either no interactions can be separated, i.e. $x_f  = 0$ for all $f \in F$, or all interactions need to be separated, i.e. $x_f = 1$ for all $f \in F$.
Thus, by \Cref{lem:msp-consistency-special-cases}, for quadratic unconstrained binary optimization, Steiner tree and multi-terminal vertex separator, consistency is efficiently decidable, whereas for the general multi-separator problem, it is \textsc{np}-hard (\Cref{thm:ms-consistency}).

\subsection{Absolute dominant costs}

In the remainder of this section, we analyze the multi-separator problem for a restricted class of cost functions. We begin with a brief motivation. 

Suppose we are not interested in finding a feasible solution that minimizes a linear cost function but a feasible solution that avoids mistakes, with respect to a given order of severity. 
More specifically, we want to find a feasible solution by looking at the variables of a problem in a given order, avoiding to assign 0 to variables that we call \emph{repulsive}, and avoiding to assign 1 to variables that we call \emph{attractive}.
For the (lifted) multicut problem, this preference problem is discussed in detail by \citet{wolf2020mutex}.
For the multi-separator problem \eqref{eq:msp}, this preference problem is discussed below.

Given a strict order $<$ on the set $V \cup F$ of variables, and given a bipartition $\{A,R\}$ of the set $V \cup F$ of variables, we refer to the variables in $A$ as \emph{attractive} and refer to the variables in $R$ as \emph{repulsive}. 
For any feasible solutions $x, y \in \text{MS}(G,F)$ with $x \neq y$, and for the smallest (with respect to $<$) variable $g$ with $x_g \neq y_g$, we \emph{prefer} $x$ over $y$, written as $x \prec y$, if both $g \in A$ and $x_g = 0$, or if both $g \in R$ and $x_g = 1$.
Otherwise, we prefer $y$ over $x$, written as $y \prec x$.
According to this definition, preference $\prec$ is a strict order on the set of all feasible solutions.
The \emph{preferred multi-separator problem} consists in finding the feasible solution that is preferred over all other feasible solutions.
This problem can be modeled as a \eqref{eq:msp} with a specific cost function:

\begin{definition}[Compare {\citet[Equation (8)]{wolf2020mutex}}]
    A cost function $c: V \cup F \to \R$ is called \emph{absolute dominant} if
	\begin{align}\label{eq:absolute-dominant}
        |c_g| & > \sum_{g' \in V \cup F, |c_{g'}| < |c_g|} |c_{g'}| & \forall g \in V \cup F \enspace .
	\end{align}
\end{definition}

Given an order $g_1,\dots,g_n$ and a partition $\{A,R\}$ of $V \cup F$ into attractive and repulsive variables, we can define absolute dominant costs as
\begin{align}\label{eq:ordered-to-linear}
    c_{g_i} & = 
        \begin{cases}
            2^{n-i} & \text{if } g_i \in A \\
            -2^{n-i} & \text{if } g_i \in R
        \end{cases}
        & \forall i \in \{1, \dots, n\}
        \enspace .
\end{align}
Now, the solutions to the preferred multi-separator problem with respect to $\{A,R\}$ and $g_1, \dots, g_n$ are precisely the solutions to the \eqref{eq:msp} with respect to the costs $c$ \citep{wolf2020mutex}.

As a direct consequence of \Cref{thm:ms-consistency}, we obtain the following hardness result.

\begin{theorem}\label{thm:ms-abs-dom-np-hard}
    The multi-separator problem \eqref{eq:msp} is \textsc{np}-hard even for absolute dominant costs.
\end{theorem}

\begin{proof}
    We show \textsc{np}-hardness by reducing the problem of deciding consistency for \eqref{eq:msp} to solving \eqref{eq:msp} for absolute dominant costs.

    To this end, let $x: V \cup F \to \{0, 1, *\}$ be a partial variable assignment for a given \eqref{eq:msp} instance with graph $G = (V, E)$ and interactions $F \subseteq \binom{V}{2}$.
    Let $n = |V \cup F|$, let $k = |\{g \in V \cup F \mid x_g \in \{0, 1\}\}|$ be the number of labeled variables, and let $g_1,\dots,g_n$ be an enumeration of the variables $V \cup F$ such that $x_{g_i} \in \{0, 1\}$ for $i \in \{1,\dots,k\}$.
    Furthermore, let $A = \{g \in V \cup F \mid x_g = 0\}$ and $R = (V \cup F) \setminus A$.
    Now, let $c$ be defined as in \eqref{eq:ordered-to-linear}.
    Observe that $x$ is consistent if and only if the optimal solution of the multi-separator problem with respect to costs $c$ has a value that is at most $-\sum_{i=1}^k x_{e_i} 2^{n-i}$.
\end{proof}

In contrast to the hardness result above, for both special cases that either all interactions are non-repulsive or all interactions are non-attractive, the multi-separator problem with absolute dominant costs can be solved efficiently:

\begin{theorem}\label{thm:ms-abs-dom-efficient}
    The multi-separator problem \eqref{eq:msp} can be solved efficiently for absolute dominant costs if either $c_f \leq 0$ for all $f \in F \setminus E$ or $c_f \geq 0$ for all $f \in F$.
\end{theorem}

\begin{proof}
    We describe an algorithm that computes the optimal solution for the multi-separator problem with absolute dominant costs.
    This algorithm involves deciding consistency (cf.~\Cref{sec:deciding-consistent}) and is therefore not a polynomial time algorithm in general.
    However, it is designed such that under the condition that either $c_f \leq 0$ for all $f \in F \setminus E$ or $c_f \geq 0$ for all $f \in F$, the consistency problems that need to be solved can be solved efficiently, by \Cref{lem:msp-consistency-special-cases}.

    Let $x: V \cup F \to \{0,1,*\}$ be a partial variable assignment.
    The algorithm starts with all variables being unassigned, i.e. $x_g = *$ for all $g \in V \cup F$.
    Clearly, $x$ is consistent.
    Now, the algorithm iterates over all variables $g \in V \cup F$ in descending order of their absolute cost.
    In each iteration, the algorithm assigns the variable $x_g$ to $0$ of $c_g > 0$ and to $1$ if $c_g \leq 0$. 
    If by assigning the variable $x_g$ to $0$ or $1$, $x$ is no longer consistent, this assignment is revoked, i.e. $x_g = *$.
    We observe that, if $x$ with $x_g = 0$ ($x_g = 1$) is not consistent, then it must hold that all consistent extensions $y \in \ms(G,F)$ of $x$ with $x_g = *$ satisfy $y_g = 1$ ($y_g = 0$).
    After iterating over all variables, $x$ is a consistent partial variable assignment that has exactly one consistent extension $y$, by the above observation.
    This consistent extension $y$ is an optimal solution to the problem by the following argument:
    Suppose there exists a feasible solution $y'$ with a strictly smaller cost.
    Let $g \in V \cup F$ be the variable with the largest absolute cost with $y_g \neq y'_g$.
    By the definition of absolute dominant costs \eqref{eq:absolute-dominant} and the assumption that the cost of $y'$ is smaller than the cost of $y$, it must hold that $c_g < 0$ if $y'_g = 1$ and $c_g > 0$ if $y'_g = 0$.
    Now, let $x: V \cup F \to \{0, 1, *\}$ be the partial variable assignment such that $x_{g'} = y'_{g'}$ for all $g' \in V \cup F$ with $|c_{g'}| \geq |c_g|$, and such that $x_{g'} = *$ for all $g' \in V \cup F$ with $|c_{g'}| < |c_g|$.
    If $x$ was consistent, then the algorithm would have assigned the variable $g$ to $y'_g$.
    This is in contradiction to $y_g \neq y'_g$. 
    Thus, $y$ is optimal.

    The algorithms consists of $|V| + |F|$ iterations.
    In each iteration, it needs to be decided if the partial assignment $x$ is consistent.
    By the design of the algorithm and the assumption that either $c_f \leq 0$ for all $f \in F \setminus E$ or $c_f \geq 0$ for all $f \in F$, the partial assignment satisfies one of the conditions of \Cref{lem:msp-consistency-special-cases}
    (if $c_f \leq 0$ for all $f \in F \setminus E$, then $x_f \in \{1, *\}$ for all $f \in F \setminus E$.
    If $c_f \geq 0$ for all $f \in F$, then $x_f \in \{0, *\}$ for all $f \in F$). 
    Therefore, consistency can be decided in linear time.
    Overall, the algorithm has a time complexity of $\mathcal{O}(n^2)$ with $n = |V| + |E| + |F|$ the size of the instance of the problem.
\end{proof}

\begin{remark}
    For absolute dominant costs and non-positive costs for all non-neighboring node pairs (i.e. $c_f \leq 0$ for all $f \in F \setminus E$), the algorithm described in the proof of \Cref{thm:ms-abs-dom-efficient} is similar to the Mutex-Watershed algorithm of \citet{wolf2020mutex}.
    The Mutex-Watershed algorithm solves the lifted multicut problem efficiently for absolute dominant costs with all non-positive costs for all non-neighboring node pairs.
\end{remark}

The fact that the multi-separator problem can be solved efficiently also for absolute dominant costs and \emph{non-negative} costs for all interactions is a key difference to the lifted multicut problem that is \textsc{np}-hard even for absolute dominant costs and non-negative costs for all non-neighboring node pairs:

\begin{theorem}\label{thm:lmc-abs-dom-hard}
    The lifted multicut problem \eqref{eq:lmp} with respect to a connected graph $G = (V, E)$, an augmented graph $\widehat{G} = (V, E \cup F)$ and absolute dominant edge costs $c: E \cup F \to \R$ with $c_e \leq 0$ for $e \in E$ and $c_f \geq 0$ for $f \in F$ is \textsc{np}-hard.
\end{theorem}

\begin{proof}
    By Theorem 1 of \citet{horvnakova2017analysis}, deciding consistency of a partially labeled lifted multicut is \textsc{np}-hard.
    Similarly to the proof of \Cref{thm:ms-abs-dom-np-hard}, we show that deciding consistency of a partially labeled lifted multicut can be reduced to solving the lifted multicut problem with respect to costs constrained as in \Cref{thm:lmc-abs-dom-hard}.

    Let $x: E \cup F \to \{0, 1, *\}$ be a partially labeled lifted multicut. 
    Without loss of generality, we may assume $x^{-1}(0) \subseteq F$ and $x^{-1}(1) \subseteq E$ by the following arguments.
    If there was an edge $e \in E$ with $x_e = 0$, we would contract the edge $e$ in $G$ and $\widehat{G}$ and consider the consistency problem with respect to the contracted graphs.
    If there was an edge $f \in F$ with $x_f = 1$, we would delete it from $F$ and add it to $E$ without altering the consistency problem.

    Now, the claim follows analogously to the proof of \Cref{thm:ms-abs-dom-np-hard}.
\end{proof}

\section{Local search algorithms}\label{sec:algorithms}
In this section, we define two efficient local search algorithms for the multi-separator problem.
Our main motivation for studying efficient algorithms that do not necessarily find optimal solutions comes from the task of segmenting volume images with voxel grid graphs with tens of millions of nodes (\Cref{fig:runtime-analysis}).

One approach to defining efficient local search algorithms would be to map instances of the multi-separator problem to instances of the lifted multicut problem, as in the proof of \Cref{thm:msp-lmp-reduction}, apply known local search algorithms to these, and construct feasible solutions from the output.
This approach is hampered, however, by the fact that local search algorithms for the lifted multicut problem are not designed for the instances constructed in the proof of \Cref{thm:msp-lmp-reduction}. 
For example, greedy additive edge contraction starting from singleton components 
\citep{keuper2015efficient}
results in the feasible solution in which every node $v$ is put together with its copy $\bar{v}$, and \emph{all} nodes $v,w$ with $\{v,w\} \in E$ remain in distinct components.

\begin{figure}
    \centering
    \begin{tikzpicture}
    \def \d {1.5cm}
    % Instance
    \node[vertex, text width=8pt] (0) at (-1.5*\d, 0) {\scriptsize 6};
    \node[vertex, text width=8pt] (1) at (-1.5*\d, 1) {\scriptsize 4};
    \node[vertex, text width=8pt] (2) at (-1.5*\d, 2) {\scriptsize 3};
    \node[vertex, text width=8pt] (3) at (-1.5*\d, 3) {\scriptsize 2};

    \draw (0) -- node[right] {\scriptsize 1} (1);
    \draw (1) -- node[right] {\scriptsize 1} (2);
    \draw (2) -- node[right] {\scriptsize 7} (3);
    \draw[mygreen] (0) to[bend left] node[left] {\scriptsize -8} (3);

    \draw (-0.75*\d, -0.5) -- (-0.75*\d, 3.5);

    % Step 0
    \node[separator, text width=8pt] (0) at (0, 0) {\scriptsize -6};
    \node[separator, text width=8pt] (1) at (0, 1) {\scriptsize -4};
    \node[separator, text width=8pt] (2) at (0, 2) {\scriptsize -3};
    \node[separator, text width=8pt] (3) at (0, 3) {\scriptsize -2};
    \draw[cut-edge] (0) -- (1);
    \draw[cut-edge] (1) -- (2);
    \draw[cut-edge] (2) -- (3);
    \draw[cut-edge, mygreen] (0) to[bend left] (3);
    \node at (0, -0.5) {16};

    % Step 1
    \node[vertex, text width=8pt] (0) at (1*\d, 0) {\scriptsize 6};
    \node[separator, text width=8pt] (1) at (1*\d, 1) {\scriptsize -5};
    \node[separator, text width=8pt] (2) at (1*\d, 2) {\scriptsize -3};
    \node[separator, text width=8pt] (3) at (1*\d, 3) {\scriptsize -2};
    \draw[cut-edge] (0) -- (1);
    \draw[cut-edge] (1) -- (2);
    \draw[cut-edge] (2) -- (3);
    \draw[cut-edge, mygreen] (0) to[bend left] (3);
    \node at (1*\d, -0.5) {10};

    % Step 2
    \node[vertex, text width=8pt] (0) at (2*\d, 0) {\scriptsize 7};
    \node[vertex, text width=8pt] (1) at (2*\d, 1) {\scriptsize 5};
    \node[separator, text width=8pt] (2) at (2*\d, 2) {\scriptsize -4};
    \node[separator, text width=8pt] (3) at (2*\d, 3) {\scriptsize -2};
    \draw (0) -- (1);
    \draw[cut-edge] (1) -- (2);
    \draw[cut-edge] (2) -- (3);
    \draw[cut-edge, mygreen] (0) to[bend left] (3);
    \node at (2*\d, -0.5) {5};

    % Step 3
    \node[vertex, text width=8pt] (0) at (3*\d, 0) {\scriptsize 7};
    \node[vertex, text width=8pt] (1) at (3*\d, 1) {\scriptsize 6};
    \node[vertex, text width=8pt] (2) at (3*\d, 2) {\scriptsize 4};
    \node[separator, text width=8pt] (3) at (3*\d, 3) {\scriptsize -1};
    \draw (0) -- (1);
    \draw (1) -- (2);
    \draw[cut-edge] (2) -- (3);
    \draw[cut-edge, mygreen] (0) to[bend left] (3);
    \node at (3*\d, -0.5) {1};

    % Step 4
    \node[vertex, text width=8pt] (0) at (4*\d, 0) {\scriptsize -1};
    \node[vertex, text width=8pt] (1) at (4*\d, 1) {\scriptsize -2};
    \node[vertex, text width=8pt] (2) at (4*\d, 2) {\scriptsize 3};
    \node[vertex, text width=8pt] (3) at (4*\d, 3) {\scriptsize 1};
    \draw (0) -- (1);
    \draw (1) -- (2);
    \draw (2) -- (3);
    \draw[mygreen] (0) to[bend left] (3);
    \node at (4*\d, -0.5) {0};

    % Step 5
    \node[vertex, text width=8pt] (0) at (5*\d, 0) {\scriptsize 6};
    \node[separator, text width=8pt] (1) at (5*\d, 1) {\scriptsize 2};
    \node[vertex, text width=8pt] (2) at (5*\d, 2) {\scriptsize 3};
    \node[vertex, text width=8pt] (3) at (5*\d, 3) {\scriptsize 1};
    \draw[cut-edge] (0) -- (1);
    \draw[cut-edge] (1) -- (2);
    \draw (2) -- (3);
    \draw[cut-edge, mygreen] (0) to[bend left] (3);
    \node at (5*\d, -0.5) {-2};

\end{tikzpicture}
    \caption{Depicted on the left is an instance of the multi-separator problem with respect to a graph, in black, and four interactions, including one non-edge, depicted in green.
    Depicted next to this graph, from left to right, are the iterations of the greedy algorithm defined in \Cref{sec:algorithms}.
    In each iteration, the separator consists of the nodes that are depicted as solid black circles.
    Initially, all nodes are in the separator.
    Separated interactions are depicted as dashed lines.
    The number inside each node is the potential of that node in the respective iteration.
    Written below each graph is the current total cost. 
    It can be seen in this example that the cost decreases with increasing iteration, until all nodes have a non-negative potential.}
    \label{fig:greedy-example}
\end{figure}
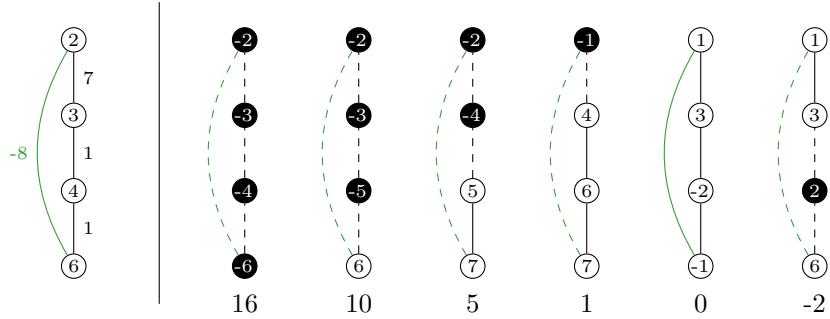

Here, we pursue a different approach and define local search algorithms for the multi-separator problem specifically. 
As local transformations, we consider the insertion and removal of single nodes into and from the separator. 
We perform these transformations greedily, always considering one that decreases the cost function maximally:
For any separator $S \subseteq V$ and any node $v \in V \setminus S$ not in the separator, we let $d(S, v)$ denote the difference in cost that results from the insertion of $v$ into the separator $S$, i.e.
\begin{align}\label{eq:obj-diff}
    d(S, v) =
    \left(
        \sum_{u \in S \cup \{v\}} c_u + \sum_{f \in F(S \cup \{v\})} c_f
    \right)  - \left(
        \sum_{u \in S} c_u + \sum_{f \in F(S)} c_f
    \right) = c_v + \sum_{f \in F(S \cup \{v\}) \setminus F(S)} c_f
    \enspace .
\end{align}
Here, $F(S \cup \{v\}) \setminus F(S)$ is the set of interactions that are separated by $S \cup \{v\}$ but not by $S$. 
In other words, it is the set of interactions not separated by $S$ for which $v$ is a cut-node in the graph obtained by removing the nodes in $S$ from $G$.
With this, we define the \emph{greedy potential} function as
\begin{align*}
    p(S, v) = 
    \begin{cases}
        {\color{white}-} d(S, v) & \text{if } v \notin S \\
        - d(S \setminus \{v\}, v) & \text{if } v \in S
    \end{cases} \enspace .
\end{align*}
This function assigns to each node $v \in V$ the difference in cost that results from inserting $v$ into the separator $S$, if $v \notin S$, or removing $v$ from the separator $S$, if $v \in S$.
With this potential function, we specify a local search algorithm as follows:
Starting from an initial separator $S \subseteq V$, e.g.~$S = V$ or $S=\emptyset$, we enter an infinite loop.
Firstly, we choose a node $v \in \argmin_{u \in V} p(S,u)$ with the most negative potential. 
Then, if $p(S,v) \geq 0$, we terminate.
If $p(S,v) < 0$, we distinguish two cases: If $v \in S$, we remove $v$ from the separator, i.e.~$S \leftarrow S \setminus \{v\}$. If $v \notin S$, we insert $v$ into the separator, i.e.~$S \leftarrow S \cup \{v\}$.
As the total cost decreases strictly with increasing iteration, the algorithm terminates.
An example is depicted in \Cref{fig:greedy-example}.

Algorithms according to this specification need not be practical.
In each iteration, the potentials of all nodes might be computed in order to identify a node with the most negative potential.
Below, we define practical algorithms which are restricted to only one type of local transformation: removal of a single node from the separator, in \Cref{sec:gcg}, and insertion of a single node into the separator, in \Cref{sec:gsg}.
In fact, we treat these cases symmetrically.
This is possible for the multi-separator problem, thanks to \Cref{thm:ms-abs-dom-efficient}, and would not be possible for the lifted multicut-problem, due to \Cref{thm:lmc-abs-dom-hard}.

\subsection{Greedy Separator Shrinking}\label{sec:gcg}
\begin{figure}[t]
    \centering
    \begin{tikzpicture}[baseline=0]
    \node[vertex, text width=8pt] (0) at (0, 0) {\scriptsize 4};
    \node[vertex, text width=8pt] (1) at (1, 0) {\scriptsize -2};
    \node[vertex, text width=8pt] (2) at (2, 0) {\scriptsize 3};
    \node[vertex, text width=8pt] (3) at (0, 1) {\scriptsize 1};
    \node[vertex, text width=8pt] (4) at (1, 1) {\scriptsize 1};
    \node[vertex, text width=8pt] (5) at (2, 1) {\scriptsize -1};
    \node[vertex, text width=8pt] (6) at (0, 2) {\scriptsize -2};
    \node[vertex, text width=8pt] (7) at (1, 2) {\scriptsize 5};
    \node[vertex, text width=8pt] (8) at (2, 2) {\scriptsize -1};
    
    \draw (0) -- (1);
    \draw (1) -- (2);
    \draw (3) -- (4);
    \draw (4) -- (5);
    \draw (6) -- (7);
    \draw (7) -- (8);
    \draw (0) -- (3);
    \draw (3) -- (6);
    \draw (1) -- (4);
    \draw (4) -- (7);
    \draw (2) -- (5);
    \draw (5) -- (8);

    \draw[mygreen] (0) to[bend right] node[below, inner sep=1pt] {\scriptsize 1} (2);
    \draw[mygreen] (0) to[bend left] node[left, inner sep=1pt] {\scriptsize 1} (3);
    \draw[mygreen] (2) to node[pos=0.7, below, inner sep=1pt] {\scriptsize 3} (3);
    \draw[mygreen] (3) to node[pos=0.6, below, inner sep=1pt] {\scriptsize -2} (7);
    \draw[mygreen] (7) to[bend left] node[above, inner sep=1pt] {\scriptsize 2} (8);
\end{tikzpicture}
    \hfill
    \begin{tikzpicture}[baseline=0]
    \node[separator, text width=8pt] (0) at (0, 0) {\scriptsize -4};
    \node[separator, text width=8pt] (1) at (1, 0) {\scriptsize 2};
    \node[separator, text width=8pt] (2) at (2, 0) {\scriptsize -3};
    \node[separator, text width=8pt] (3) at (0, 1) {\scriptsize -1};
    \node[separator, text width=8pt] (4) at (1, 1) {\scriptsize -1};
    \node[separator, text width=8pt] (5) at (2, 1) {\scriptsize 1};
    \node[separator, text width=8pt] (6) at (0, 2) {\scriptsize 2};
    \node[separator, text width=8pt] (7) at (1, 2) {\scriptsize -5};
    \node[separator, text width=8pt] (8) at (2, 2) {\scriptsize 1};
    
    \draw (0) -- (1);
    \draw (1) -- (2);
    \draw (3) -- (4);
    \draw (4) -- (5);
    \draw (6) -- (7);
    \draw (7) -- (8);
    \draw (0) -- (3);
    \draw (3) -- (6);
    \draw (1) -- (4);
    \draw (4) -- (7);
    \draw (2) -- (5);
    \draw (5) -- (8);

    \draw[mygreen] (0) to[bend right] node[below, inner sep=1pt] {\scriptsize 1} (2);
    \draw[mygreen] (0) to[bend left] node[left, inner sep=1pt] {\scriptsize 1} (3);
    \draw[mygreen] (2) to node[pos=0.7, below, inner sep=1pt] {\scriptsize 3} (3);
    \draw[mygreen] (3) to node[pos=0.6, below, inner sep=1pt] {\scriptsize -2} (7);
    \draw[mygreen] (7) to[bend left] node[above, inner sep=1pt] {\scriptsize 2} (8);
\end{tikzpicture}
    \hfill
    \begin{tikzpicture}[baseline=0]
    \node[separator, text width=8pt] (0) at (0, 0) {\scriptsize -4};
    \node[separator, text width=8pt] (1) at (1, 0) {\scriptsize 2};
    \node[separator, text width=8pt] (2) at (2, 0) {\scriptsize -3};
    \node[separator, text width=8pt] (3) at (0, 1) {\scriptsize -1};
    \node[separator, text width=8pt] (4) at (1, 1) {\scriptsize -1};
    \node[separator, text width=8pt] (5) at (2, 1) {\scriptsize 1};
    \node[separator, text width=8pt] (6) at (0, 2) {\scriptsize 2};
    \node[vertex, text width=8pt] (7) at (1, 2) {};
    \node[separator, text width=8pt] (8) at (2, 2) {\scriptsize -1};
    
    \draw (0) -- (1);
    \draw (1) -- (2);
    \draw (3) -- (4);
    \draw (4) -- (5);
    \draw (6) -- (7);
    \draw (7) -- (8);
    \draw (0) -- (3);
    \draw (3) -- (6);
    \draw (1) -- (4);
    \draw (4) -- (7);
    \draw (2) -- (5);
    \draw (5) -- (8);

    \draw[mygreen] (0) to[bend right] node[below, inner sep=1pt] {\scriptsize 1} (2);
    \draw[mygreen] (0) to[bend left] node[left, inner sep=1pt] {\scriptsize 1} (3);
    \draw[mygreen] (2) to node[pos=0.7, below, inner sep=1pt] {\scriptsize 3} (3);
    \draw[mygreen] (3) to node[pos=0.6, below, inner sep=1pt] {\scriptsize -2} (7);
    \draw[mygreen] (7) to[bend left] node[above, inner sep=1pt] {\scriptsize 2} (8);

\end{tikzpicture}
    \hfill
    \begin{tikzpicture}[baseline=0]
    \node[vertex, text width=8pt] (0) at (0, 0) {};
    \node[separator, text width=8pt] (1) at (1, 0) {\scriptsize 2};
    \node[separator, text width=8pt] (2) at (2, 0) {\scriptsize -3};
    \node[separator, text width=8pt] (3) at (0, 1) {\scriptsize -2};
    \node[separator, text width=8pt] (4) at (1, 1) {\scriptsize -1};
    \node[separator, text width=8pt] (5) at (2, 1) {\scriptsize 1};
    \node[separator, text width=8pt] (6) at (0, 2) {\scriptsize 2};
    \node[vertex, text width=8pt] (7) at (1, 2) {};
    \node[separator, text width=8pt] (8) at (2, 2) {\scriptsize -1};
    
    \draw (0) -- (1);
    \draw (1) -- (2);
    \draw (3) -- (4);
    \draw (4) -- (5);
    \draw (6) -- (7);
    \draw (7) -- (8);
    \draw (0) -- (3);
    \draw (3) -- (6);
    \draw (1) -- (4);
    \draw (4) -- (7);
    \draw (2) -- (5);
    \draw (5) -- (8);

    \draw[mygreen] (0) to[bend right] node[below, inner sep=1pt] {\scriptsize 1} (2);
    \draw[mygreen] (0) to[bend left] node[left, inner sep=1pt] {\scriptsize 1} (3);
    \draw[mygreen] (2) to node[pos=0.7, below, inner sep=1pt] {\scriptsize 3} (3);
    \draw[mygreen] (3) to node[pos=0.6, below, inner sep=1pt] {\scriptsize -2} (7);
    \draw[mygreen] (7) to[bend left] node[above, inner sep=1pt] {\scriptsize 2} (8);
\end{tikzpicture}

    \begin{tikzpicture}[baseline=0]
    \node[vertex, text width=8pt] (0) at (0, 0) {};
    \node[separator, text width=8pt] (1) at (1, 0) {\scriptsize 1};
    \node[vertex, text width=8pt] (2) at (2, 0) {};
    \node[separator, text width=8pt] (3) at (0, 1) {\scriptsize -2};
    \node[separator, text width=8pt] (4) at (1, 1) {\scriptsize -1};
    \node[separator, text width=8pt] (5) at (2, 1) {\scriptsize 1};
    \node[separator, text width=8pt] (6) at (0, 2) {\scriptsize 2};
    \node[vertex, text width=8pt] (7) at (1, 2) {};
    \node[separator, text width=8pt] (8) at (2, 2) {\scriptsize -1};
    
    \draw (0) -- (1);
    \draw (1) -- (2);
    \draw (3) -- (4);
    \draw (4) -- (5);
    \draw (6) -- (7);
    \draw (7) -- (8);
    \draw (0) -- (3);
    \draw (3) -- (6);
    \draw (1) -- (4);
    \draw (4) -- (7);
    \draw (2) -- (5);
    \draw (5) -- (8);

    \draw[mygreen] (0) to[bend right] node[below, inner sep=1pt] {\scriptsize 1} (2);
    \draw[mygreen] (0) to[bend left] node[left, inner sep=1pt] {\scriptsize 1} (3);
    \draw[mygreen] (2) to node[pos=0.7, below, inner sep=1pt] {\scriptsize 3} (3);
    \draw[mygreen] (3) to node[pos=0.6, below, inner sep=1pt] {\scriptsize -2} (7);
    \draw[mygreen] (7) to[bend left] node[above, inner sep=1pt] {\scriptsize 2} (8);
\end{tikzpicture}
    \hfill
    \begin{tikzpicture}[baseline=0]
    \node[vertex, text width=8pt] (03) at (0, 0.5) {};
    \node[separator, text width=8pt] (1) at (1, 0) {\scriptsize -2};
    \node[vertex, text width=8pt] (2) at (2, 0) {};
    \node[separator, text width=8pt] (4) at (1, 1) {\scriptsize 1};
    \node[separator, text width=8pt] (5) at (2, 1) {\scriptsize 1};
    \node[separator, text width=8pt] (6) at (0, 2) {\scriptsize 4};
    \node[vertex, text width=8pt] (7) at (1, 2) {};
    \node[separator, text width=8pt] (8) at (2, 2) {\scriptsize -1};
    
    \draw (03) -- (1);
    \draw (1) -- (2);
    \draw (03) -- (4);
    \draw (4) -- (5);
    \draw (6) -- (7);
    \draw (7) -- (8);
    \draw (03) -- (6);
    \draw (1) -- (4);
    \draw (4) -- (7);
    \draw (2) -- (5);
    \draw (5) -- (8);

    \draw[mygreen] (03) to[bend left=10] node[pos=0.7, above, inner sep=1pt] {\scriptsize 4} (2);
    \draw[mygreen] (03) to node[pos=0.6, below, inner sep=1pt] {\scriptsize -2} (7);
    \draw[mygreen] (7) to[bend left] node[above, inner sep=1pt] {\scriptsize 2} (8);
\end{tikzpicture}
    \hfill
    \begin{tikzpicture}[baseline=0]
    \node[vertex, text width=8pt] (0123) at (0.75, 0.25) {};
    \node[separator, text width=8pt] (4) at (1, 1) {\scriptsize 1};
    \node[separator, text width=8pt] (5) at (2, 1) {\scriptsize 1};
    \node[separator, text width=8pt] (6) at (0, 2) {\scriptsize 4};
    \node[vertex, text width=8pt] (7) at (1, 2) {};
    \node[separator, text width=8pt] (8) at (2, 2) {\scriptsize -1};
    
    \draw (0123) -- (4);
    \draw (4) -- (5);
    \draw (6) -- (7);
    \draw (7) -- (8);
    \draw (0123) -- (6);
    \draw (0123) -- (4);
    \draw (4) -- (7);
    \draw (0123) -- (5);
    \draw (5) -- (8);

    \draw[mygreen] (0123) to[bend left=20] node[pos=0.7, left, inner sep=1pt] {\scriptsize -2} (7);
    \draw[mygreen] (7) to[bend left] node[above, inner sep=1pt] {\scriptsize 2} (8);
\end{tikzpicture}
    \hfill
    \begin{tikzpicture}[baseline=0]
    \node[vertex, text width=8pt] (0123) at (0.75, 0.25) {};
    \node[separator, text width=8pt] (4) at (1, 1) {\scriptsize 1};
    \node[separator, text width=8pt] (5) at (2, 1) {\scriptsize 3};
    \node[separator, text width=8pt] (6) at (0, 2) {\scriptsize 4};
    \node[vertex, text width=8pt] (78) at (1.5, 2) {};
    
    \draw (0123) -- (4);
    \draw (4) -- (5);
    \draw (6) -- (78);
    \draw (0123) -- (6);
    \draw (0123) -- (4);
    \draw (4) -- (78);
    \draw (0123) -- (5);
    \draw (5) -- (78);

    \draw[mygreen] (0123) to[bend left=40] node[pos=0.6, left, inner sep=1pt] {\scriptsize -2} (78);
\end{tikzpicture}
    \caption{Depicted on the top left is an instance of the multi-separator problem consisting of a graph (black) and a set of interactions (green).
    Each node and every interaction has a cost.
    The second graph on the top depicts the initial feasible solution of \Cref{algo:gss}, i.e.~all nodes are in the separator, and the potential of each node is precisely the negative of its cost.
    The subsequent graphs illustrate the iterations of \Cref{algo:gss}. 
    In each iteration, one node is removed from the separator, all neighboring nodes that are not in the separator are contracted, and the potentials of all nodes in the separator are updated.
    The algorithm terminates after six iterations, as all nodes in the separator have a non-negative potential.}
    \label{fig:gss}
\end{figure}
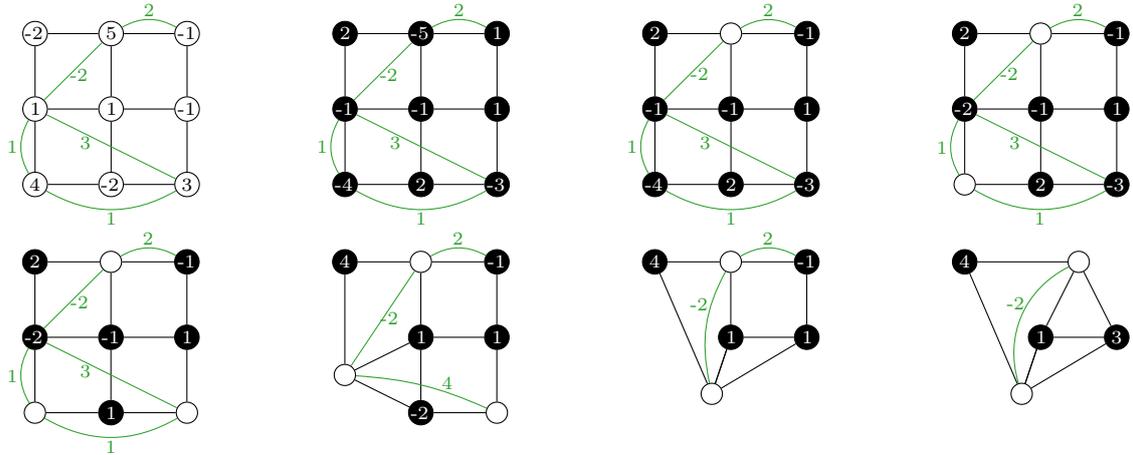

In this section, we discuss \Cref{algo:gss}, a local search algorithm we call \emph{greedy separator shrinking} (GSS) that starts with the separator $S = V$ containing all nodes and, in every iteration, removes from the separator one node so as to reduce the cost maximally.
The operations of this algorithm are shown for one example in \Cref{fig:gss}.

The algorithm maintains two graphs, $(V^t, E^t)$ and $(V^t, F^t)$ with edge costs $c^t: F^t \to \R$, a separator $S^t$, and a function $p^t \colon S^t \to \mathbb{R}$. 
Initially, $S^0 = V^0 = V$ and $E^0 = E$ and $F^0 = F$ and $p^0_v = -c_v$ for all $v \in V$. 
In every iteration $t$, the set $V^t$ contains all nodes in the separator $S^t$ and, in addition, all components of the subgraph of $(V, E)$ induced by $V \setminus S^t$.
When a node $v$ is removed from the separator (\Cref{line:gss-shrink-separator}) the subsequent graphs $(V^{t+1}, E^{t+1})$ and $(V^{t+1}, F^{t+1})$ are obtained by contracting in $(V^t, E^t)$ and $(V^t, F^t)$ the set containing $v$ and all neighbors of $v$ that are not elements of the separator into a new node $v'$ (\Cref{line:gss-merged-component} to \Cref{line:gss-interaction-contraction}).
The contraction may result in parallel edges which are merged into one edge (\Cref{line:gss-edge-contraction}). 
Analogously, parallel interactions are merged into one interaction whose costs is the sum of the costs of the merged interactions (\Cref{line:gss-interaction-contraction} to \Cref{line:gss-compute-contracted-cost}).
Subsequently, the potentials of the nodes in $S^{t+1}$ are computed with respect to the contracted graphs $(V^{t+1}, E^{t+1})$, $(V^{t+1},F^{t+1})$, and $c^{t-1}$ (\Cref{line:gss-update-potentials-loop} to \Cref{line:gss-unchanged-potential}).
By definition of the potential in \eqref{eq:obj-diff}, the potentials of all nodes that are not adjacent to the new component remain unchanged (\Cref{line:gss-unchanged-potential}).
The potentials of the nodes adjacent to $v'$ are recomputed according to \eqref{eq:obj-diff} (\Cref{line:gss-recompute-potential}).

Let us compare greedy separator shrinking (GSS) for the multi-separator problem to greedy additive edge contraction (GAEC) as defined by \citet{keuper2015efficient} for the (lifted) multicut problem:
In each iteration of GSS, one node is removed from the separator and thus, an unconstrained number of components become connected to form one component.
In each iteration of GAEC, one edge is contracted and thus, precisely two components are joined to become one component. 
Moreover, GSS removes the node with the most negative potential, while GAEC contracts the edge with the largest positive cost.
In GSS, the computation of the node potentials of the graph obtained by removing one node from the separator requires the operations in \Cref{line:gss-update-potentials-loop} to \Cref{line:gss-unchanged-potential} that are illustrated also in \Cref{fig:gss-potential-update}.
In GAEC, the computation of the edge costs of the graph obtained by contracting an edge consists in simply summing the costs of parallel edges.

\begin{figure}
    \centering
    \begin{tikzpicture}
    \node[vertex, inner sep=1pt] (0) at (0, 0) {\scriptsize 1};
    \node[separator, inner sep=1pt] (1) at (1, 0) {\scriptsize 2};
    \node[vertex, inner sep=1pt] (2) at (2, 0) {\scriptsize 3};
    \node[separator, inner sep=1pt] (3) at (3, 0) {\scriptsize 4};
    \node[vertex, inner sep=1pt] (4) at (4, 0) {\scriptsize 5};

    \draw (0) -- (1);
    \draw (1) -- (2);
    \draw (2) -- (3);
    \draw (3) -- (4);
    \draw[mygreen] (0) to[bend left] (1);
    \draw[mygreen] (0) to[bend left] (2);
    \draw[mygreen] (0) to[bend left] (3);
    \draw[mygreen] (0) to[bend left] (4);
\end{tikzpicture}
    \caption{Depicted above are a graph (in black) and a set of interactions (in green). 
    Let $S^t=\{2, 4\}$ be the separator at iteration $t$ of the GSS algorithm.
    The current potential of Node $2$ is $p^t_2 = -c_2 - c_{\{1,2\}} - c_{\{1, 3\}}$.
    If Node $4$ is removed from the separator in iteration $t$, the potential of Node $2$ becomes 
    $p^{t+1}_2 = -c_2 - c_{\{1,2\}} - c_{\{1, 3\}} - c_{\{1, 4\}} - c_{\{1, 5\}}$.}
    \label{fig:gss-potential-update}.
\end{figure}
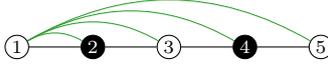

\begin{lemma}\label{lem:gss-worst-case-runtime}
    \Cref{algo:gss} has a worst case time complexity of $\mathcal{O}(|V|^2 \log|V| \, d^2)$ where $d$ is the degree of the input graph $G = (V,E)$.
\end{lemma}

\begin{proof}
    For $v \in V$, let $\deg_{(V,E)}(v)$ be the degree of $v$ in the graph $(V,E)$.
    Let $d = \max_{v \in V} \deg_{(V,E)}(v)$ be the degree of the graph $(V,E)$.
    Observe for all $t$ and all $v \in S^t$ that $\deg_{(V^t,E^t)}(v) \leq \deg_{(V,E)}(v)$ because the contractions that lead to the graph $(V^t,E^t)$ only involve nodes in $V^t \setminus S^t$.
    In contrast, the degree of the nodes in $V^t \setminus S^t$ is not bounded in such a way; in fact, it is in $\mathcal{O}(n)$.
    The same observation holds true for the interaction graphs $(V,F)$ and $(V^t,F^t)$.

    The algorithm terminates after at most $|V|$ iterations.
    The runtime of each iteration is dominated by the time for recomputing the potentials of the nodes that are adjacent to the newly formed component (\Cref{line:gss-recompute-potential}).
    By the above observation, the number of nodes for which the potential needs to be recomputed is $\mathcal{O}(|V|)$.
    Furthermore, the degree of each node $u \in S^{t+1}$ is bounded by $d$.
    Therefore, the size of the set $N$ (\Cref{line:gss-compute-neighborhood}) that contains $u$ and the neighbors of $u$ in $(V^{t+1}, E^{t+1})$ that are not in the separator is bounded by $d + 1$.
    To compute the potential in \Cref{line:gss-recompute-potential}, all interactions between any two nodes in $N$ need to be identified.
    As the set $N$ contains nodes that are not in the separator, the degree of these nodes is not bounded by $d$ but is instead in $\mathcal{O}(n)$.
    For each of the nodes $w \in N$, it needs to be checked if any of the other nodes $w' \in N \setminus \{w\}$ is adjacent to $w$ in the interaction graph $(V^{t+1},F^{t+1})$.
    By maintaining for each node a reference to all its neighbors in the interaction graph in a sorted array, this can be checked in time $\mathcal{O}(\log|V|)$.
    Consequently, all interactions in $\{f \in F^{t+1} \mid f \subseteq N\}$ can be identified in time $\mathcal{O}(d^2\log|V|)$.
\end{proof}

The \textsc{c++} implementation of \Cref{algo:gss} in the supplement utilizes a priority queue of all nodes in the separator where the node with the smallest potential has the highest priority.
In order to process updates of priorities (increases and decreases), we store a version number for each node and define as entries of the priority queue pairs consisting of a node and a version number.
Initially, the version numbers of all nodes are zero.
Whenever the priority of a node changes, we increment the version number associated with that node and insert into the priority queue an additional element consisting of that node and the updated version number.
This facilitates updates of the potential of a node in constant time in \Cref{line:gss-recompute-potential}.
Whenever we pop the highest-priority node from the queue (\Cref{line:gss-get-argmin}), we check whether its version number is the highest ever associated with that node. 
If so, we process the node.
Otherwise, we discard that entry and pop the next.
This implementation does not improve over the worst case time complexity of \Cref{lem:gsg-worstcase-runtime}.
Measurements of the absolute runtime for specific instances are reported in \Cref{fig:runtime-analysis}, \Cref{sec:experiments}.

\begin{algorithm}\small
    \DontPrintSemicolon
    \SetAlgoNoEnd
    \LinesNumbered
    \KwData{Graph $G=(V,E)$, interactions $F \subseteq \binom{V}{2}$, costs $c:V \cup F \to \R$}
    \KwResult{separator $S \subseteq V$}
    $V^0 := V$, $E^0 := E$, $F^0 := F$, $c^0 := c$, $S^0 := V$ \;
    $p^0_v := -c^0_v \; \forall v \in V^0$ \;
    $t := 0$\;
    \While{$S^t \neq \emptyset$}{
        $v \in \argmin_{s \in S^t} p^t_s$\; \label{line:gss-get-argmin}
        \If{$p^t_v > 0$}
        {
            \textbf{break}\;
        }
        $S^{t+1} := S^t \setminus \{v\}$\; \label{line:gss-shrink-separator}
        $C := \{v\} \cup \{u \in V^t \mid \{u,v\} \in E^t \text{ and } u \notin S \}$ \; \label{line:gss-merged-component}
        create new node $v'$\;
        $V^{t+1} := V^t \setminus C \cup \{v'\}$\;
        $E^{t+1} := E^t \setminus \{\{u, w\} \in E^t \mid u \in C \text{ or } w \in C\} \cup \{\{v',u\} \mid \exists w \in C: \{u,w\} \in E^t \}$\; \label{line:gss-edge-contraction}
        $F^{t+1} := F^t \setminus \{\{u, w\} \in F^t \mid u \in C \text{ or } w \in C\} \cup \{\{v',u\} \mid \exists w \in C: \{u,w\} \in F^t \}$\; \label{line:gss-interaction-contraction}
        \For{$\{u,w\} \in F^{t+1}$}{ \label{line:gss-contracted-costs-loop}
            \If{$v'\notin \{u,w\}$}{
                $c^{t+1}_{\{u,w\}} := c^t_{\{u,w\}}$\;
            }
            \Else{
                $c^{t+1}_{\{v',w\}} := \sum_{u \in C: \{u,w\} \in F^t} c^t_{\{u,w\}}$\;
                \label{line:gss-compute-contracted-cost}
            }
        }
        \For{$u \in S^{t+1}$}{ \label{line:gss-update-potentials-loop}
            \If{$\{v',u\} \in E^{t+1}$}{
                $N := \{u\} \cup \{w \in V^{t+1} \setminus S^{t+1} \mid \{u,w\} \in E^{t+1}\}$\; \label{line:gss-compute-neighborhood}
                $p^{t+1}_u := -c_u - \sum_{f \in F^{t+1}, f \subseteq N} c_f$ \; \label{line:gss-recompute-potential}
            }
            \Else{
                $p^{t+1}_u := p^t_u$\; \label{line:gss-unchanged-potential}
            }   
        }
        $t := t+1$\;
    }
    \KwRet{$S^t$}
    \caption{Greedy Component Shrinking (GSS)}
    \label{algo:gss}
\end{algorithm}
\subsection{Greedy Separator Growing}\label{sec:gsg}
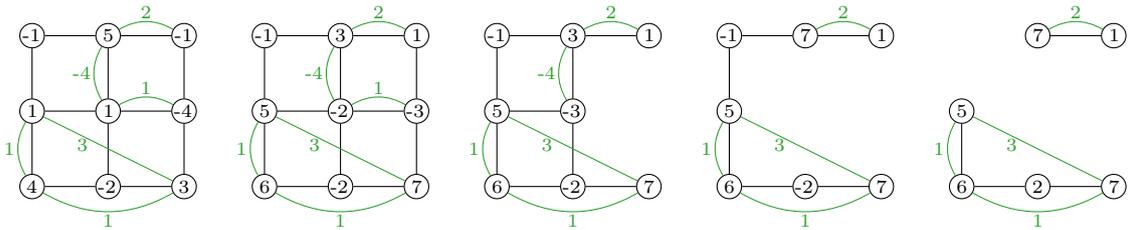
\begin{figure}
    \centering
    \begin{tikzpicture}[baseline=0]
    \node[vertex, text width=8pt] (0) at (0, 0) {\scriptsize 4};
    \node[vertex, text width=8pt] (1) at (1, 0) {\scriptsize -2};
    \node[vertex, text width=8pt] (2) at (2, 0) {\scriptsize 3};
    \node[vertex, text width=8pt] (3) at (0, 1) {\scriptsize 1};
    \node[vertex, text width=8pt] (4) at (1, 1) {\scriptsize 1};
    \node[vertex, text width=8pt] (5) at (2, 1) {\scriptsize -4};
    \node[vertex, text width=8pt] (6) at (0, 2) {\scriptsize -1};
    \node[vertex, text width=8pt] (7) at (1, 2) {\scriptsize 5};
    \node[vertex, text width=8pt] (8) at (2, 2) {\scriptsize -1};
    
    \draw (0) -- (1);
    \draw (1) -- (2);
    \draw (3) -- (4);
    \draw (4) -- (5);
    \draw (6) -- (7);
    \draw (7) -- (8);
    \draw (0) -- (3);
    \draw (3) -- (6);
    \draw (1) -- (4);
    \draw (4) -- (7);
    \draw (2) -- (5);
    \draw (5) -- (8);

    \draw[mygreen] (0) to[bend right] node[below, inner sep=1pt] {\scriptsize 1} (2);
    \draw[mygreen] (0) to[bend left] node[left, inner sep=1pt] {\scriptsize 1} (3);
    \draw[mygreen] (2) to node[pos=0.7, below, inner sep=1pt] {\scriptsize 3} (3);
    \draw[mygreen] (4) to[bend left] node[above, inner sep=1pt] {\scriptsize 1} (5);
    \draw[mygreen] (4) to[bend left] node[left, inner sep=1pt] {\scriptsize -4} (7);
    \draw[mygreen] (7) to[bend left] node[above, inner sep=1pt] {\scriptsize 2} (8);
\end{tikzpicture}
    \hfill
    \begin{tikzpicture}[baseline=0]
    \node[vertex, text width=8pt] (0) at (0, 0) {\scriptsize 6};
    \node[vertex, text width=8pt] (1) at (1, 0) {\scriptsize -2};
    \node[vertex, text width=8pt] (2) at (2, 0) {\scriptsize 7};
    \node[vertex, text width=8pt] (3) at (0, 1) {\scriptsize 5};
    \node[vertex, text width=8pt] (4) at (1, 1) {\scriptsize -2};
    \node[vertex, text width=8pt] (5) at (2, 1) {\scriptsize -3};
    \node[vertex, text width=8pt] (6) at (0, 2) {\scriptsize -1};
    \node[vertex, text width=8pt] (7) at (1, 2) {\scriptsize 3};
    \node[vertex, text width=8pt] (8) at (2, 2) {\scriptsize 1};
    
    \draw (0) -- (1);
    \draw (1) -- (2);
    \draw (3) -- (4);
    \draw (4) -- (5);
    \draw (6) -- (7);
    \draw (7) -- (8);
    \draw (0) -- (3);
    \draw (3) -- (6);
    \draw (1) -- (4);
    \draw (4) -- (7);
    \draw (2) -- (5);
    \draw (5) -- (8);

    \draw[mygreen] (0) to[bend right] node[below, inner sep=1pt] {\scriptsize 1} (2);
    \draw[mygreen] (0) to[bend left] node[left, inner sep=1pt] {\scriptsize 1} (3);
    \draw[mygreen] (2) to node[pos=0.7, below, inner sep=1pt] {\scriptsize 3} (3);
    \draw[mygreen] (4) to[bend left] node[above, inner sep=1pt] {\scriptsize 1} (5);
    \draw[mygreen] (4) to[bend left] node[left, inner sep=1pt] {\scriptsize -4} (7);
    \draw[mygreen] (7) to[bend left] node[above, inner sep=1pt] {\scriptsize 2} (8);
\end{tikzpicture}
    \hfill
    \begin{tikzpicture}[baseline=0]
    \node[vertex, text width=8pt] (0) at (0, 0) {\scriptsize 6};
    \node[vertex, text width=8pt] (1) at (1, 0) {\scriptsize -2};
    \node[vertex, text width=8pt] (2) at (2, 0) {\scriptsize 7};
    \node[vertex, text width=8pt] (3) at (0, 1) {\scriptsize 5};
    \node[vertex, text width=8pt] (4) at (1, 1) {\scriptsize -3};
    \node[vertex, text width=8pt] (6) at (0, 2) {\scriptsize -1};
    \node[vertex, text width=8pt] (7) at (1, 2) {\scriptsize 3};
    \node[vertex, text width=8pt] (8) at (2, 2) {\scriptsize 1};
    
    \draw (0) -- (1);
    \draw (1) -- (2);
    \draw (3) -- (4);
    \draw (6) -- (7);
    \draw (7) -- (8);
    \draw (0) -- (3);
    \draw (3) -- (6);
    \draw (1) -- (4);
    \draw (4) -- (7);

    \draw[mygreen] (0) to[bend right] node[below, inner sep=1pt] {\scriptsize 1} (2);
    \draw[mygreen] (0) to[bend left] node[left, inner sep=1pt] {\scriptsize 1} (3);
    \draw[mygreen] (2) to node[pos=0.7, below, inner sep=1pt] {\scriptsize 3} (3);
    \draw[mygreen] (4) to[bend left] node[left, inner sep=1pt] {\scriptsize -4} (7);
    \draw[mygreen] (7) to[bend left] node[above, inner sep=1pt] {\scriptsize 2} (8);
\end{tikzpicture}
    \hfill
    \begin{tikzpicture}[baseline=0]
    \node[vertex, text width=8pt] (0) at (0, 0) {\scriptsize 6};
    \node[vertex, text width=8pt] (1) at (1, 0) {\scriptsize -2};
    \node[vertex, text width=8pt] (2) at (2, 0) {\scriptsize 7};
    \node[vertex, text width=8pt] (3) at (0, 1) {\scriptsize 5};
    \node[vertex, text width=8pt] (6) at (0, 2) {\scriptsize -1};
    \node[vertex, text width=8pt] (7) at (1, 2) {\scriptsize 7};
    \node[vertex, text width=8pt] (8) at (2, 2) {\scriptsize 1};
    
    \draw (0) -- (1);
    \draw (1) -- (2);
    \draw (6) -- (7);
    \draw (7) -- (8);
    \draw (0) -- (3);
    \draw (3) -- (6);

    \draw[mygreen] (0) to[bend right] node[below, inner sep=1pt] {\scriptsize 1} (2);
    \draw[mygreen] (0) to[bend left] node[left, inner sep=1pt] {\scriptsize 1} (3);
    \draw[mygreen] (2) to node[pos=0.7, below, inner sep=1pt] {\scriptsize 3} (3);
    \draw[mygreen] (7) to[bend left] node[above, inner sep=1pt] {\scriptsize 2} (8);
\end{tikzpicture}
    \hfill
    \begin{tikzpicture}[baseline=0]
    \node[vertex, text width=8pt] (0) at (0, 0) {\scriptsize 6};
    \node[vertex, text width=8pt] (1) at (1, 0) {\scriptsize 2};
    \node[vertex, text width=8pt] (2) at (2, 0) {\scriptsize 7};
    \node[vertex, text width=8pt] (3) at (0, 1) {\scriptsize 5};
    \node[vertex, text width=8pt] (7) at (1, 2) {\scriptsize 7};
    \node[vertex, text width=8pt] (8) at (2, 2) {\scriptsize 1};
    
    \draw (0) -- (1);
    \draw (1) -- (2);
    \draw (7) -- (8);
    \draw (0) -- (3);

    \draw[mygreen] (0) to[bend right] node[below, inner sep=1pt] {\scriptsize 1} (2);
    \draw[mygreen] (0) to[bend left] node[left, inner sep=1pt] {\scriptsize 1} (3);
    \draw[mygreen] (2) to node[pos=0.7, below, inner sep=1pt] {\scriptsize 3} (3);
    \draw[mygreen] (7) to[bend left] node[above, inner sep=1pt] {\scriptsize 2} (8);
\end{tikzpicture}
    \caption{Depicted on the left is an instance of the multi-separator problem consisting of a graph (black) and a set of interactions (green).
    Each node and every interaction has a costs.
    The second graph depicts the starting solution of \Cref{algo:gsg}, i.e.~no nodes are in the separator, and all nodes form one large component.
    The potential of each node is the cost of that node plus the sum of the costs of all interactions adjacent to that node.
    The subsequent graphs illustrate the iterations of \Cref{algo:gsg}. 
    In each iteration, a node with minimal negative potential is added to the separator, i.e.~deleted from the graph, and the potentials of the remaining nodes are updated.
    The algorithm terminates after the three iterations, as all remaining nodes have non-negative potential.
    In the last iteration, the node in the bottom center has the most negative potential of $-2$. 
    However, it is a cut-node and would separate the two interactions adjacent to the bottom right node.
    Therefore, its potential is updated to $2=-2 + 1 + 3$.
    Now, the node with the most negative potential is the top left node, which is deleted from the graph.}
    \label{fig:gsg}
\end{figure}

In this section, we discuss \Cref{algo:gsg}, a local search algorithm we call \emph{greedy separator growing} (GSG) that starts with the separator $S=\emptyset$ being empty and, in every iteration, adds to the separator one node so as to reduce the cost maximally.
The operations of this algorithm are shown for one example in \Cref{fig:gsg}.

Before describing the algorithm in detail, we discuss one property informally:
By adding a node to the separator $S$, other nodes that are not in the separator can become cut-nodes of the graph induced by $V \setminus S$.
According to \eqref{eq:obj-diff}, the potential of a node $v \notin S$ is the cost of the node $c_v$ plus the costs of all interactions that are separated by $S \cup \{v\}$ and not by $S$.
If $v$ is not a cut-node, then the set of interactions that are separated by $S \cup \{v\}$ but not by $S$ is a subset of the interactions that are adjacent to $v$. 
However, if $v$ is a cut-node, then the set of interactions that are separated by $S \cup \{v\}$ but not by $S$ can contain additional interactions that are not adjacent to $v$.
As identifying cut-nodes and those interactions that would be separated by adding a cut-node to the separator is computationally expensive \citep{hopcroft1973algorithm}, the GSG algorithm checks if a node is a cut-node only before this node is added to the separator.
If it is a cut-node, then the potential of that node is recomputed by identifying all interactions that would be separated if the node was added to the separator.
By this strategy, the true potential of a node can be smaller than the potential that is known to the algorithm.
As a result, a node $v$ can be added to the separator even though there exists another node $u$ whose potential is strictly less than that of $v$, but the true potential of $u$ is not known to the algorithm.
Yet, in the special case where the costs of all interactions that are not edges in $G$, have non-negative cost, the true potential of a node cannot be smaller than the potential that is known to the algorithm. 
Therefore, in that special case, the algorithm always adds to the separator one node that reduces the cost maximally.

More specifically, GSG works as described below and as illustrated for one example in \Cref{fig:gsg}.
The algorithm maintains an induced subgraph $(V^t, E^t)$ of the input graph $(V,E)$, a subset of interactions $F^t \subseteq F$, a function $p^t: V^t \to \R$, and a map $\cn^t: F^t \to 2^{V^t}$ from each $f \in F^t$ to the subset of nodes in $V^t$ that are identified as $f$-cut-nodes in $(V^t,E^t)$.
Initially, $V^0 = V$, $E^0 = E$, $F^0 = F$, $p^0_v := c_v + \sum_{\{u,v\} \in F} c_{\{u,v\}}$ for all $v \in V$, and $\cn^0_{\{u,v\}} = \{u, v\}$ for all $\{u,v\} \in F^0$.
In every iteration $t$, the set $V^t$ contains all nodes that are not in the separator, the set $E^t$ contains the edges of the subgraph of $(V,E)$ induced by $V^t$, and the set $F^t$ contains all interactions that are not separated by $V \setminus V^t$ in $(V, E)$.
Note that the potential $p^0_v$ as computed in \Cref{line:gsg-initial-potential} does not account for the interactions that are not adjacent to $v$, for which $v$ is a cut-node.
Also, the set $\cn^0_f$ only accounts for the cut-nodes of $f$ that are adjacent to $f$.
Once a node is selected to be added to the separator, i.e.~deleted from $V^t$ (\Cref{line:gsg-select-best-potential}), the potential of that node is recomputed according to \eqref{eq:obj-diff} (\Cref{line:gsg-recompute-potential}).
For the separator $S = V \setminus V^t$, we have $F(S \cup \{v\}) \setminus F(S) = F^t(\{v\})$, i.e.~the set of interactions in $F$ that are separated by $S \cup \{v\}$ but not by $S$ in $(V,E)$ is equal to the set of interactions in $F^t$ that are separated by $\{v\}$ in $(V^t, E^t)$.
The set $F^t(\{v\})$ is computed by first computing the components of the subgraph of $(V^t,E^t)$ induced by $V^t \setminus \{v\}$ and then selecting all interactions in $F^t$ whose endpoints lie in distinct components.
This takes linear time.
For all interactions $f \in F^t(\{v\})$, the node $v$ is added to the set $\cn^t_f$ of nodes that have been identified as $f$-cut-nodes in $(V^t,E^t)$ (\Cref{line:gsg-update-cut-nodes}).
If the recomputed potential of $v$ is positive, or if $v$ is no longer the node with the smallest potential, a new node with smallest potential is selected (\Cref{line:gsg-continue}).
Otherwise, the node $v$ is added to the separator, i.e.~deleted from the graph (\Cref{line:gsg-delete-node} to \Cref{line:gsg-delete-interactions}).
All interactions $F^t(\{v\})$ that are separated by deleting $v$ from the graph can no longer be separated by deleting any other node from the graph.
The potentials of all nodes that have been identified as cut-nodes for any of the interactions in $F^t(\{v\})$ are updated accordingly (\Cref{line:gsg-copy-potential} to \Cref{line:gsg-update-potential}).
For all interactions $f \in F^{t+1}$, the nodes that have been identified as $f$-cut-nodes remain unchanged (\Cref{line:gsg-copy-cut-nodes}).

\begin{algorithm}[t]\small
    \DontPrintSemicolon
    \SetAlgoNoEnd
    \LinesNumbered
    \KwData{Connected graph $G=(V,E)$, interactions $F \subseteq \binom{V}{2}$, costs $c:V \cup F \to \R$ with $c_f \geq 0$ for $f \in F \setminus E$}
    \KwResult{separator $S \subseteq V$}
    $V^0 := V$, $E^0 := E$, $F^0 := F$ \;
    $p^0_v := c_v + \sum_{\{u,v\} \in F} c_{\{u,v\}} \; \forall v \in V^0$\; \label{line:gsg-initial-potential}
    $\cn^0_{\{u,v\}} := \{u, v\} \; \forall \{u,v\} \in F^0$\;
    $t := 0$\;
    \While{$V^t \neq \emptyset$}{
        $v \in \argmin_{u \in V^t} p^t_u$\; \label{line:gsg-select-best-potential}
        \If{$p^t_v > 0$}
        {
            \textbf{break}\;
        }
        $p^t_v := c_v + \sum_{f \in F^t(\{v\})} c_f$\; \label{line:gsg-recompute-potential}
        \For{$f \in F^t(\{v\})$}{
            $\cn^t_f := \cn^t_f \cup \{v\}$\; \label{line:gsg-update-cut-nodes}
        }
        \If{$p^t_v > 0$ or $p^t_v > \min_{u \in V^t} p^t_u$}{
            \textbf{continue}\; \label{line:gsg-continue}
        }
        $V^{t+1} := V^t \setminus \{v\}$\; \label{line:gsg-delete-node}
        $E^{t+1} := E^t \setminus \{e \in E \mid v \in E\}$\;
        $F^{t+1} := F^t \setminus F^t(\{v\})$\; \label{line:gsg-delete-interactions}
        $p^{t+1} := p^t$\; \label{line:gsg-copy-potential}
        \For{$f \in F^t(\{v\})$}{ \label{line:gsg-update-potential-outer-loop}
            \For{$u \in \cn^t_f$}{
                $p^{t+1}_u := p^{t+1}_u - c_f$\; \label{line:gsg-update-potential}
            }
        }
        \For{$f \in F^{t+1}$}{
            $\cn^{t+1}_f := \cn^t_f$\; \label{line:gsg-copy-cut-nodes}
        }
        $t := t+1$\;
    }
    \KwRet{$V \setminus V^t$}
    \caption{Greedy Separator Growing (GSG)}
    \label{algo:gsg}
\end{algorithm}

\begin{remark}
    \Cref{algo:gsg} exploits the fact that the node variables in the multi-separator problem \eqref{eq:msp} are unconstrained, i.e.~any node subset $S \subseteq V$ is a feasible solution.
    This is in contrast to the lifted multicut problem \eqref{eq:lmp} where not all edge subsets $M \subseteq E$ are feasible. 
    Thus, an analog to \Cref{algo:gsg} for \eqref{eq:lmp} that iteratively adds single edges to a set of cut edges is not guaranteed to output a feasible solution to \eqref{eq:lmp}.
\end{remark}

\begin{remark}
    \Cref{algo:gsg} can be applied to arbitrary instances of \eqref{eq:msp}, including those where not all costs of interactions $F \setminus E$ are non-negative. 
    For those, however, it is not guaranteed that, in each iteration, the node added to the separator decreases the cost maximally.
\end{remark}

\begin{lemma}\label{lem:gsg-worstcase-runtime}
    \Cref{algo:gsg} has a worst case time complexity of $\mathcal{O}(|V|^2(|V| + |E| + |F|))$.
\end{lemma}

\begin{proof}
    There are at most $|V|$ iterations in which a node is removed from the graph, i.e. added to the separator (\Cref{line:gsg-delete-node}), since each node can be removed at most once.
    However, there can be more than $|V|$ iterations as a node $v$ might be discarded for removal from the graph (\Cref{line:gsg-continue}) if the potential of that node is, after recomputing it in \Cref{line:gsg-recompute-potential}, no longer minimal.
    In the worst case, all nodes in $V^t$ are discarded once, without removing a node from the graph. 
    After that, the potentials $p_v^t$ of all nodes $v \in V^t$ are correct.
    Subsequently, the recomputation of the potential in \Cref{line:gsg-recompute-potential} will not change the potential, and thus, this node will not be discarded. 
    Overall, there are $\mathcal{O}(|V|^2)$ iterations in which Lines \ref{line:gsg-select-best-potential} to \ref{line:gsg-continue} are executed.

    The set of interactions $F^t(\{v\})$ that are separated by $\{v\}$ in $(V^t,E^t)$ can be identified by, firstly, computing the components of the subgraph of $(V^t,E^t)$ that is induced by $V^t \setminus \{v\}$ and, secondly, selecting all interactions in $F^t$ whose endpoints belong to different components.
    This can be done in time $\mathcal{O}(|V|+|E|+|F|)$, using breadth-first search for computing the components.
    Therefore, Lines \ref{line:gsg-select-best-potential} to \ref{line:gsg-copy-potential} can be executed in time $\mathcal{O}(|V| + |E| + |F|)$.
    Clearly, Lines \ref{line:gsg-update-potential-outer-loop} to \ref{line:gsg-update-potential} can be executed in time $\mathcal{O}(|V||F|)$.
    By the previous argument, these lines are executed at most $|V|$ times.
    Thus follows the claim.
\end{proof}

As for \Cref{algo:gss}, the \textsc{c++} implementation of \Cref{algo:gsg} in the supplementary material of this article utilizes a priority queue for efficiently querying a node with smallest potential.
Also this implementation does not improve over the worst case time complexity of \Cref{lem:gsg-worstcase-runtime}.
Measurements of the absolute runtime for specific instances are reported in \Cref{fig:runtime-analysis}, \Cref{sec:experiments}.

\section{Application to image segmentation}\label{sec:experiments}
\begin{figure}
    \centering
    \setlength\tabcolsep{2pt}
    \begin{tabular}{
        >{\centering}m{0.03\textwidth} 
        >{\centering}m{0.23\textwidth} 
        >{\centering}m{0.23\textwidth} 
        >{\centering}m{0.23\textwidth} 
        >{\centering\arraybackslash}m{0.23\textwidth}}
        & Synthetic Image & True Segmentation & Multi-separator & Watershed \\

        \rotatebox{90}{$t=0.0$} &
        \includegraphics[width=\linewidth]{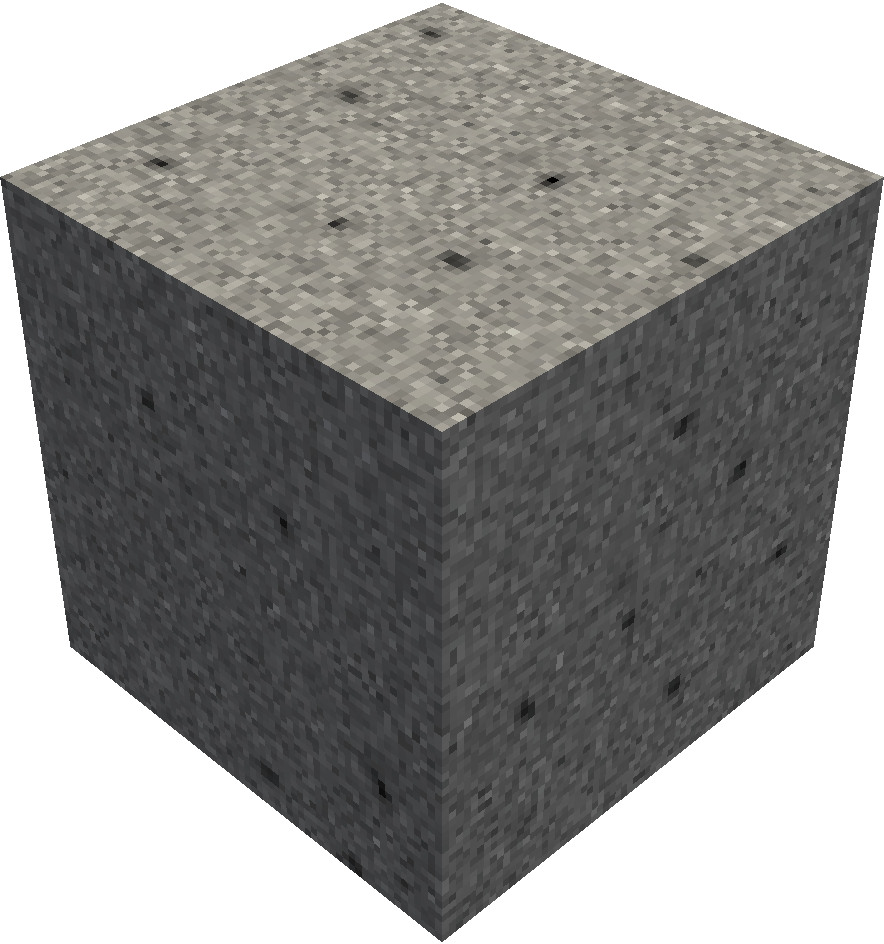} &
        \includegraphics[width=\linewidth]{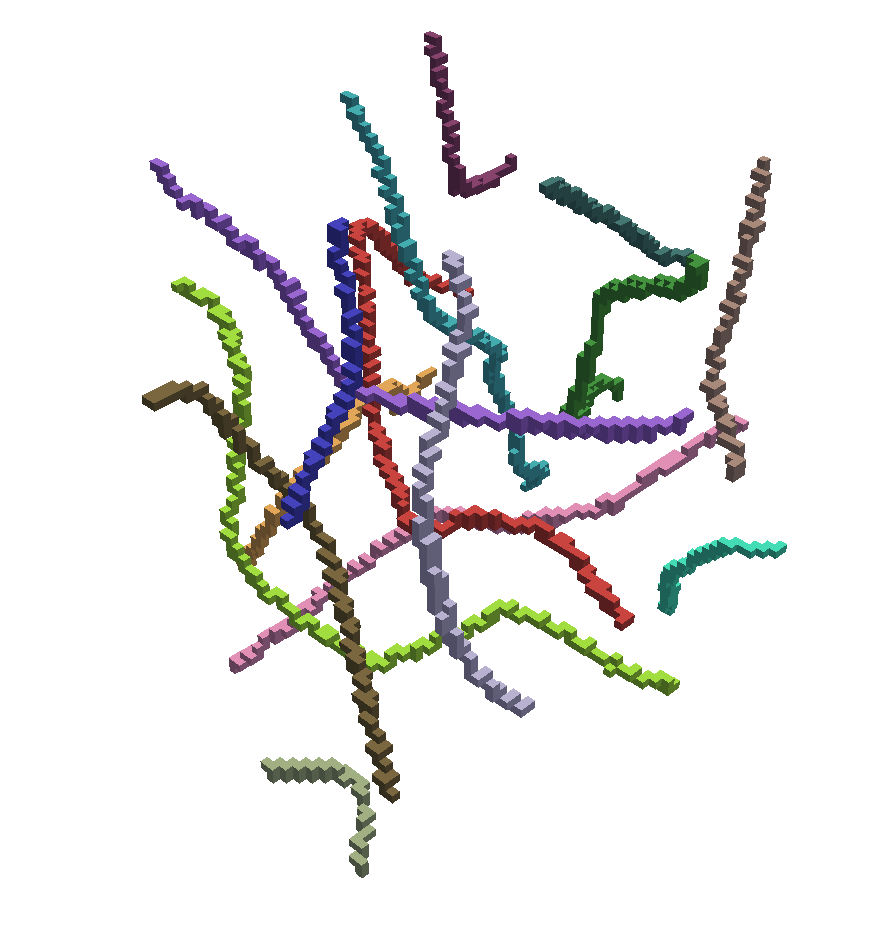} &
        \includegraphics[width=\linewidth]{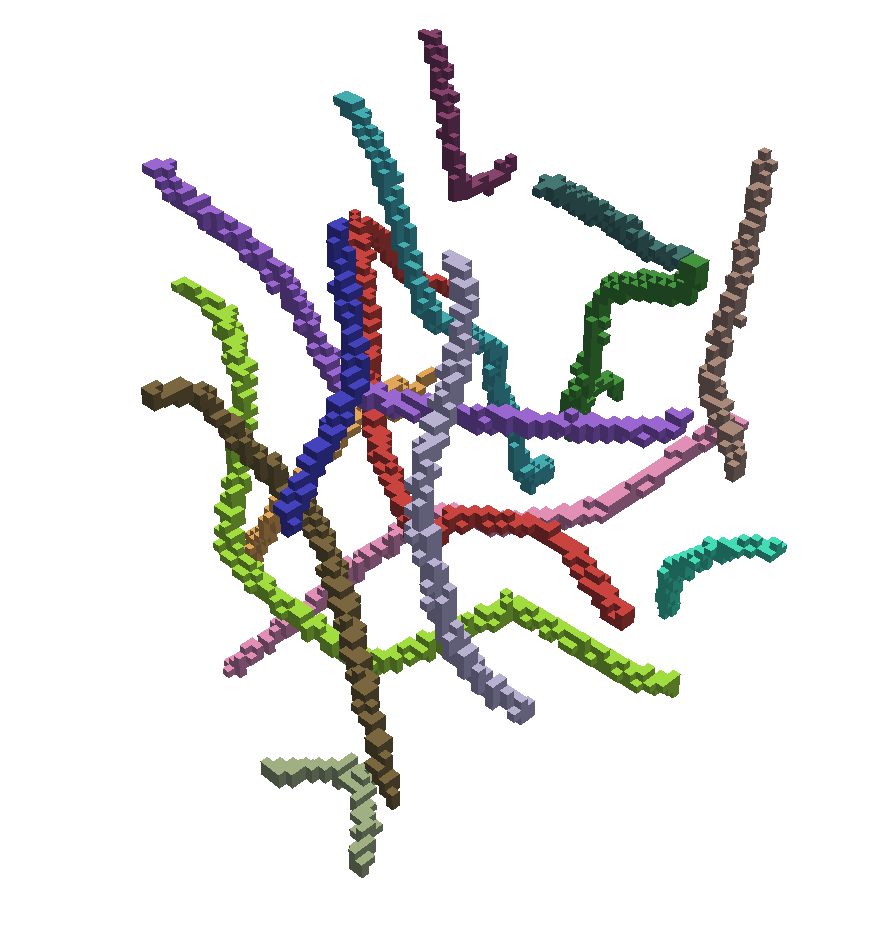} &
        \includegraphics[width=\linewidth]{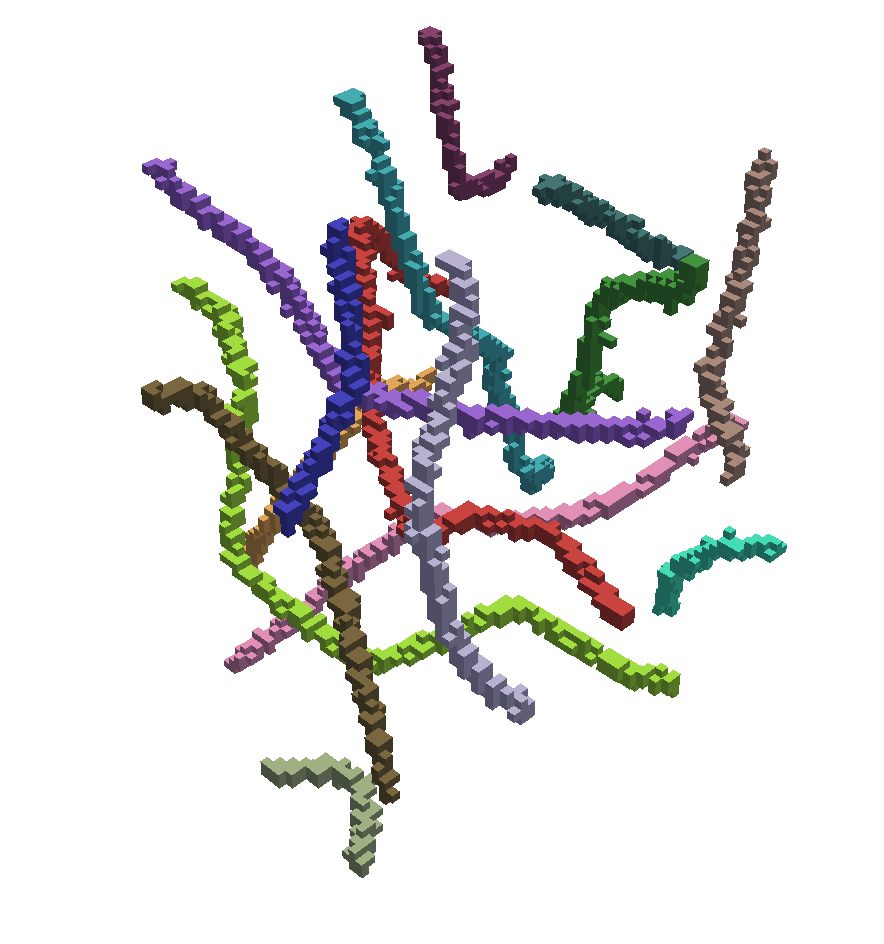} \\
        
        \rotatebox{90}{$t=0.25$} &
        \includegraphics[width=\linewidth]{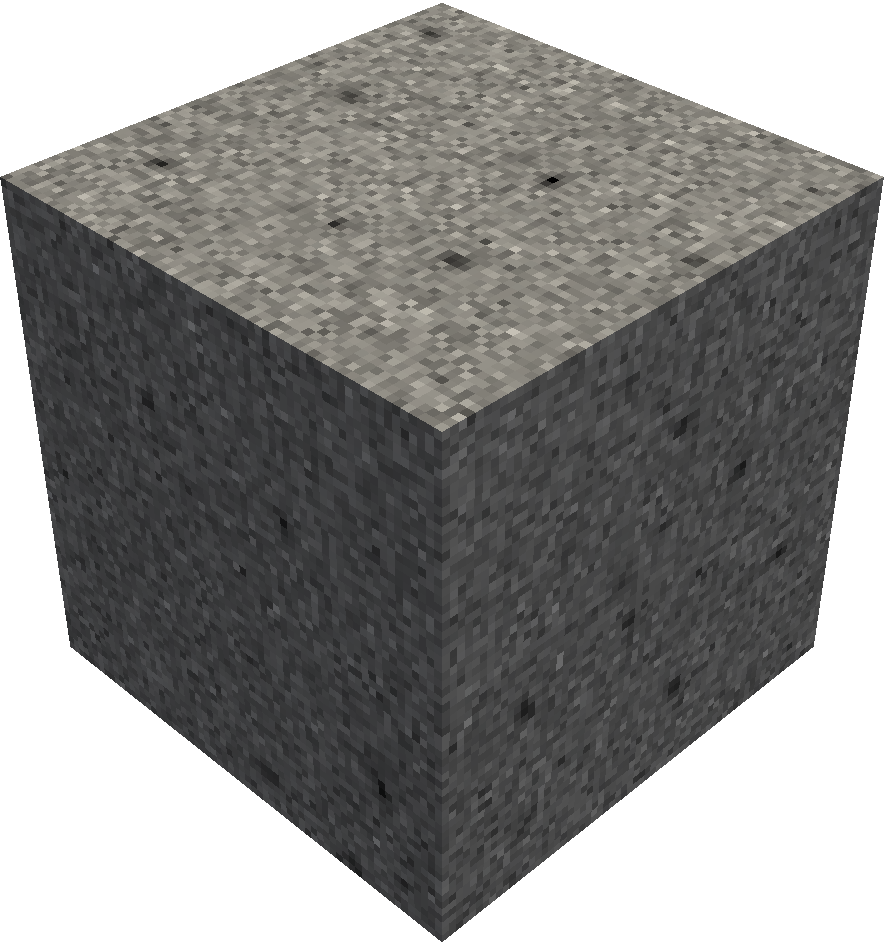} &
        \includegraphics[width=\linewidth]{images/renders/filament_1_gt.png} &
        \includegraphics[width=\linewidth]{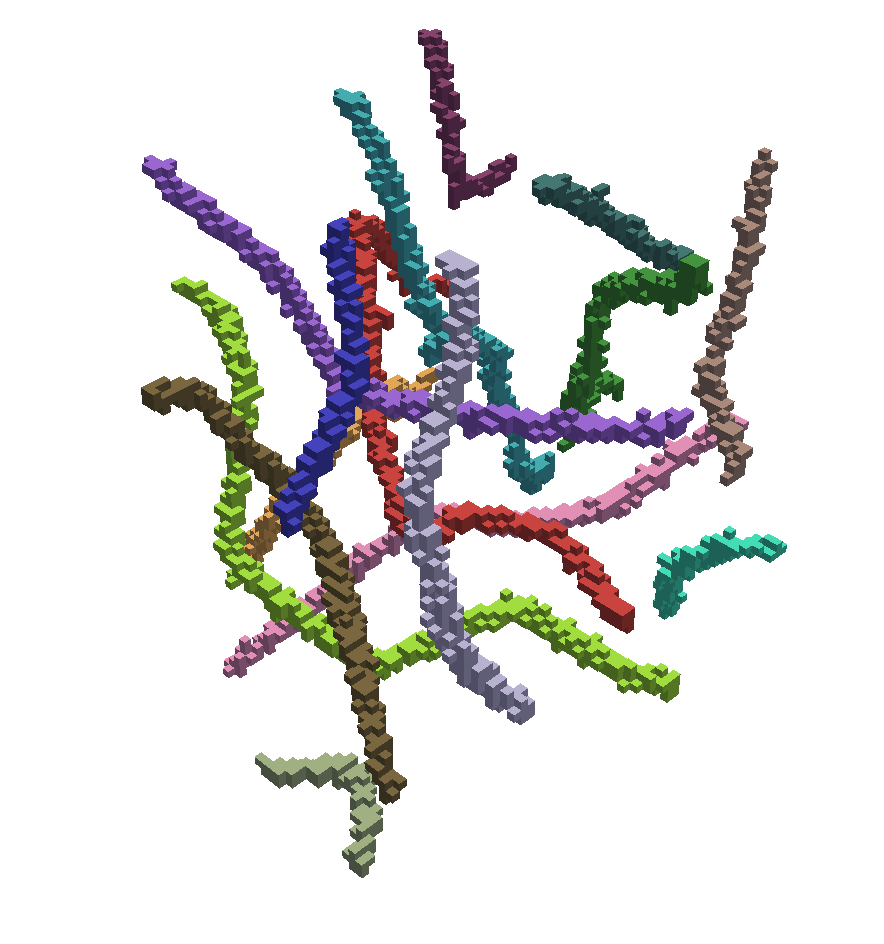} &
        \includegraphics[width=\linewidth]{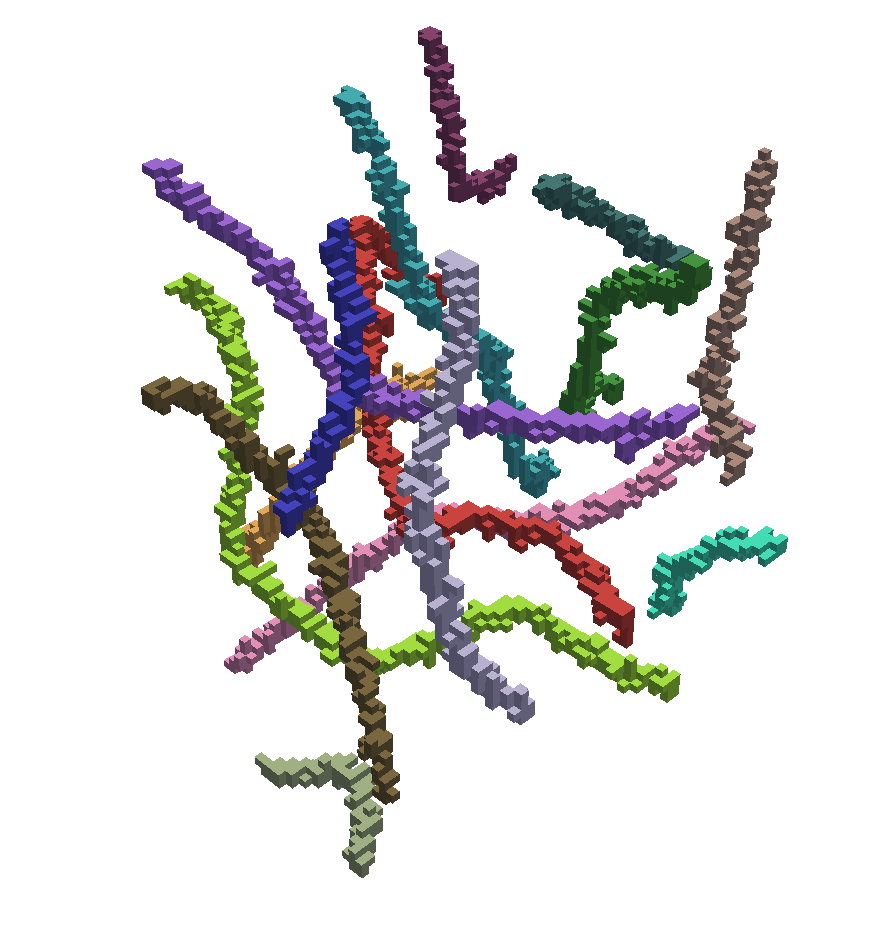} \\
        
        \rotatebox{90}{$t=0.5$} &
        \includegraphics[width=\linewidth]{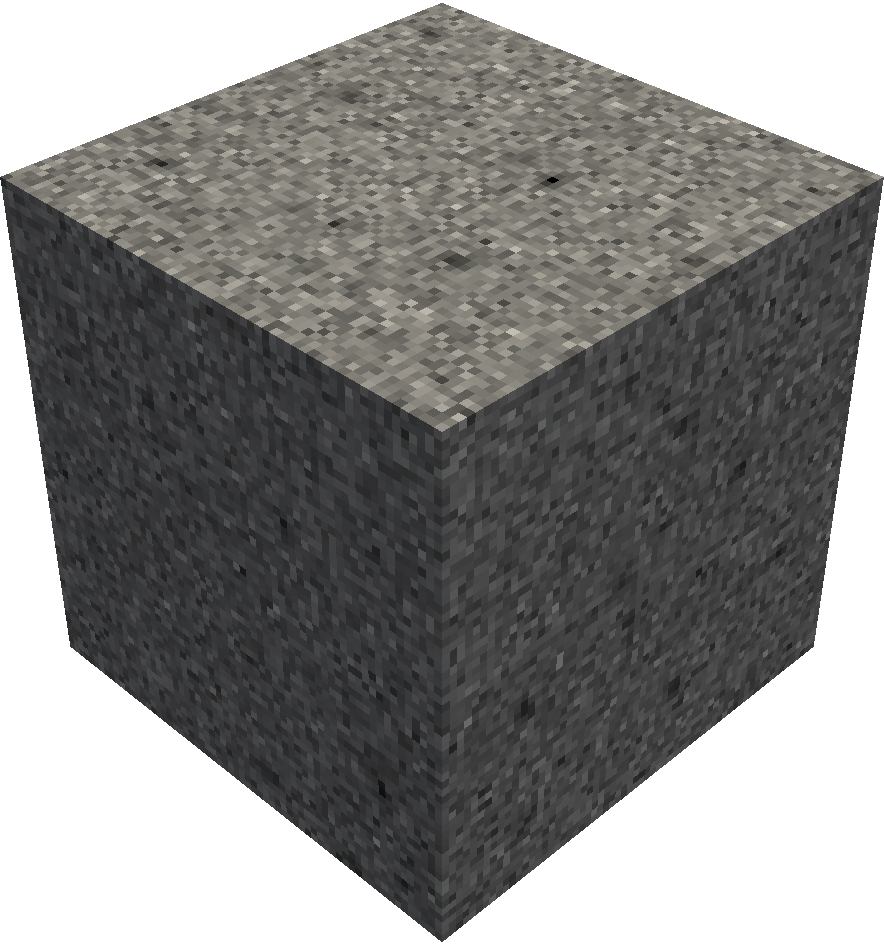} &
        \includegraphics[width=\linewidth]{images/renders/filament_1_gt.png} &
        \includegraphics[width=\linewidth]{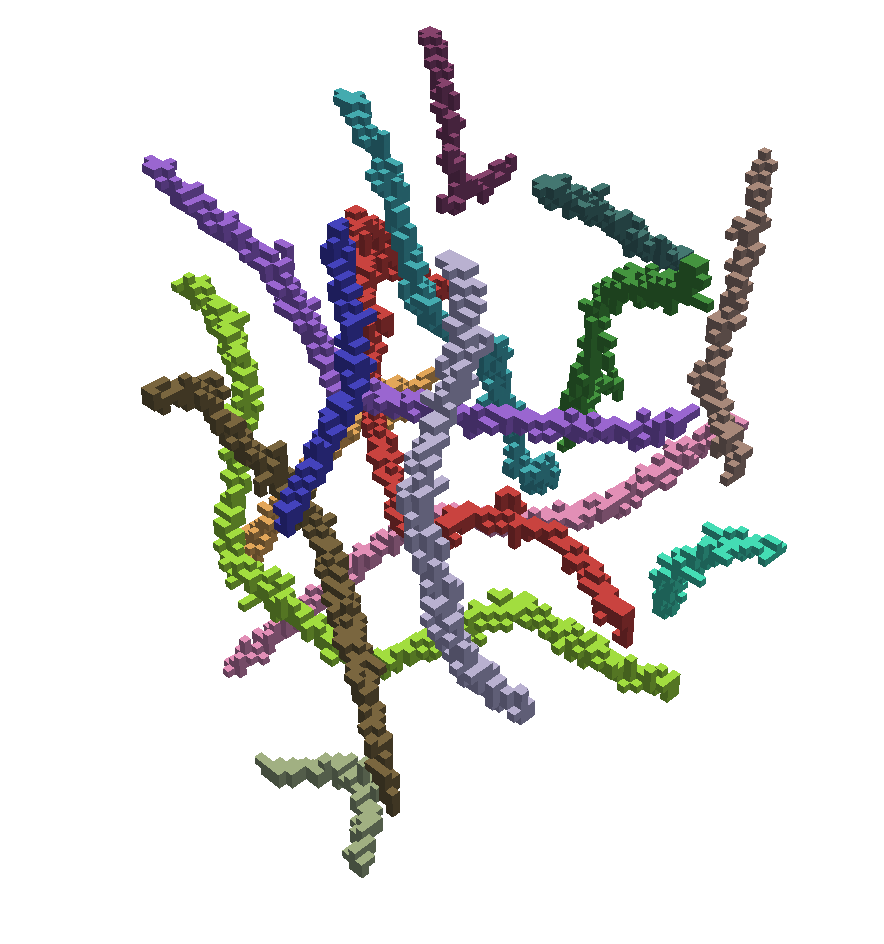} &
        \includegraphics[width=\linewidth]{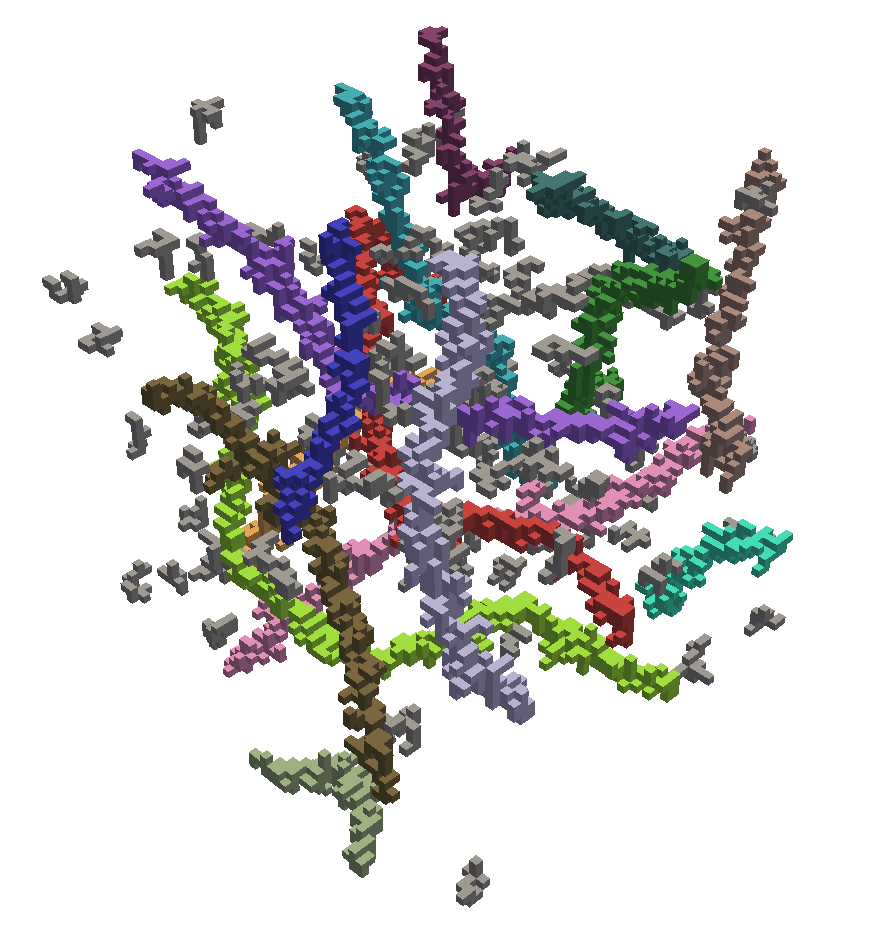} \\

        \rotatebox{90}{$t=0.75$} &
        \includegraphics[width=\linewidth]{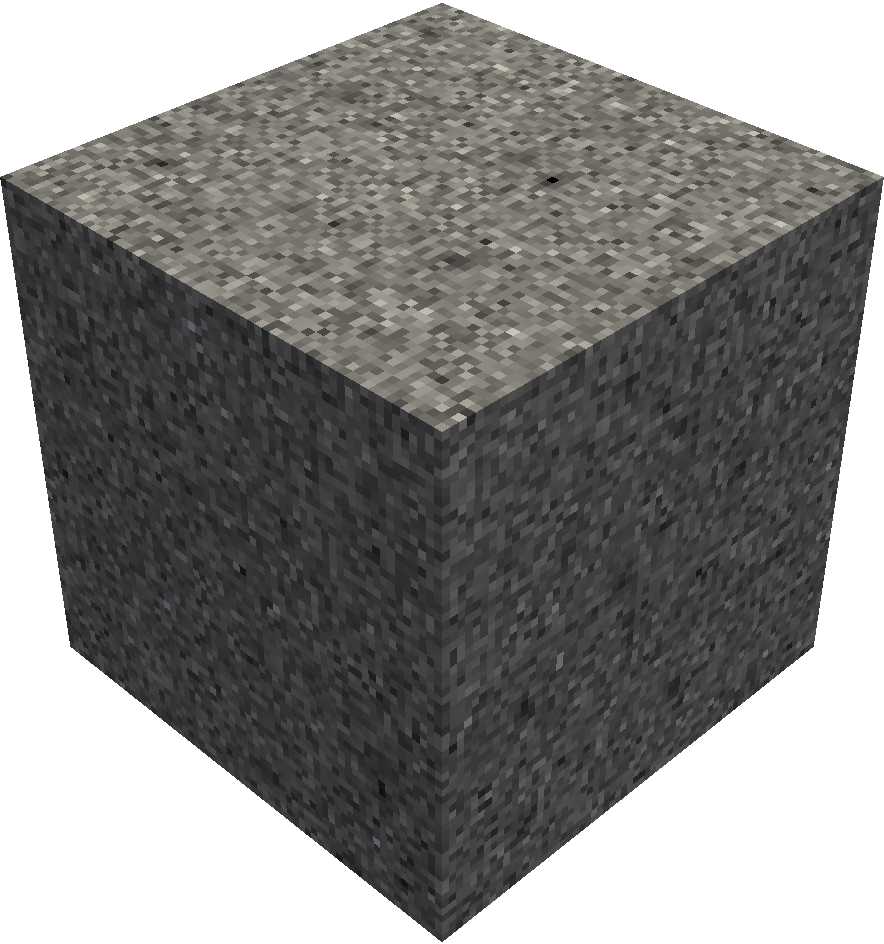} &
        \includegraphics[width=\linewidth]{images/renders/filament_1_gt.png} &
        \includegraphics[width=\linewidth]{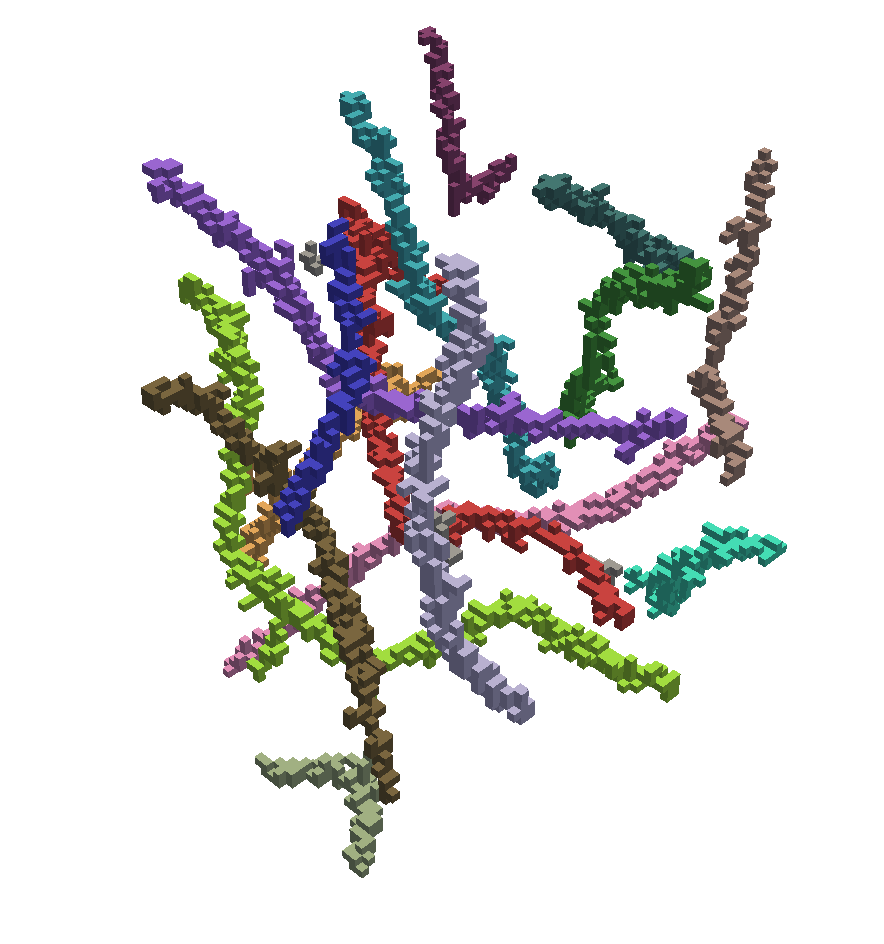} &
        \includegraphics[width=\linewidth]{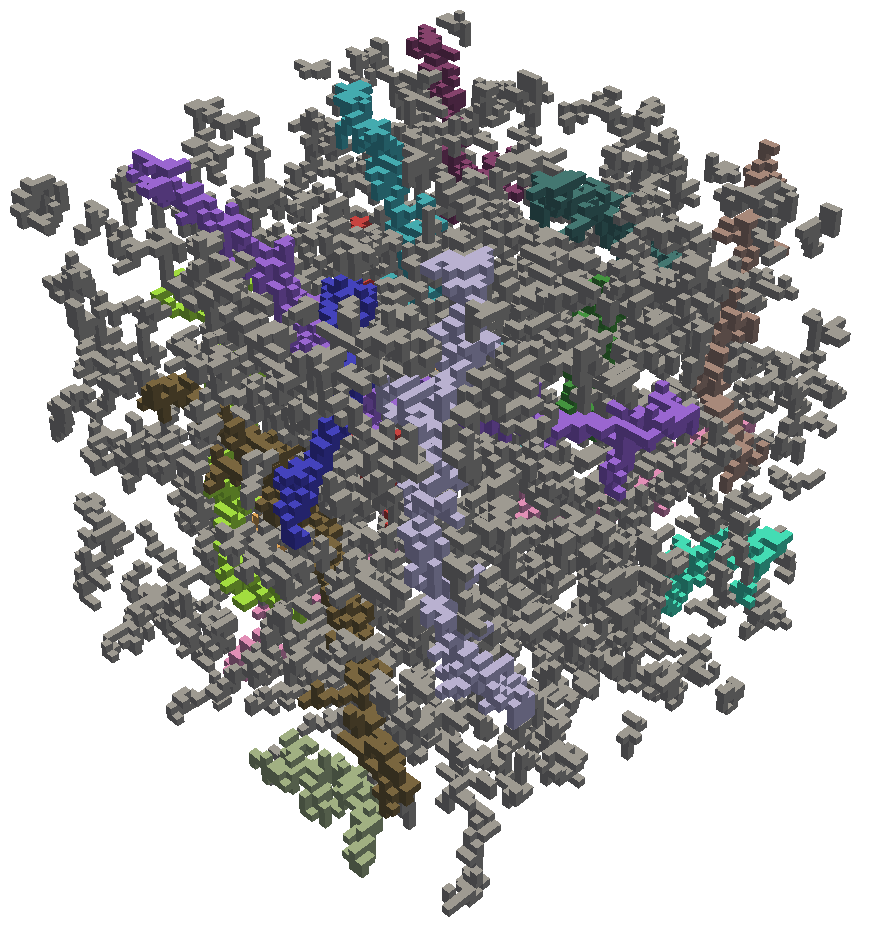} \\

        \rotatebox{90}{$t=1.0$} &
        \includegraphics[width=\linewidth]{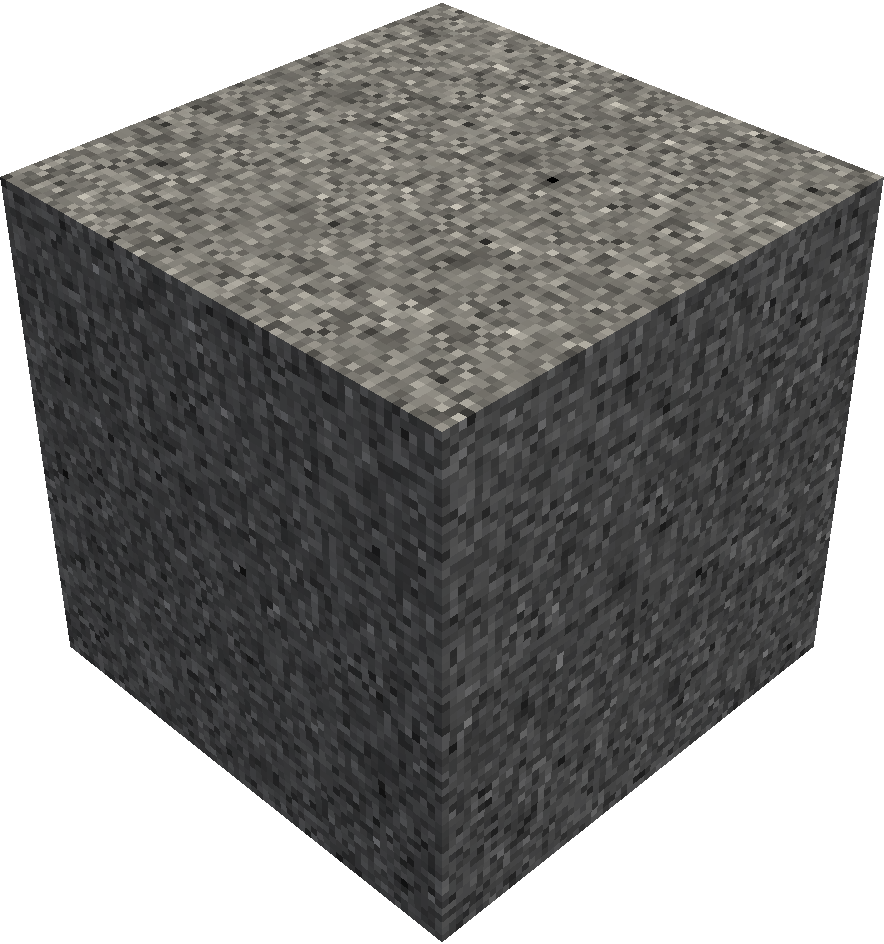} &
        \includegraphics[width=\linewidth]{images/renders/filament_1_gt.png} &
        \includegraphics[width=\linewidth]{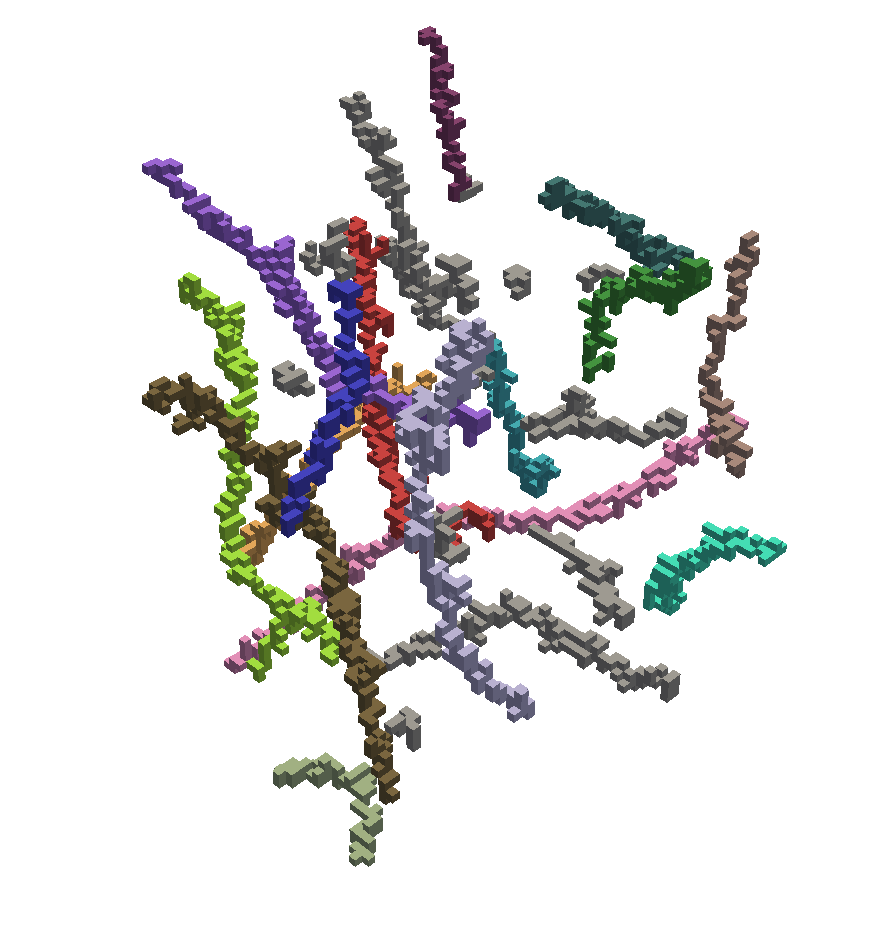} &
        \includegraphics[width=\linewidth]{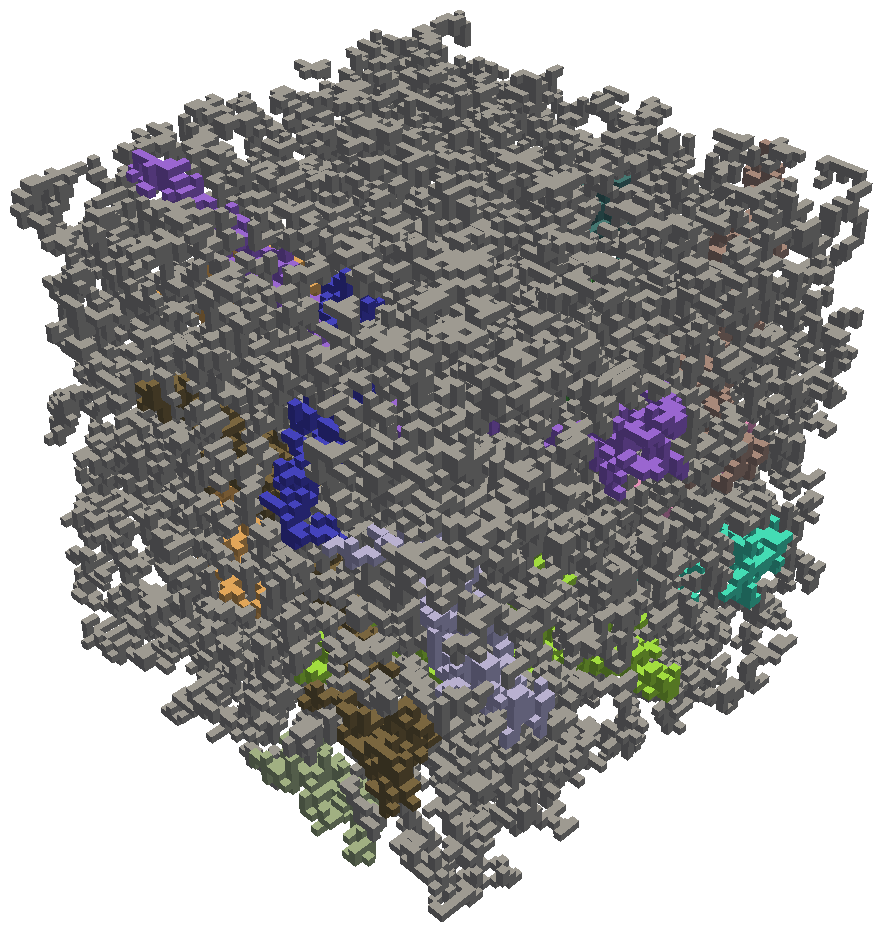}
    \end{tabular}
    \caption{Depicted above are multi-separators of volume images of simulated \textbf{filaments} with five amounts of noise, $t$ (rows): The true multi-separator (Column 2), the multi-separators output by \Cref{algo:gsg} (Column 3), and the multi-separator output by the watershed algorithm (Column 4). 
    For each $t$, the parameters ($\theta_\text{start}$ and $\theta_\text{end}$ for the watershed algorithm, $b$ for \Cref{algo:gsg}) are chosen so as to minimize the average $\viws$ across those images of the data set with the amount of noise $t$.
    Components that do not match with any true component are depicted in gray.
    For clarity, only components containing at least $10$ voxels are shown.}
    \label{fig:filament-example}
\end{figure}

\begin{figure}
    \centering
    \setlength\tabcolsep{2pt}
    \begin{tabular}{
        >{\centering}m{0.03\textwidth} 
        >{\centering}m{0.23\textwidth} 
        >{\centering}m{0.23\textwidth} 
        >{\centering}m{0.23\textwidth} 
        >{\centering\arraybackslash}m{0.23\textwidth}}
        & Synthetic Image & True Segmentation & Multi-separator & Watershed\\

        \rotatebox{90}{$t=0.0$} &
        \includegraphics[width=\linewidth]{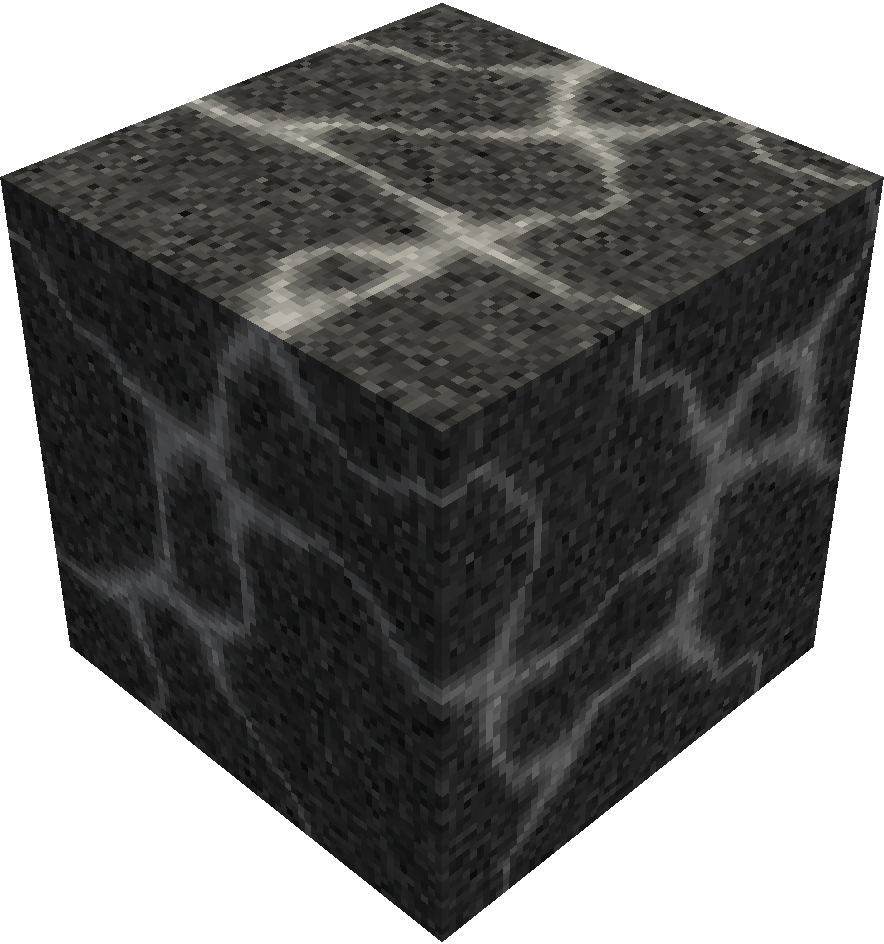} &
        \includegraphics[width=\linewidth]{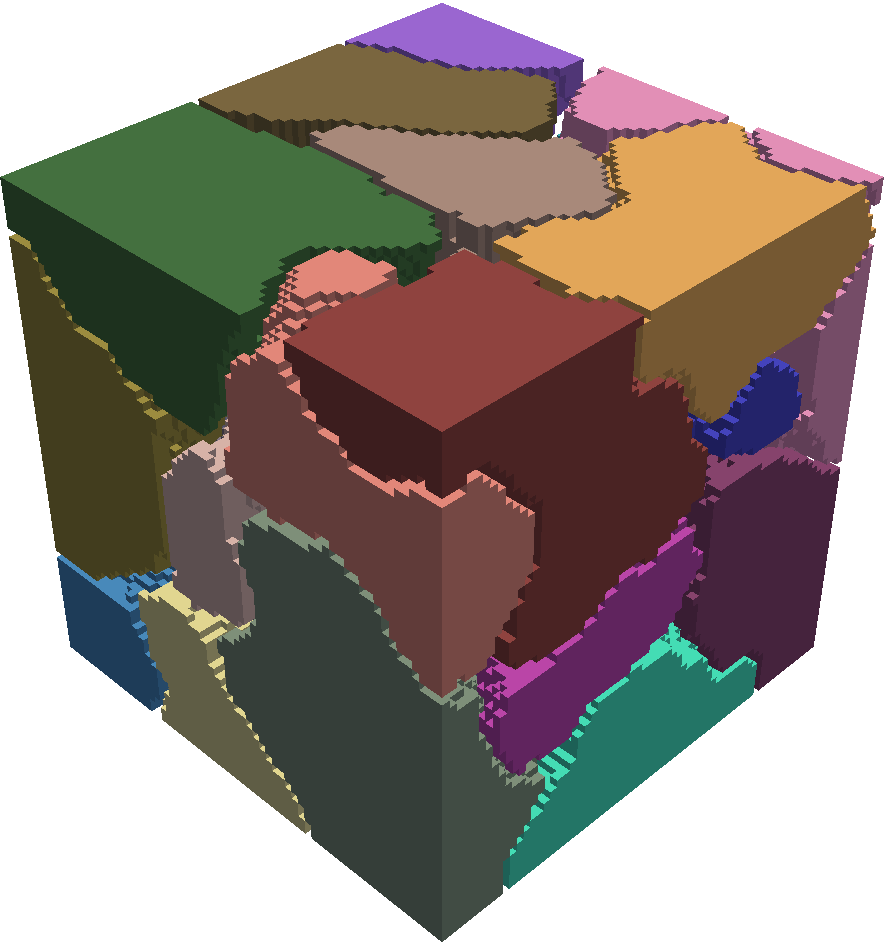} &
        \includegraphics[width=\linewidth]{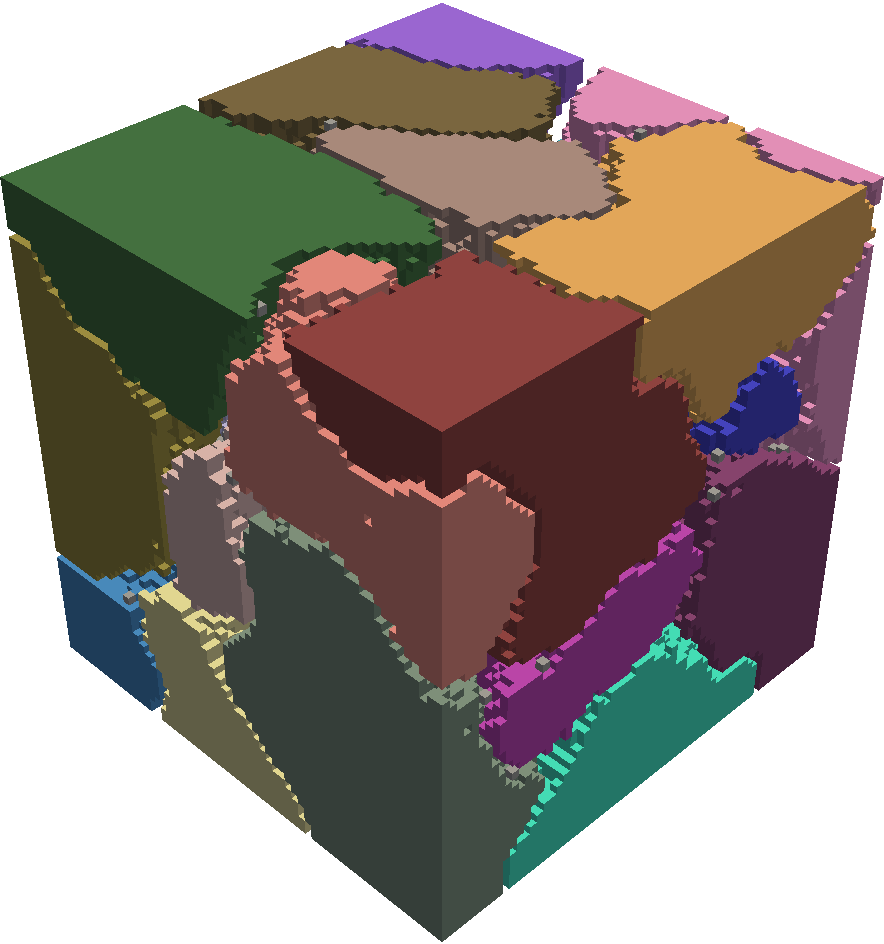} &
        \includegraphics[width=\linewidth]{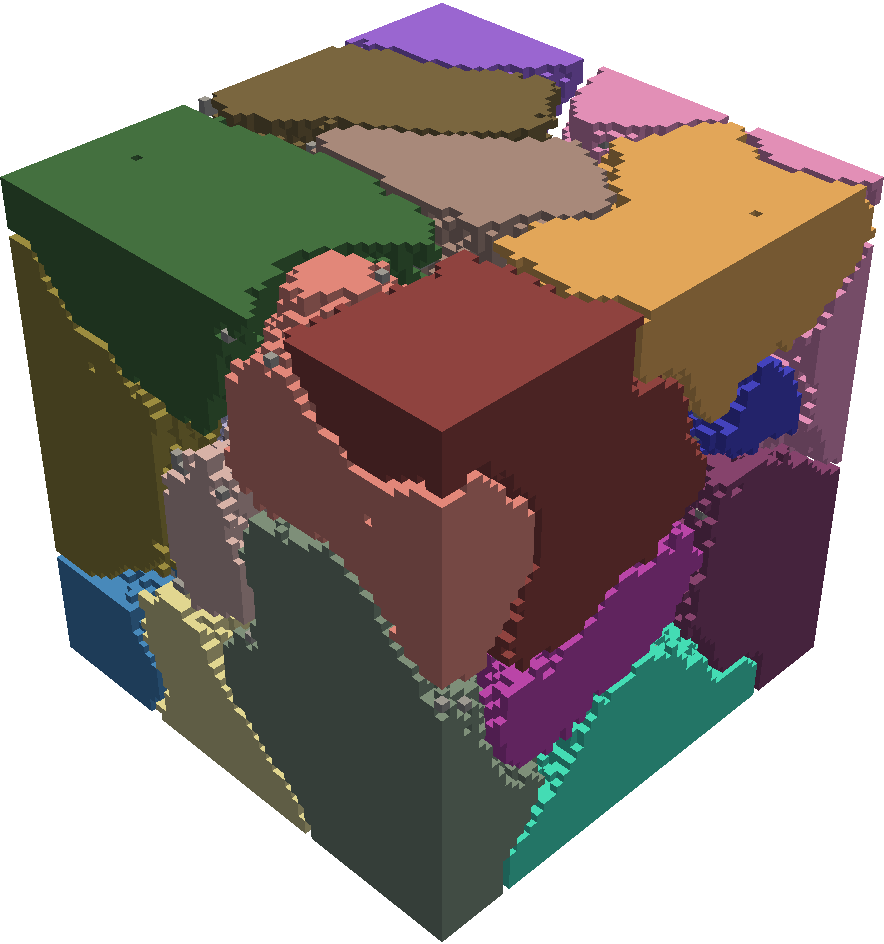} \\
        
        \rotatebox{90}{$t=0.25$} &
        \includegraphics[width=\linewidth]{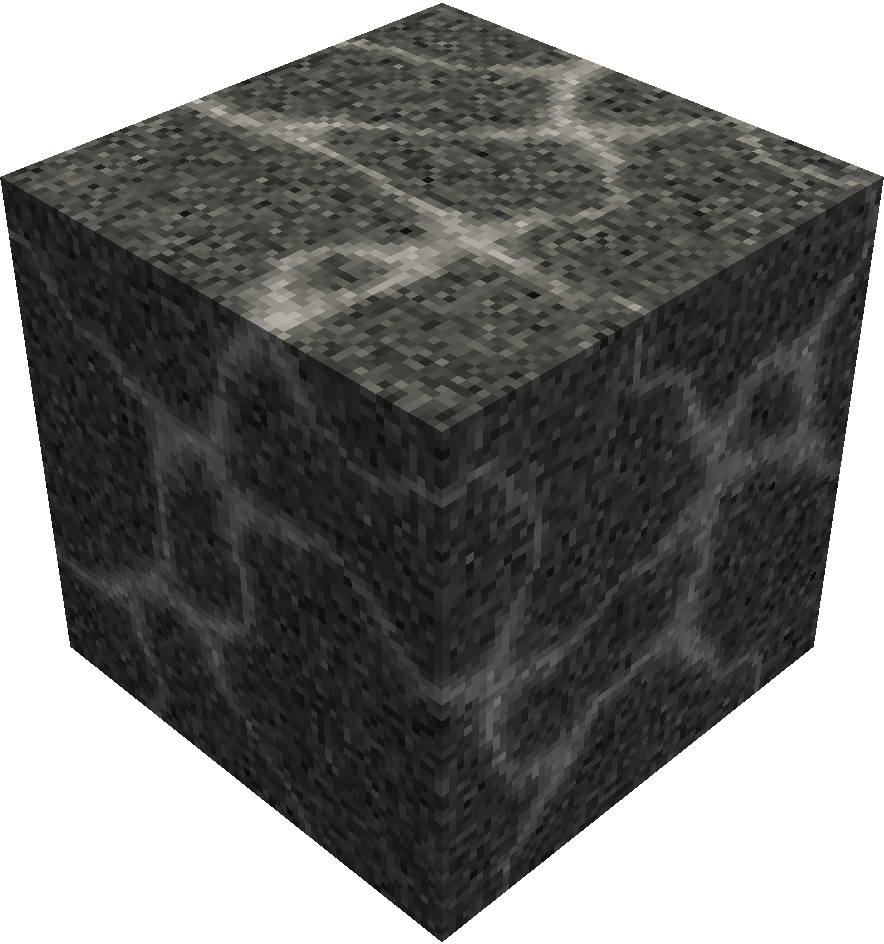} &
        \includegraphics[width=\linewidth]{images/renders/cell_1_gt.png} &
        \includegraphics[width=\linewidth]{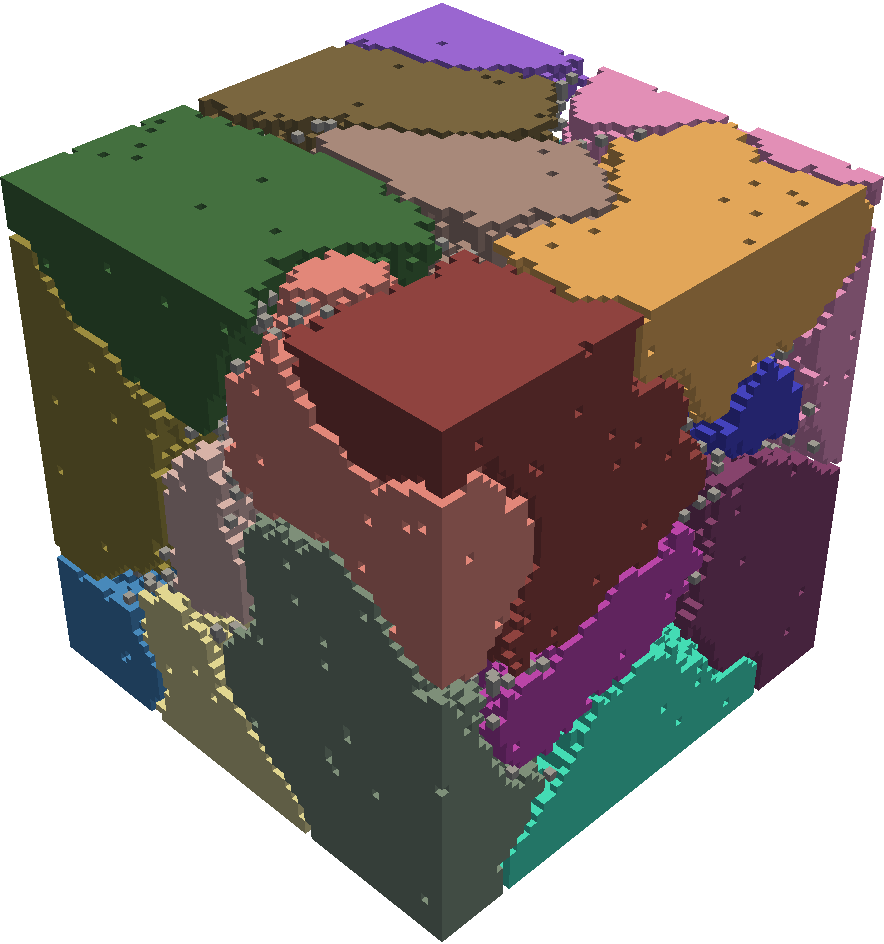} &
        \includegraphics[width=\linewidth]{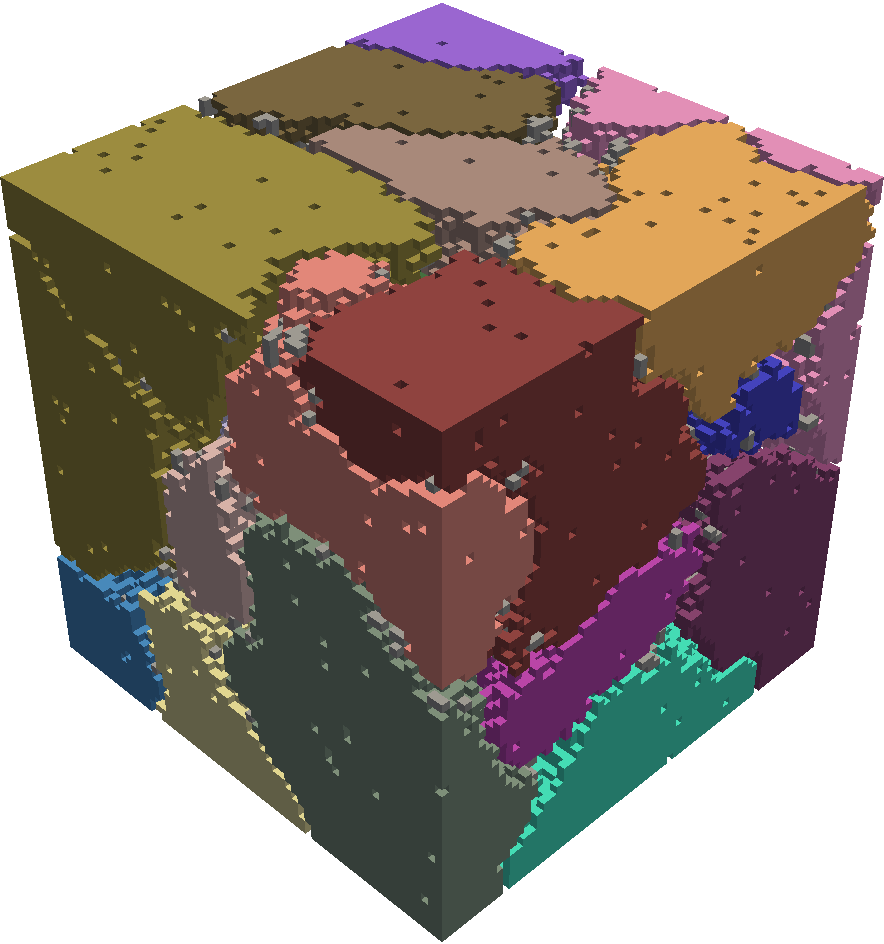} \\
        
        \rotatebox{90}{$t=0.5$} &
        \includegraphics[width=\linewidth]{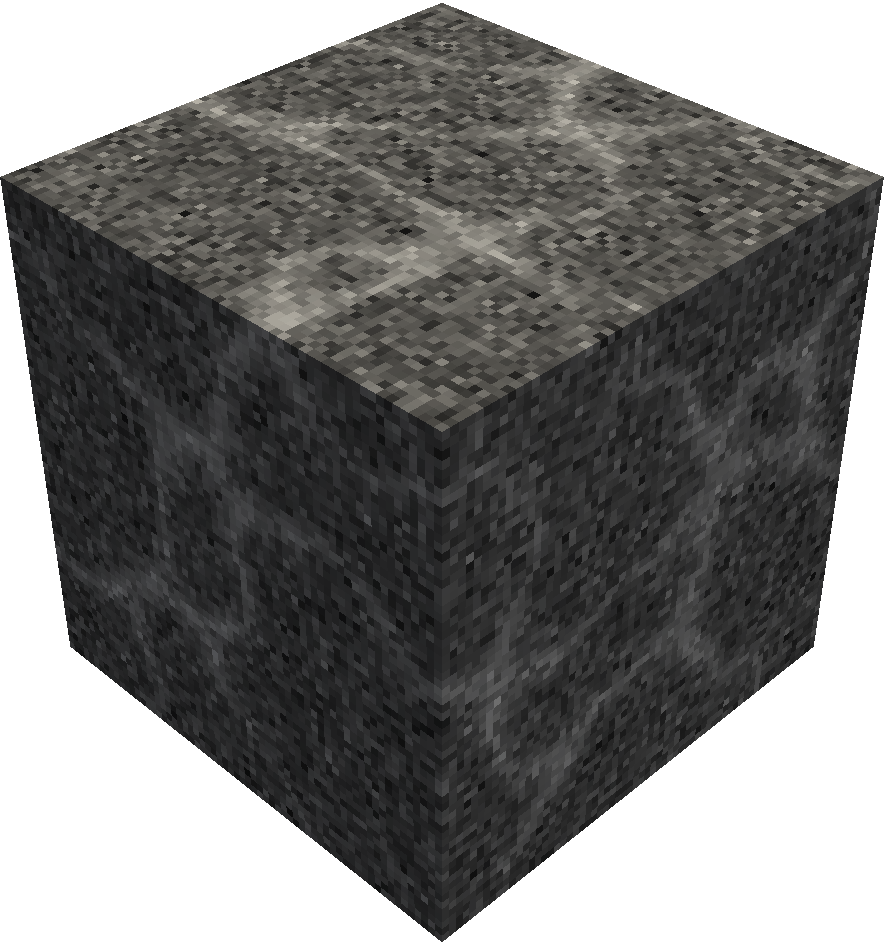} &
        \includegraphics[width=\linewidth]{images/renders/cell_1_gt.png} &
        \includegraphics[width=\linewidth]{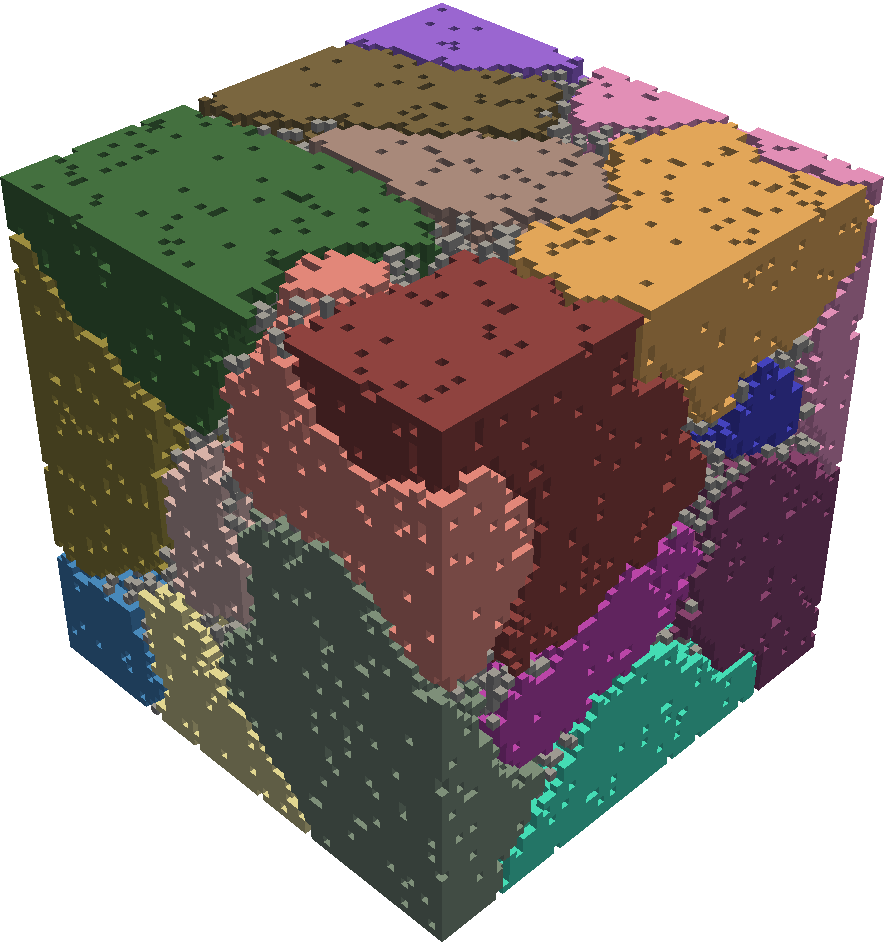} &
        \includegraphics[width=\linewidth]{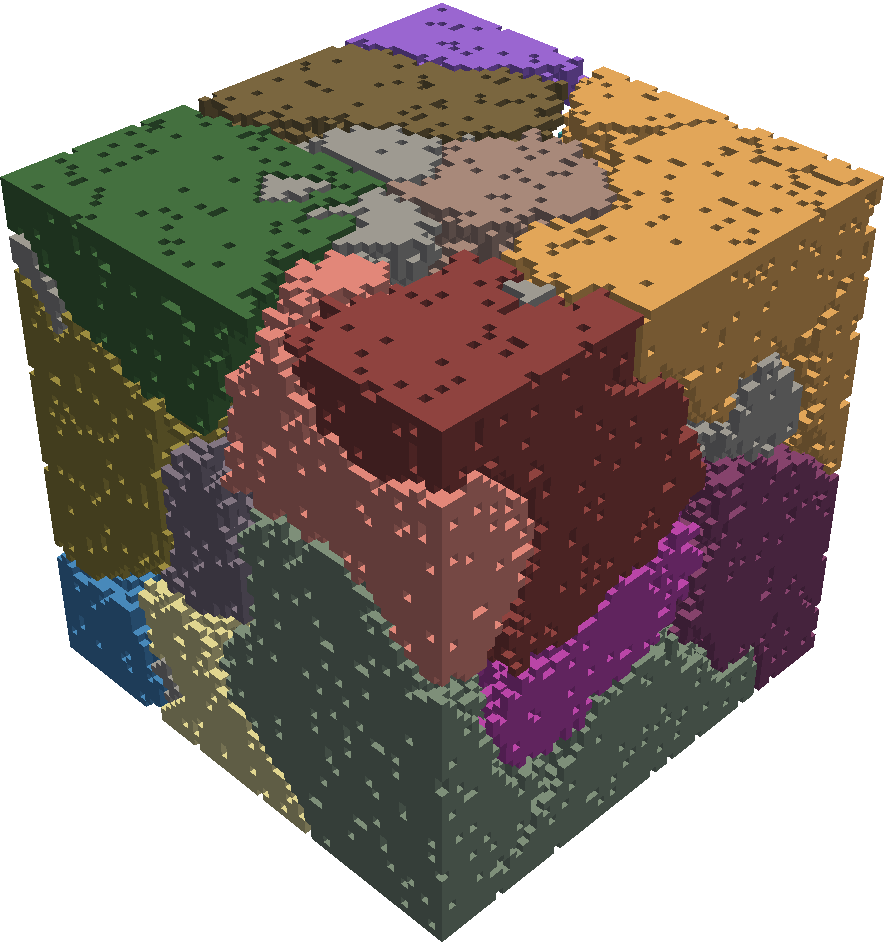} \\

        \rotatebox{90}{$t=0.75$} &
        \includegraphics[width=\linewidth]{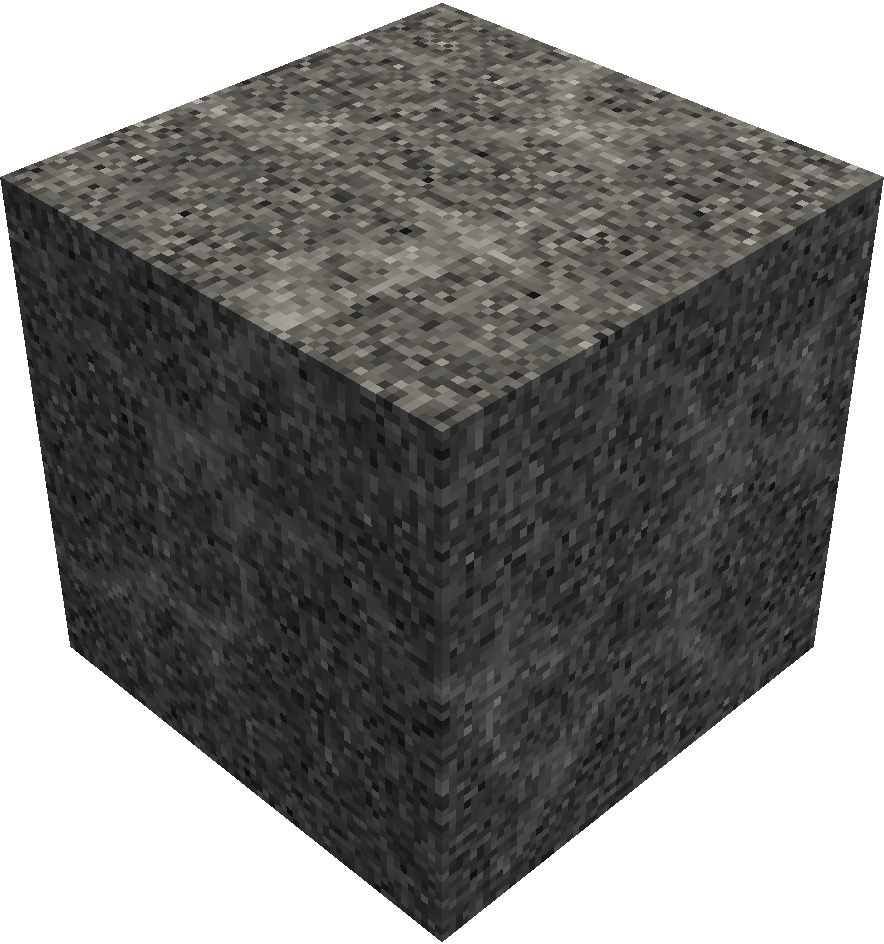} &
        \includegraphics[width=\linewidth]{images/renders/cell_1_gt.png} &
        \includegraphics[width=\linewidth]{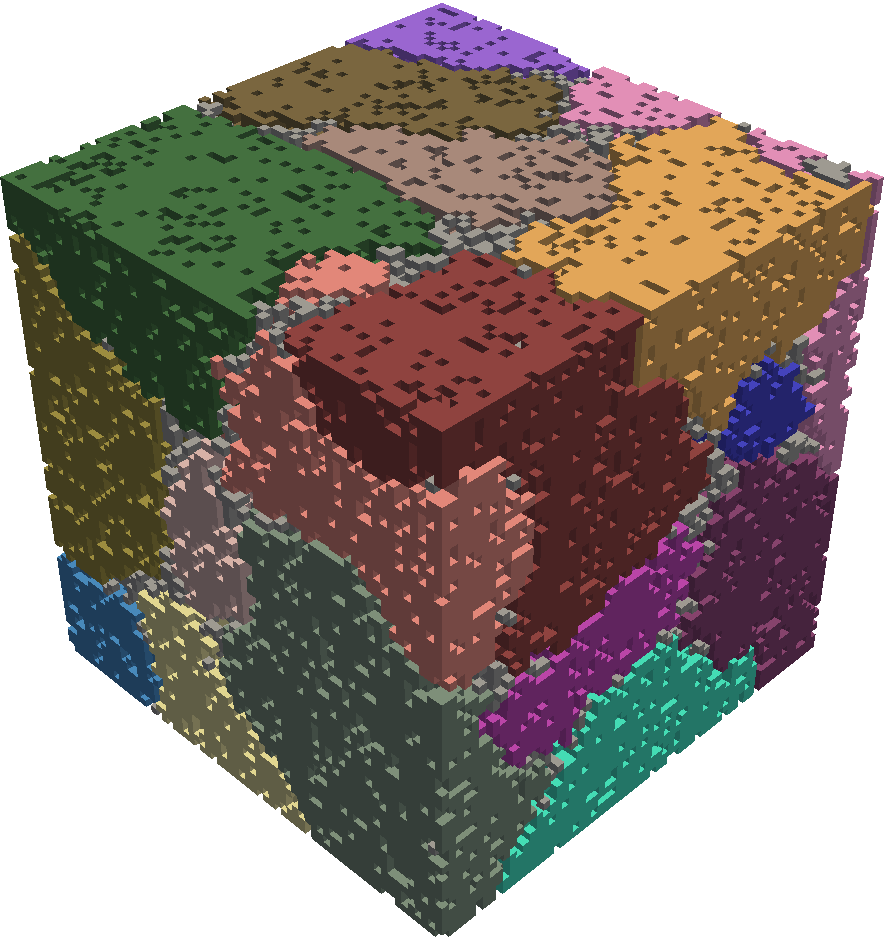} &
        \includegraphics[width=\linewidth]{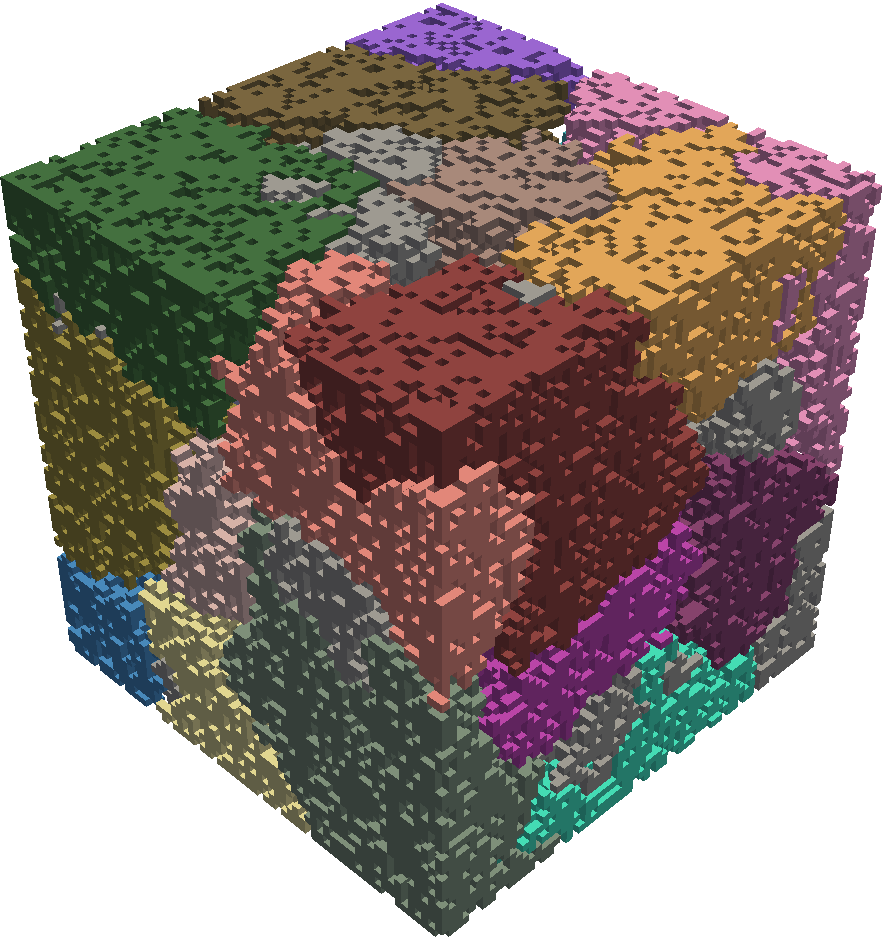} \\

        \rotatebox{90}{$t=1.0$} &
        \includegraphics[width=\linewidth]{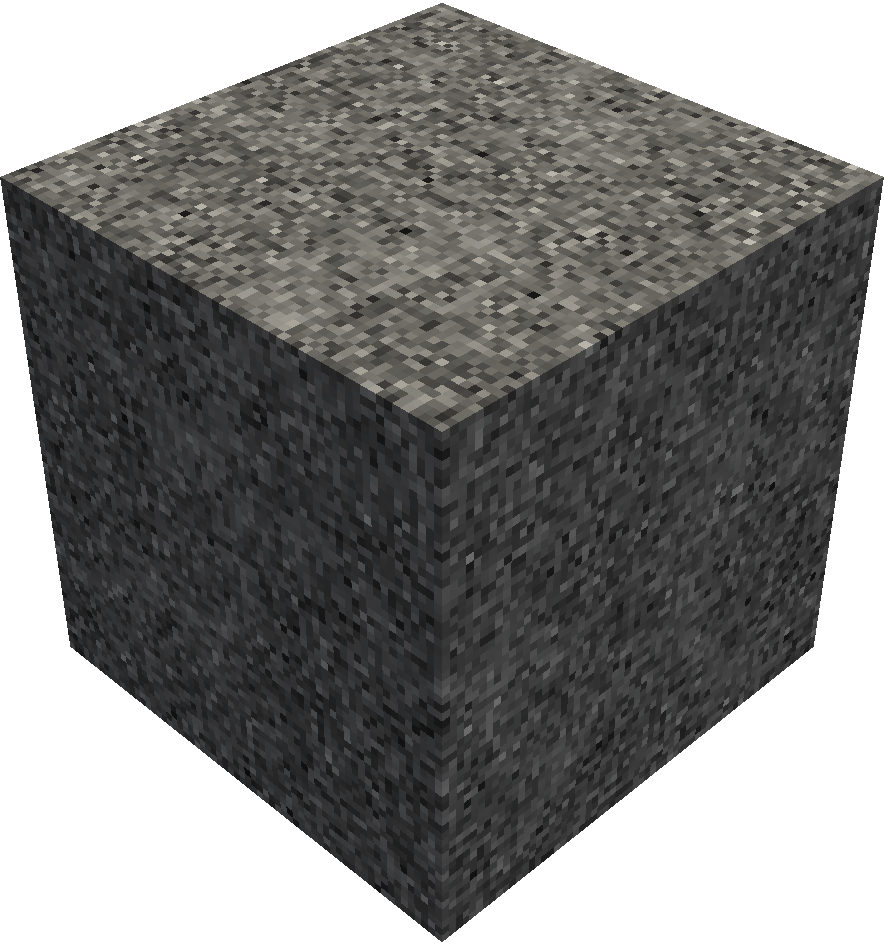} &
        \includegraphics[width=\linewidth]{images/renders/cell_1_gt.png} &
        \includegraphics[width=\linewidth]{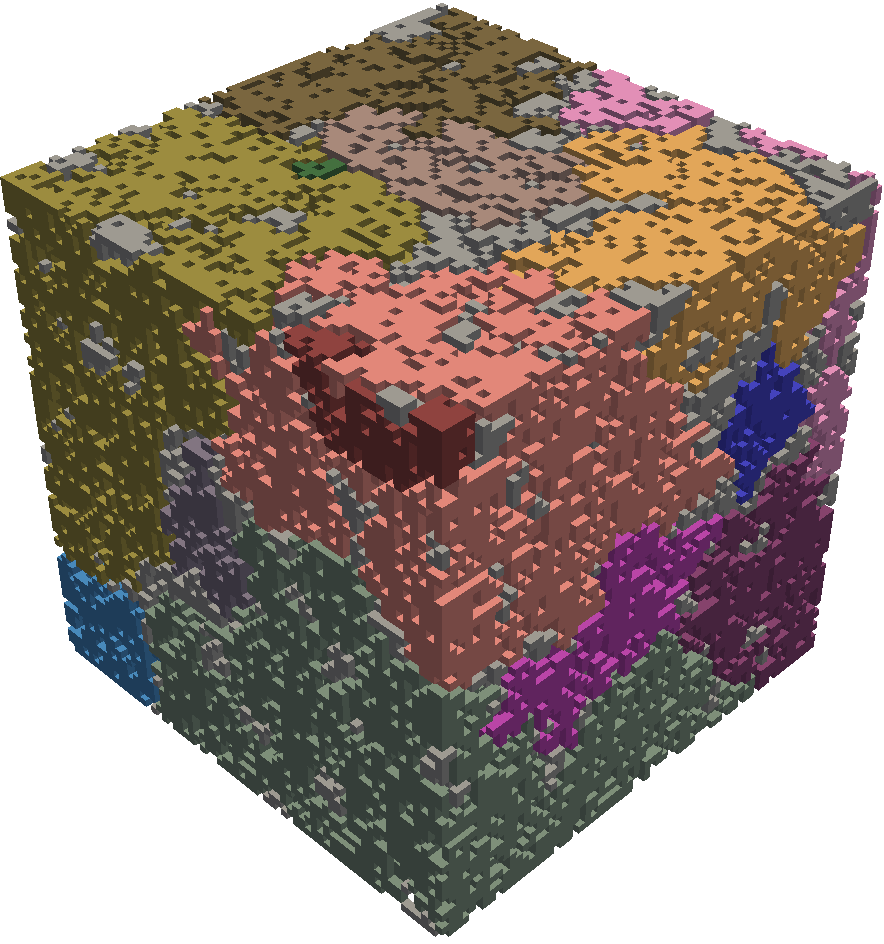} &
        \includegraphics[width=\linewidth]{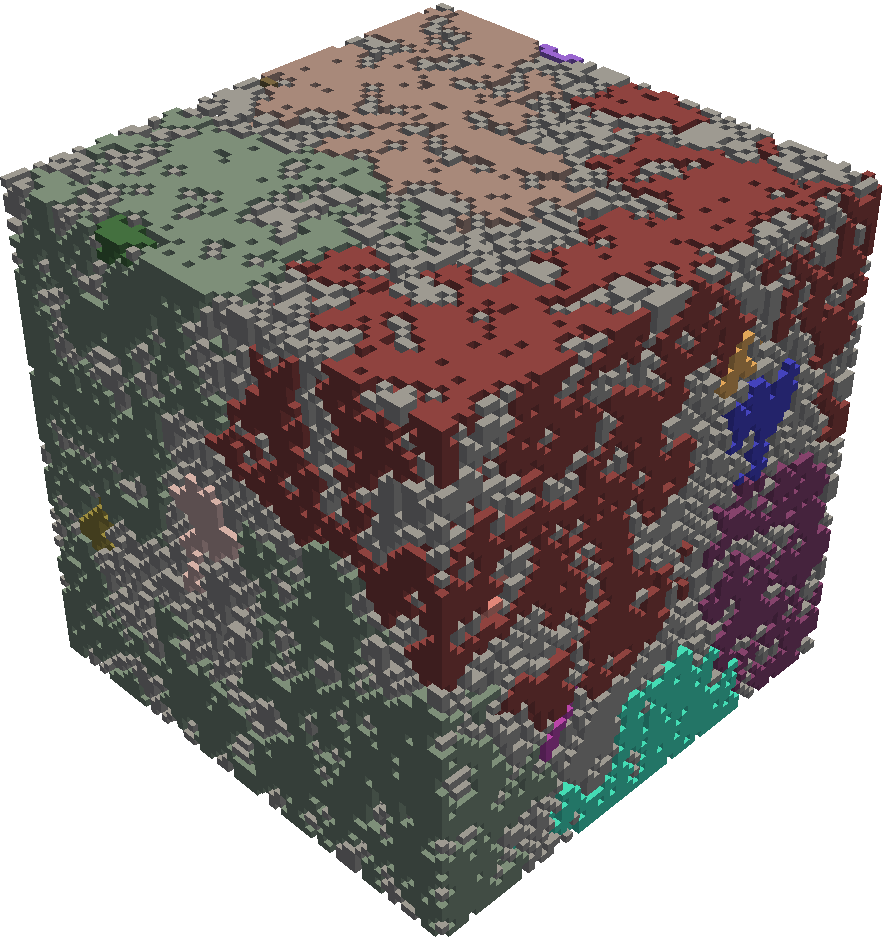}
    \end{tabular}
    \caption{Depicted above are multi-separators of volume images of simulated \textbf{foam cells} with five amounts of noise, $t$ (rows):
    The true multi-separator (Column 2), the multi-separator output by \Cref{algo:gss} (Column 3), and the multi-separator output by the watershed algorithm (Column 4).
    For each $t$, the parameters ($\theta_\text{start}$ and $\theta_\text{end}$ for the watershed algorithm, $b$ for \Cref{algo:gss}) are chosen so as to minimize the average $\viws$ across those images of the data set with the amount of noise $t$.
    Components that do not match with any true component are depicted in gray.}
    \label{fig:cell-example}
\end{figure}

\subsection{Volume images of simulated foams and filaments}\label{sec:synth-img-gen}

In order to examine the multi-separator problem defined in \Cref{sec:model}, in connection with the algorithms introduced in \Cref{sec:algorithms}, empirically, in a well-defined and adjustable setting, we synthesize two types of volume images, volume images of filaments (\Cref{fig:filament-example}), and volume images of foam cells (\Cref{fig:cell-example}).
Both types of volume images are constructed in two steps. 
The first step is to construct a binary volume image of $m \times m \times m$ voxels in which voxels labeled $0$ depict randomly constructed filaments or foam cells, and voxels labeled $1$ depict void or foam membranes separating these structures.
The second step is to construct from these binary volume images and a parameter $t \in [0,1]$ defining an amount of noise, gray scale volume images of the same size.
We repeat the first step in order to obtain volume images of different filaments or foam cells. 
We repeat the second step in order to obtain volume images with different amounts of noise.
The procedure is described in detail in \Cref{appendix:image-synthesis}.
Overall, we synthesize 10 volume images of filaments and 10 volume images of foam cells, each with 21 different amounts of noise, $t=\tfrac{i}{20}$ for $i=0,\dots,20$.
This defines a data set of 20 binary volume images and 420 gray scale volume images.
Examples are depicted in \Cref{fig:filament-example,fig:cell-example}.

The task posed by this data consists in reconstructing the binary volume image, henceforth also referred to as the \emph{truth}, from the associated gray scale volume image.
As with many synthetic data sets, this task is not necessarily difficult to solve if one exploits the complete knowledge of how the data is constructed.
We use it here to pose non-trivial problems by artificially ignoring much of this knowledge.

\subsection{Problem setup}\label{sec:msp-instance}

We state the task of reconstructing the true binary volume image from a gray scale volume image in the form of the multi-separator problem \eqref{eq:msp}.
%
% graph and interactions
To this end, we consider the grid graph consisting of one node for each voxel and edges connecting precisely the adjacent voxels.
For volume images of filaments, we consider as interactions the edges of the voxel grid graph as well as those pairs of voxels that have a rounded Euclidean distance of $8$ and a positive cost (defined below).
For volume images of foam cells, we define the set of interactions by a collection of offsets in the three directions of the voxel grid graph, namely $\mathcal{O} = \{(1, 0, 0)$, $(0, 1, 0)$, $(0, 0, 1)$, $(5, 0, 0)$, $(0, 5, 0)$, $(0, 0, 5)$, $(0, 4, 4)$, $(0, 4, -4)$, $(4, 4, 0)$, $(4, -4, 0)$, $(4, 0, 4)$, $(4, 0, -4)$, $(3, 3, 3)$, $(3, 3, -3)$, $(3, -3, 3)$, $(3, -3, -3)\}$. 
For any voxel $v$ at coordinates $(x, y, z)$ and any $(\delta_x,\delta_y,\delta_z) \in \mathcal{O}$, the voxel $v$ is connected by an interaction to the voxel at coordinates $(x+\delta_x,y+\delta_y,z+\delta_z)$ if these coordinates are within the bounds of the grid.

% node costs
For a voxel $v$ with gray value $g \in (0,1)$, we define the node cost as $c_v = \log\bigl(\tfrac{1-g}{g}\bigr)$.

% interaction costs
For an interaction $\{u,v\}$ between two voxels $u$ and $v$, we define the interaction cost $c_{\{u,v\}}$ with respect to the costs of the voxels that lie on the digital straight line connecting $u$ and $v$ in the volume image.
For volume images of filaments, we define $c_{\{u,v\}}$ as the median of the costs of the voxels on that line.
Thus, if at least one half of the voxels on the line have positive cost (i.e.~likely belong to a filament), then the interaction $\{u,v\}$ is attractive. 
Otherwise, it is repulsive.
For volume images of foam cells, we define $c_{\{u,v\}}$ as the minimum of the costs of the voxels on that line.
Thus, if there exists a voxel on the line that has negative cost (i.e.~likely belongs to a membrane), then the interaction $\{u,v\}$ is repulsive.
Otherwise, it is attractive.

The instances of the multi-separator problem thus constructed are artificial in several ways:
Firstly, we define the node costs with respect to the gray value of just one voxel, whereas most researchers working with real data, including \citet{lee2021learning}, consider such costs to depend on the gray values of voxel neighborhoods and estimate (learn) these dependencies from examples \citep{minaee2022image}.
Secondly, the node cost $c_v$ of a voxel $v$ labeled 1 in the true binary image, will, in expectation, be negative, while the cost $c_v$ of a voxel $v$ labeled 0 in the true binary image, will, in expectation, be positive.
This unbiasedness of the node costs holds by construction of the volume images, more specifically, because the average of the mean gray value of filaments and foam cells, on the one hand, and the mean gray value of void and foam membranes, on the other hand, is exactly $\frac{1}{2}$ (\Cref{sec:synth-img-gen} and \Cref{appendix:image-synthesis}).
We exploit this knowledge in the problem statement.
However, we do not limit our analysis to unbiased instances of the multi-separator problem.
Instead, we add an bias systematically, as described in \Cref{sec:empirical-procedure} below.

\subsection{Alternative}\label{section:watershed}

In order to compare the output of the algorithms we define in \Cref{sec:algorithms} to the output of an algorithm that is well-understood and widely-used, we compute multi-separators from each volume image also by means of a watershed algorithm \citep[cf.][for an overview]{roerdink2000watershed}, more specifically, the watershed algorithm and implementation by \citet{scikit-image}.
Therefor, we introduce two parameters, $\theta_\text{start}, \theta_\text{end} \in [0,1]$ with $\theta_\text{start} \leq \theta_\text{end}$.
$\theta_\text{start}$ defines seeds for the watershed algorithm as the maximal connected components of all voxels with gray value less than $\theta_\text{start}$.
$\theta_\text{end}$ defines the gray value at which the flood-filling procedure of the watershed algorithm is terminated.
We consider several combinations of these parameters, from intervals large enough to be clearly sub-optimal at the end-points, for all amounts of noise.
For filaments, we consider 41 values of $\theta_\text{start}$ equally spaced in $[0.45, 0.65]$, and 41 values of $\theta_\text{end}$ equally spaced in $[0.5, 0.7]$.
For foam cells, we consider 51 values of $\theta_\text{start}$ equally spaced in $[0.0, 0.5]$, and 21 values of $\theta_\text{end}$ equally spaced in $[0.4, 0.6]$.
The effect of the parameters $\theta_\text{start}$ and $\theta_\text{end}$ is depicted in \Cref{appendix:watershed-parameters}.

\subsection{Metric}\label{sec:metric}

In order to measure a distance between any multi-separator computed from a gray scale volume image, on the one hand, and the true multi-separator, i.e.~the set of voxels labeled 1 in the true binary volume image, on the other hand, we take two steps. 
Fristly, we consider partitions of the node set induced by these two separators.
Secondly, we measure the distance between these partitions by evaluating a metric known as the variation of information \citep{arabie-1973,meila-2007}.
The two steps are described in more detail below.

For any multi-separator $S$ of a graph $G=(V,E)$, we consider the partition of the node set $V$ consisting of (i) the inclusion-wise maximal subsets of $V\setminus S$ that induce  a component of the subgraph of $G$ induced by $V\setminus S$, and (ii) one singleton set $\{v\}$ for every node $v \in S$.

For 
any two partitions $\mathcal{A}, \mathcal{B}$ of $V$, 
any probability mass function $p \colon V \to [0,1]$ and
the extension of $p$ to subsets of $V$ such that for any $U \subseteq V$, we have $p(U) = \sum_{v \in U} p(v)$, 
the variation of information holds
\begin{align*}
\mathrm{VI}_p (\mathcal{A}, \mathcal{B}) 
& = 2 H_p(\mathcal{A}, \mathcal{B}) - H_p(\mathcal{A}) - H_p(\mathcal{B}) \\
& = H_p(\mathcal{A} \mid \mathcal{B}) + H_p(\mathcal{B} \mid \mathcal{A}) 
\end{align*}
with
\begin{align*}
H_p(\mathcal{A}) 
	& = - \sum_{A \in \mathcal{A}} p(A) \log_2 p(A)
\\
H_p(\mathcal{A}, \mathcal{B}) 
	& = - \sum_{A \in \mathcal{A}, B \in \mathcal{B}, A \cap B \neq \emptyset} \hspace{-5.5ex} p(A \cap B) \log_2 p(A \cap B) 
\\
H_p(\mathcal{A} \mid \mathcal{B}) 
	& = H_p(\mathcal{A}, \mathcal{B}) - H_p(\mathcal{B})
\end{align*}
the entropy, joint entropy and conditional entropy with respect to the probability mass function $p$.

Usually, one considers the variation of information with respect to the uniform probability mass $p(v) = |V|^{-1}$ for all $v \in V$ \citep{arabie-1973,meila-2007}.
In our experiments, there is an imbalance between the number of voxels in the true separator $T$ and the number of voxels in the complement $V \setminus T$.
For filaments, most voxels are in $T$.
For foam cells, most voxels are in $V \setminus T$.
We chose $p(v) = \tfrac{1}{2} |T|^{-1}$ for all $v \in T$ and $p(v) = \tfrac{1}{2} |V \setminus T|^{-1}$ for all $v \in V \setminus T$.
This way, nodes in the separator and nodes not in the separator both contribute $\tfrac{1}{2}$ to the total probability mass.
We refer to the metric thus obtained as the \emph{variation of information between separator-induced weighted partitions with singletons} and abbreviate it as $\viws$.
The conditional entropies $H_p(\mathcal{A} \mid \mathcal{B})$ and $H_p(\mathcal{B} \mid \mathcal{A})$ are indicative of false cuts and false joins of $\mathcal{A}$ compared to $\mathcal{B}$. We abbreviate these as $\fc$ and $\fj$.

Complementary to $\viws$, we report also the variation of information with respect to the uniform probability mass function and partitions of the set of those voxels that are neither in the true multi-separator nor in the computed multi-separator.
We refer to this metric as the \emph{variation of information with respect to non-separator nodes}.
We abbreviate this metric as $\vins$ and abbreviate the conditional entropies due to false cuts and false joins as $\fcns$ and $\fjns$.
The metric $\vins$ suffers from a degeneracy:
If the computed separator contains all nodes, $\vins$ evaluates to zero.
This degeneracy of $\vins$ has been the motivation for us to introduce $\viws$.

\subsection{Experiments}\label{sec:empirical-procedure}

For each volume image from the data set described in \Cref{sec:synth-img-gen}, we construct several instances of the multi-separator problem according to \Cref{sec:msp-instance}, each with a different bias $b \in \mathbb{R}$ added to all coefficients in the cost function.
A positive bias penalizes voxels being in the separator; 
a negative bias rewards voxels being in the separator.
We consider several biases, from intervals large enough to be clearly sub-optimal at the end-points, for all amounts of noise (\Cref{fig:inaccuracy-by-bias}).
For both filaments and foam cells, we consider 51 values equally spaced in $[-0.25, 0.25]$.
For each instance of the multi-separator problem thus obtained, we compute one multi-separator, using \Cref{algo:gsg}, for filaments, and \Cref{algo:gss}, for foam cells.

For each volume image and every combination of the parameters $\theta_\text{start}, \theta_\text{end}$, we compute one additional multi-separator by means of the watershed algorithm described in \Cref{section:watershed}.

We measure the distance between each multi-separator thus computed from a gray scale volume image, on the one hand, and the set of voxels labeled 1 in the true binary image, on the other hand, by evaluating the metric described in \Cref{sec:metric}.

In addition, we measure the runtime of our single-thread \cplusplus{} implementation of \Cref{algo:gsg,algo:gss} as a function of the size of a volume image, for additional volumes images of $m \times m \times m$ voxels, with $m \in \{20, \dots, 216\}$, synthesized as described in \Cref{sec:synth-img-gen} and \Cref{appendix:image-synthesis}, with noise $t=0.5$.
From these, we construct instances of the multi-separator problem as described in \Cref{sec:msp-instance}.
We measure the runtime on a Lenovo X1 Carbon laptop equipped with an Intel Core i7-10510U CPU @ 1.80GHz processor and 16 GB LPDDR3 RAM.

The complete \cplusplus{} code for reproducing these experiments is included as supplementary material.

\subsection{Results}

\begin{figure}
    \centering
    \includegraphics[width=\textwidth]{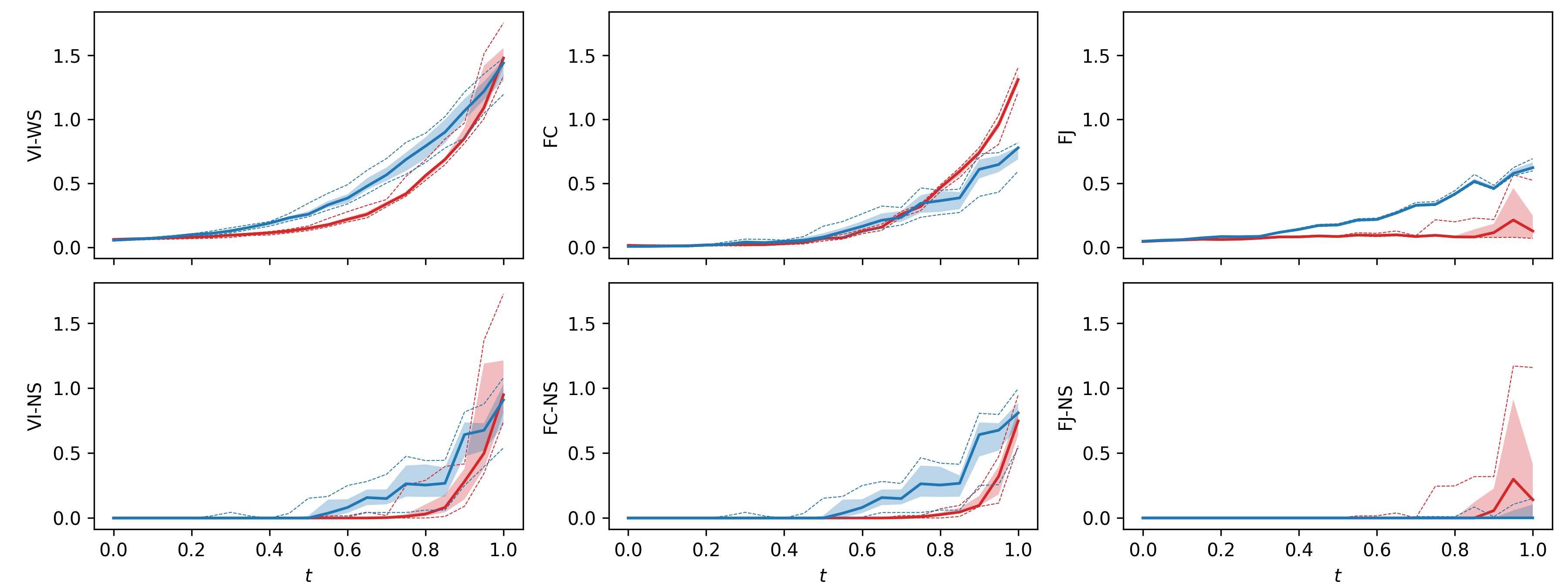}
    \caption{Depicted above is the inaccuracy (distance from the truth) of multi-separators computed by \Cref{algo:gsg} (red) or the watershed algorithm (blue), for volume images of \textbf{filaments} with different amounts of noise $t$.
    For each volume image, the parameters for both algorithms are chosen such that the multi-separator output by the algorithm minimizes the average $\viws$ across those images of the data set with the amount of noise $t$.
    Depicted on top from left to right are the $\viws$ as well as the conditional entropies due to false cuts ($\fc$) and false joins ($\fj$).
    Depicted on the bottom from left to right are the variation of information with respect to only those voxels that are not labeled as separator in both the true multi-separator and the computed multi-separator ($\vins$) as well as the respective conditional entropies due to false cuts ($\fcns$) and false joins ($\fjns$).
    Thick lines indicate the median across those images of the data set with the amount of noise $t$. 
    Shaded areas indicate the second and third quartile.
    Dashed lines indicate the 0.1 and 0.9-quantile.}
    \label{fig:filament-inaccuracy-by-noise}
\end{figure}

\begin{figure}
    \centering
    \includegraphics[width=\textwidth]{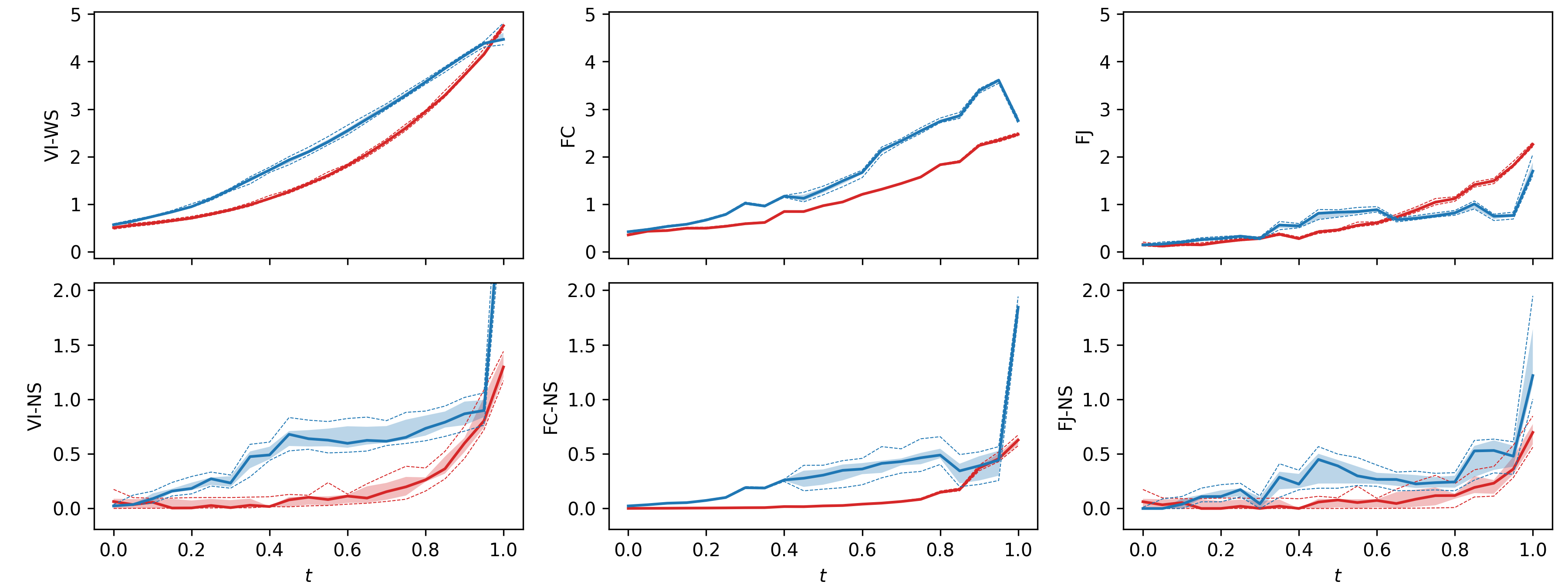}
    \caption{Depicted above is the inaccuracy (distance from the truth) of multi-separators computed by \Cref{algo:gsg} (red) or the watershed algorithm (blue), for volume images of \textbf{foam cells} with different amounts of noise $t$.
    The interpretation of the figure is completely analogous to \Cref{fig:filament-inaccuracy-by-noise}.}
    \label{fig:cell-inaccuracy-by-noise}
\end{figure}

\begin{figure}
    \centering
    \includegraphics[width=\textwidth]{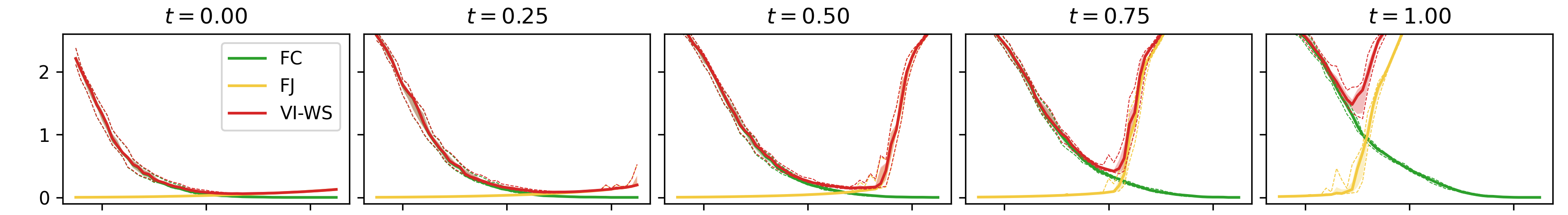}
    \includegraphics[width=\textwidth]{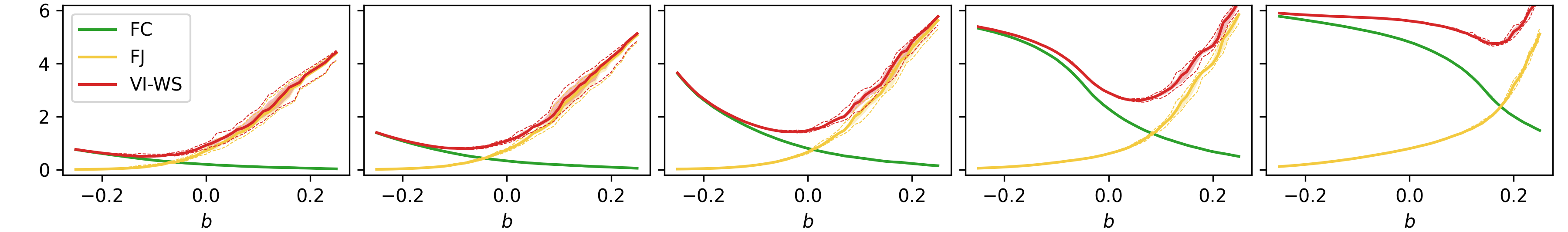}
    \caption{Depicted above is the inaccuracy (distance from the truth) of multi-separators computed by \Cref{algo:gsg} for volume images of \textbf{filaments} (top) and \Cref{algo:gss} for volume images of \textbf{foam cells} (bottom) with five amounts of noise $t$ (columns), as a function of the bias $b$.
    Depicted are the $\viws$ (red) as well as the conditional entropies due to false cuts ($\fc$, green) and false joins ($\fj$, yellow).
    Thick lines indicate the median across those images of the data set with the amount of noise $t$. 
    Shaded areas indicate the second and third quartile.
    Dashed lines indicate the 0.1 and 0.9-quantile.}
    \label{fig:inaccuracy-by-bias}
\end{figure}

\begin{figure}
    \centering
    \includegraphics[height=5cm]{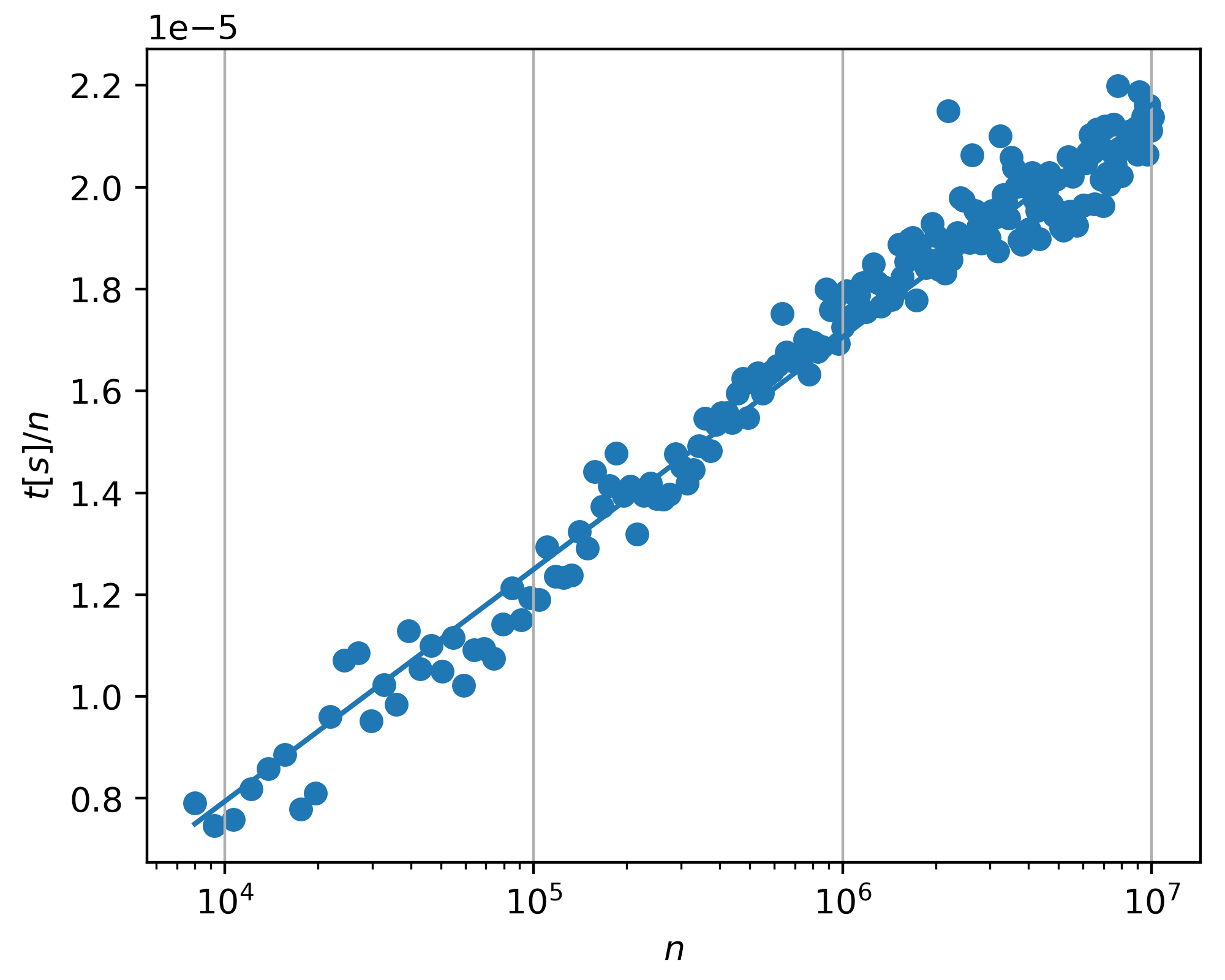}
    \includegraphics[height=5cm]{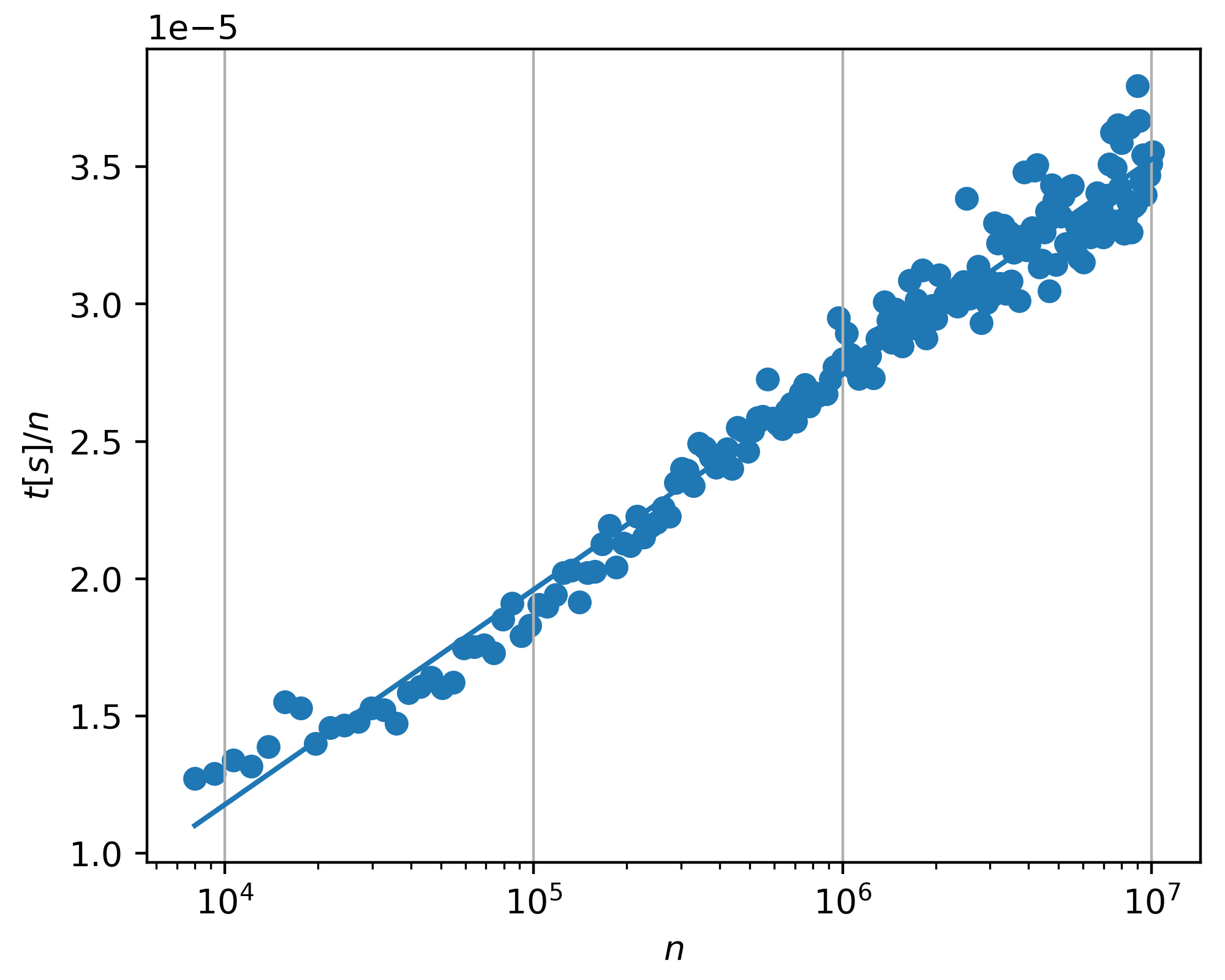}
    \caption{Depicted on the left is the runtime per voxel of our single-thread \cplusplus{} implementation of \Cref{algo:gss} for instances of the multi-separator problem from volume images of \textbf{foam cells}.
    Depicted on the right is the runtime per voxel of our single-thread \cplusplus{} implementation of \Cref{algo:gsg} for instances of the multi-separator problem from volume images of \textbf{filaments}.
    Both are for volume images of $n = m \times m \times m$ voxels, with $m \in \{20, \dots, 216\}$, synthesized as described in \Cref{sec:synth-img-gen} and \Cref{appendix:image-synthesis}, with the noise level $t = 0.5$ and bias $b = 0$.}
    \label{fig:runtime-analysis}
\end{figure}

Depicted in \Cref{fig:filament-inaccuracy-by-noise,fig:cell-inaccuracy-by-noise,fig:inaccuracy-by-bias} is the inaccuracy, i.e.~the distance from the truth, of multi-separators computed by \Cref{algo:gsg} (for filaments), by \Cref{algo:gss} (for foam cells) and by the watershed algorithm (for both), for volume images with different amounts of noise $t$.
Depicted in \Cref{fig:filament-example,fig:cell-example} are qualitative examples.

\Cref{fig:filament-inaccuracy-by-noise,fig:cell-inaccuracy-by-noise} show, for each amount of noise $t$, the distribution of distances across those volume images in the data set that exhibit this amount of noise.
Here, the parameters ($b$, for \Cref{algo:gsg,algo:gss}, $\theta_\text{start}$ and $\theta_\text{end}$ for the watershed algorithm) are chosen so as to minimize the average $\viws$ across those images of the data set with the amount of noise $t$.
It cannot be seen from \Cref{fig:filament-inaccuracy-by-noise,fig:cell-inaccuracy-by-noise} how sensitive the $\viws$ is to the choice of the bias parameter $b$.
To this end, \Cref{fig:inaccuracy-by-bias} shows the inaccuracy of multi-separators computed by \Cref{algo:gsg} (for filaments) or \Cref{algo:gss} (for foam cells) as a function of the bias $b$.

Finally, \Cref{fig:runtime-analysis} shows measurements of the runtime of our implementations of \Cref{algo:gsg,algo:gss} for instances of the multi-separator problem of varying size.

\subsection{Observations}

\subsubsection{General observations}

For volume images with little noise, $t \leq 0.1$, \Cref{algo:gsg,algo:gss} and the watershed algorithm output a multi-separator almost without false joins or false cuts, and those inaccuracies that do occur have a small effect on the variation of information (\Cref{fig:filament-inaccuracy-by-noise,fig:cell-inaccuracy-by-noise}).
For $t>0.1$, multi-separators output by \Cref{algo:gsg,algo:gss} can be significantly closer to the truth than those output by the watershed algorithm in terms of the distance $\viws$.
For $t \leq 0.5$, the multi-separators output by \Cref{algo:gsg,algo:gss} are not sensitive to small deviations of the bias $b$ from the optimum (\Cref{fig:inaccuracy-by-bias}, Columns 1 to 3).
For $t > 0.5$, multi-separators output by \Cref{algo:gsg,algo:gss} can still be significantly closer to the truth than those output by the watershed algorithm in terms of the distance $\viws$ (\Cref{fig:filament-inaccuracy-by-noise,fig:cell-inaccuracy-by-noise}, left).
However, the sensitivity of the multi-separators output by \Cref{algo:gsg,algo:gss} to the choice of the bias $b$ increases with increasing noise (\Cref{fig:inaccuracy-by-bias}, Columns 4 and 5).
For $t=1$, all multi-separators output by any of the algorithms are inaccurate.

A negative bias $b$ generally leads to more false cuts and less false joins (\Cref{fig:inaccuracy-by-bias}).
The optimal bias does not differ much across multiple volume images exhibiting the same amount of noise. This can be seen from \Cref{fig:inaccuracy-by-bias} where the area within the 0.1 and 0.9 quantile is small.

For instances of the multi-separator problem that witnessed the worst case time complexities established in \Cref{lem:gss-worst-case-runtime,lem:gsg-worstcase-runtime}, \Cref{algo:gsg,algo:gss} would not be practical.
For the instances we consider here, however, we observe an empirical runtime of approximately $\mathcal{O}(|V|\log|V|)$, for both algorithms (\Cref{fig:runtime-analysis}).

\subsubsection{Specific observations for filaments}

For volume images of filaments, segmentations are prone to false cuts:
If the gray values of just a few voxels on a filament are distorted due to noise, \emph{locally} it may appear as if there were two distinct filaments.
By modeling the segmentation task as a multi-separator problem with attractive \emph{long-range} interactions, it is possible to avoid false cuts due to such local distortions (\Cref{fig:filament-inaccuracy-by-noise}, bottom center).
Only for large amounts of noise ($t \geq 0.8$) do the multi-separators output by \Cref{algo:gsg} suffer from false cuts and false joins of filaments (\Cref{fig:filament-inaccuracy-by-noise} bottom).

The watershed algorithm computes multi-separators with the least inaccuracy for $\theta_\text{start} \approx \theta_\text{end}$ (\Cref{fig:ws-threshold-filament} in \Cref{appendix:watershed-parameters}).
Smaller values of $\theta_\text{start}$ can lead to multiple distinct seeds for one filament and thus to false cuts (\Cref{fig:ws-threshold-filament}, middle row).
Greater values of $\theta_\text{end}$ can lead to a later termination of the flood-filling procedure, which leads to thicker filaments and thus to false joins (\Cref{fig:ws-threshold-filament}, bottom row).
Moreover, a large value of $\theta_\text{start}$ leads to many voxels falsely output by the watershed algorithm as non-separator, which causes false joins (\Cref{fig:filament-inaccuracy-by-noise} top right and \Cref{fig:filament-example} fourth column).
These effects are one reason why fundamentally different techniques are used for reconstructing filaments (cf.~\Cref{sec:related-work}).

Numerical results corresponding to the examples depicted in \Cref{fig:filament-example} are reported in \Cref{tab:filament-example} in \Cref{appendix:detailed-numerical-results}.
For noise $t=0.5$, for instance, the multi-separator computed by \Cref{algo:gsg} has distance $\viws=0.144$ from the truth, with $\fc=0.059$ and $\fj = 0.084$ while the multi-separator output by the watershed algorithm has distance $\viws=0.250$ from the truth, with $\fc=0.075$ and $\fj = 0.174$.
For greater levels of noise, the quantitative difference between the multi-separators output by the watershed algorithm and \Cref{algo:gsg} is also reflected in the variation of information with respect to non-separator voxels.
For instance, for $t=0.75$, the multi-separator computed by \Cref{algo:gsg} has $\vins = 0$ due to no false cuts or joins of filaments.
In contrast, the multi-separator output by the watershed algorithm has $\vins = 0.435$ due to false cuts of some filaments (green and red filaments in row 4 column 4 of \Cref{fig:filament-example}).

\subsubsection{Specific observations for foam cells}

For volume images of foam cells, segmentations are prone to false joins:
If the gray values of just a few voxels on a membrane are distorted due to noise, \emph{locally} it may appear as if the two cells that are separated by that membrane are one cell.
By modeling the segmentation task as a multi-separator problem involving also \emph{long-range} interactions, it is possible to avoid false cuts due to such local distortions (\Cref{fig:cell-inaccuracy-by-noise}, bottom right).
Only for large amounts of noise ($t \geq 0.8$) do the multi-separators output by \Cref{algo:gss} suffer from false cuts and false joins of foam cells (\Cref{fig:cell-inaccuracy-by-noise} bottom).

For the watershed algorithm, there are two regions for the parameters $\theta_\text{start}$ and $\theta_\text{end}$ for which multi-separators with the least inaccuracy are computed for foam cells (\Cref{fig:ws-threshold-cell} in \Cref{appendix:watershed-parameters}):
In order for the watershed algorithm to avoid false cuts, there needs to be exactly one seed region per foam cell.
This can be archived by either large $\theta_\text{start} \approx 0.4$ which ideally leads to one large seed per foam cell, or by small $\theta_\text{start} \approx 0.15$ which ideally leads to one small seed per foam cell.
Intermediate values for $\theta_\text{start}$ lead to multiple seeds per foam cell which leads to false cuts (\Cref{fig:ws-threshold-cell}, middle row).
Very small or very large values of $\theta_\text{start}$ lead to false joins (\Cref{fig:ws-threshold-cell}, bottom row): 
If $\theta_\text{start}$ is too small, there need not be a seed in each foam cell.
During the flood filling step of the watershed algorithm, the seed from one foam cell can spill into another foam cell that does not have a seed. 
If $\theta_\text{start}$ is too large, the maximal components of all voxels with a value less than $\theta_\text{start}$ can already cross the membranes.
The tradeoff between false cuts and false joins in multi-separators computed by the watershed algorithm can be seen in the second and third column of \Cref{fig:cell-inaccuracy-by-noise}.
The parameter $\theta_\text{end}$ determines at which point the flood filling algorithm is terminated. 
A small $\theta_\text{end}$ leads to an early termination which results in smaller foam cells and thicker membrane while a large $\theta_\text{end}$ leads to a late termination which results in larger foam cells and a thinner membrane.
Too thin or thick membranes are penalized by the $\viws$ (as this results in to few or two many singleton sets in the computation of the $\viws$).

Numerical results corresponding to the examples depicted in \Cref{fig:cell-example} are reported in \Cref{tab:cell-example} in \Cref{appendix:detailed-numerical-results}.
For noise $t=0.5$, for instance, the multi-separator computed by \Cref{algo:gss} has distance $\viws=1.400$, with $\fc=0.970$ and $\fj=0.430$, while the multi-separator output by the watershed algorithm has distance $\viws=2.110$ from the truth, with $\fc=1.281$ and $\fj=0.828$.
The variation of information with respect to non-separator voxels is $\vins = 0.023$ with $\fcns = 0.022$ and $\fjns = 0.001$ for the multi-separator computed by \Cref{algo:gss} and $\vins = 0.671$ with $\fcns = 0.283$ and $\fjns = 0.388$ for multi-separator computed by the watershed algorithm.
The low $\vins$ value for the multi-separator computed by \Cref{algo:gss} indicates that there are very few false cuts and false joins of foam cells.
The higher $\viws$ value is a consequence of individual voxels wrongly identified as separator/non-separator.
In contrast, the multi-separator output by the watershed algorithm suffers from false cuts (gray segments in Row 3 Column 4 of \Cref{fig:cell-example}) and false joins (segments in bottom center and top right in Row 3 Column 4 of \Cref{fig:cell-example}) of foam cells as well as voxels that are wrongly identified as separator/non-separator.

The greedy separator shrinking algorithm (\Cref{algo:gss}) has a bias toward removing nodes from the separator that are adjacent to nodes that are not in the separator, instead of removing nodes from the separator whose neighbors are also part of the separator:
By the greedy paradigm, the node whose removal from the separator decreases the total cost maximally is removed.
If all neighbors of a node are in the separator, then, removing that node from the separator leaves all interactions adjacent to that node separated, i.e. the total cost is only changed by the cost of that node.
Otherwise, the total cost is changed by the cost of that node \emph{and} the cost of all interactions that are no longer separated.
Such a bias is not unique to \Cref{algo:gss} for the multi-separator problem but it also exists for greedy contraction based algorithms for the (lifted) multicut problem, like GAEC \citep{keuper2015efficient}.

In the context of segmenting foam cells, this bias can result in the algorithm finding poor local optima for the following reason:
Once a node is removed from the separator, the described bias leads to neighboring nodes being more likely to be removed next. 
The result can be that the component corresponding to one foam cell is growing before a component corresponding to a different foam is started. 
Consequently, it can happen that the component corresponding to one foam cell grows beyond its membrane because there does not yet exists a component corresponding to a neighboring foam cell (the existence of a component corresponding to a neighboring foam cell would result in a penalty for removing a node from the separator between the two components as many repulsive interactions between these components would no longer be separated).
This ultimately leads to false joins in the computed separator.
In our experiments, this behavior manifests only for low amounts of noise ($t \leq 0.2$) as can be seen from \Cref{fig:cell-inaccuracy-by-noise} (bottom right).
Moreover, it can be overcome easily, by starting \Cref{algo:gss} not with all but with most nodes in the separator, e.g.~those 90\% of the nodes with the lowest cost.
For consistency, we do not implement such heuristics and only report results for \Cref{algo:gss} initialized with all nodes in the separator.

\section{Conclusion}\label{sec:conclusion}
We have introduced the multi-separator problem, a combinatorial optimization problem with applications to the task of image segmentation. 
While the general problem is \textsc{np}-hard and even hard to approximate, we have identified two special cases of the objective function for which the problem can be solved efficiently.
For the general \textsc{np}-hard case, we have defined two efficient local search algorithms.
In experiments with synthetic volume images of simulated filaments and foam cells, we have seen for all but extreme amounts of noise, that the multi-separator problem in connection with \Cref{algo:gsg,algo:gss} can result in significantly more accurate segmentations than a watershed algorithm, with a practical runtime.
For moderate noise, these segmentations are not sensitive to small biases in objective function of the multi-separator problem, as long as the bias is positive, for filaments, and negative, for foam cells.
Given the popularity of the watershed algorithm for segmenting foam cells, and given that, so far, fundamentally different models and algorithms are used for segmenting filaments, (cf.~\Cref{sec:related-work}), the multi-separator problem contributes to a more symmetric treatment of these extreme cases of the image segmentation task and holds promise for volume images in which both structures are present, e.g., for volume images of nervous systems acquired in the field of connectomics \citep{lee2021learning}.

The results of this work motivate further studies of the multi-separator problem both theoretically and in connection with practical applications.
Theoretical questions concern, e.g., the complexity and approximability of the problem for special classes of graphs.
Specifically, planar graphs are of interest in the context of segmenting 2-dimensional images.
Also of interest are exact algorithms for the multi-separator problem.
To this end, the problem can be stated in the form of an integer linear program which then can be solved by a branch and cut algorithm.
Studying the polytope associated with the set of feasible solutions of the integer linear program will be of interest to improve exact and approximate algorithms.
Another open problem asks for an efficient implementation of a local search algorithm that adds and removes nodes from the separator.
To realize such an algorithm, findings on the dynamic connectivity problem \citep{holm2001poly} can be considered.

\bibliography{references}

\begin{thebibliography}{60}
\providecommand{\natexlab}[1]{#1}
\providecommand{\url}[1]{\texttt{#1}}
\expandafter\ifx\csname urlstyle\endcsname\relax
  \providecommand{\doi}[1]{doi: #1}\else
  \providecommand{\doi}{doi: \begingroup \urlstyle{rm}\Url}\fi

\bibitem[Alush and Goldberger(2016)]{alush-2016}
Amir Alush and Jacob Goldberger.
\newblock Hierarchical image segmentation using correlation clustering.
\newblock \emph{{IEEE} Trans. Neural Networks Learn. Syst.}, 27\penalty0
  (6):\penalty0 1358--1367, 2016.
\newblock \doi{10.1109/TNNLS.2015.2505181}.

\bibitem[Andres et~al.(2011)Andres, Kappes, Beier, K{\"o}the, and
  Hamprecht]{andres2011probabilistic}
Bjoern Andres, J{\"o}rg~Hendrik Kappes, Thorsten Beier, Ullrich K{\"o}the, and
  Fred~A Hamprecht.
\newblock Probabilistic image segmentation with closedness constraints.
\newblock In \emph{ICCV}, 2011.
\newblock \doi{10.1109/ICCV.2011.6126550}.

\bibitem[Andres et~al.(2023)Andres, Di~Gregorio, Irmai, and
  Lange]{andres2023polyhedral}
Bjoern Andres, Silvia Di~Gregorio, Jannik Irmai, and Jan-Hendrik Lange.
\newblock A polyhedral study of lifted multicuts.
\newblock \emph{Discrete Optimization}, 47:\penalty0 100757, 2023.
\newblock \doi{10.1016/j.disopt.2022.100757}.

\bibitem[Arabie and Boorman(1973)]{arabie-1973}
Phipps Arabie and Scott~A. Boorman.
\newblock Multidimensional scaling of measures of distance between partitions.
\newblock \emph{Journal of Mathematical Psychology}, 10\penalty0 (2):\penalty0
  148--203, 1973.
\newblock \doi{10.1016/0022-2496(73)90012-6}.

\bibitem[Bachrach et~al.(2013)Bachrach, Kohli, Kolmogorov, and
  Zadimoghaddam]{bachrach2013optimal}
Yoram Bachrach, Pushmeet Kohli, Vladimir Kolmogorov, and Morteza Zadimoghaddam.
\newblock Optimal coalition structure generation in cooperative graph games.
\newblock In \emph{AAAI}, 2013.
\newblock \doi{10.1609/aaai.v27i1.8653}.

\bibitem[Balas and Souza(2005)]{balas2005vertex}
Egon Balas and Cid Carvalho~de Souza.
\newblock The vertex separator problem: a polyhedral investigation.
\newblock \emph{Mathematical Programming}, 103\penalty0 (3):\penalty0 583--608,
  2005.
\newblock \doi{10.1007/s10107-005-0574-7}.

\bibitem[Bansal et~al.(2004)Bansal, Blum, and Chawla]{bansal2004correlation}
Nikhil Bansal, Avrim Blum, and Shuchi Chawla.
\newblock Correlation clustering.
\newblock \emph{Machine learning}, 56:\penalty0 89--113, 2004.
\newblock \doi{10.1023/B:MACH.0000033116.57574.95}.

\bibitem[Barahona(1982)]{barahona1982computational}
Francisco Barahona.
\newblock On the computational complexity of {I}sing spin glass models.
\newblock \emph{Journal of Physics A: Mathematical and General}, 15\penalty0
  (10):\penalty0 3241, 1982.
\newblock \doi{10.1088/0305-4470/15/10/028}.

\bibitem[Beier et~al.(2014)Beier, Kroeger, Kappes, K{\"o}the, and
  Hamprecht]{beier2014cut}
Thorsten Beier, Thorben Kroeger, J{\"o}rg~Hendrik Kappes, Ullrich K{\"o}the,
  and Fred~A. Hamprecht.
\newblock Cut, glue \& cut: A fast, approximate solver for multicut
  partitioning.
\newblock In \emph{CVPR}, 2014.
\newblock \doi{10.1109/CVPR.2014.17}.

\bibitem[Beier et~al.(2015)Beier, Hamprecht, and Kappes]{beier-2015-fusion}
Thorsten Beier, Fred~A. Hamprecht, and J{\"o}rg~Hendrik Kappes.
\newblock Fusion moves for correlation clustering.
\newblock In \emph{CVPR}, 2015.
\newblock \doi{10.1109/CVPR.2015.7298973}.

\bibitem[Beier et~al.(2017)Beier, Pape, Rahaman, Prange, Berg, Bock, Cardona,
  Knott, Plaza, Scheffer, et~al.]{beier2017multicut}
Thorsten Beier, Constantin Pape, Nasim Rahaman, Timo Prange, Stuart Berg,
  Davi~D. Bock, Albert Cardona, Graham~W. Knott, Stephen~M. Plaza, Louis~K.
  Scheffer, et~al.
\newblock Multicut brings automated neurite segmentation closer to human
  performance.
\newblock \emph{Nature methods}, 14\penalty0 (2):\penalty0 101--102, 2017.
\newblock \doi{10.1038/nmeth.4151}.

\bibitem[Berry et~al.(2000)Berry, Bordat, and Cogis]{berry2000generating}
Anne Berry, Jean-Paul Bordat, and Olivier Cogis.
\newblock Generating all the minimal separators of a graph.
\newblock \emph{International Journal of Foundations of Computer Science},
  11\penalty0 (03):\penalty0 397--403, 2000.
\newblock \doi{10.1142/S0129054100000211}.

\bibitem[Charikar et~al.(2005)Charikar, Guruswami, and
  Wirth]{charikar2005clustering}
Moses Charikar, Venkatesan Guruswami, and Anthony Wirth.
\newblock Clustering with qualitative information.
\newblock \emph{Journal of Computer and System Sciences}, 71\penalty0
  (3):\penalty0 360--383, 2005.
\newblock \doi{10.1016/j.jcss.2004.10.012}.

\bibitem[Chopra and Rao(1993)]{chopra1993partition}
Sunil Chopra and Mendu~R. Rao.
\newblock The partition problem.
\newblock \emph{Mathematical programming}, 59\penalty0 (1-3):\penalty0 87--115,
  1993.
\newblock \doi{10.1007/BF01581239}.

\bibitem[Cornaz et~al.(2018)Cornaz, Furini, Lacroix, Malaguti, Mahjoub, and
  Martin]{cornaz2019vertex}
Denis Cornaz, Fabio Furini, Mathieu Lacroix, Enrico Malaguti, A.~Ridha Mahjoub,
  and S{\'e}bastien Martin.
\newblock The vertex k-cut problem.
\newblock \emph{Discrete Optimization}, 31:\penalty0 8--28, 2018.
\newblock \doi{10.1016/j.disopt.2018.07.003}.

\bibitem[Cornaz et~al.(2019)Cornaz, Magnouche, Mahjoub, and
  Martin]{cornaz2019multi}
Denis Cornaz, Youcef Magnouche, Ali~Ridha Mahjoub, and S{\'e}bastien Martin.
\newblock The multi-terminal vertex separator problem: Polyhedral analysis and
  branch-and-cut.
\newblock \emph{Discrete Applied Mathematics}, 256:\penalty0 11--37, 2019.
\newblock \doi{10.1016/j.dam.2018.10.005}.

\bibitem[Demaine et~al.(2006)Demaine, Emanuel, Fiat, and
  Immorlica]{demaine2006correlation}
Erik~D. Demaine, Dotan Emanuel, Amos Fiat, and Nicole Immorlica.
\newblock Correlation clustering in general weighted graphs.
\newblock \emph{Theoretical Computer Science}, 361\penalty0 (2-3):\penalty0
  172--187, 2006.
\newblock \doi{10.1016/j.tcs.2006.05.008}.

\bibitem[Didi~Biha and Meurs(2011)]{didi2011exact}
Mohamed Didi~Biha and Marie-Jean Meurs.
\newblock An exact algorithm for solving the vertex separator problem.
\newblock \emph{Journal of Global Optimization}, 49:\penalty0 425--434, 2011.

\bibitem[Escalante(1972)]{escalante1972schnittverbande}
Fernando Escalante.
\newblock Schnittverb{\"a}nde in {G}raphen.
\newblock \emph{Abh.Math.Semin.Univ.Hambg.}, 38:\penalty0 199--220, 1972.
\newblock \doi{10.1007/BF02996932}.

\bibitem[Fukuyama(2006)]{fukuyama2006np}
Junichiro Fukuyama.
\newblock {NP}-completeness of the planar separator problems.
\newblock \emph{J. Graph Algorithms Appl.}, 10\penalty0 (2):\penalty0 317--328,
  2006.
\newblock \doi{10.7155/jgaa.00130}.

\bibitem[Furini et~al.(2020)Furini, Ljubi{\'c}, Malaguti, and
  Paronuzzi]{furini2020integer}
Fabio Furini, Ivana Ljubi{\'c}, Enrico Malaguti, and Paolo Paronuzzi.
\newblock On integer and bilevel formulations for the k-vertex cut problem.
\newblock \emph{Mathematical Programming Computation}, 12:\penalty0 133--164,
  2020.
\newblock \doi{10.1007/s12532-019-00167-1}.

\bibitem[Garg et~al.(1994)Garg, Vazirani, and Yannakakis]{garg1994multiway}
Naveen Garg, Vijay~V. Vazirani, and Mihalis Yannakakis.
\newblock Multiway cuts in directed and node weighted graphs.
\newblock In \emph{ICALP}, 1994.
\newblock \doi{10.1007/3-540-58201-0_92}.

\bibitem[Garg et~al.(2004)Garg, Vazirani, and Yannakakis]{garg2004multiway}
Naveen Garg, Vijay~V. Vazirani, and Mihalis Yannakakis.
\newblock Multiway cuts in node weighted graphs.
\newblock \emph{Journal of Algorithms}, 50\penalty0 (1):\penalty0 49--61, 2004.
\newblock \doi{10.1016/S0196-6774(03)00111-1}.

\bibitem[Gr{\"o}tschel et~al.(1981)Gr{\"o}tschel, Lov{\'a}sz, and
  Schrijver]{grotschel1981ellipsoid}
Martin Gr{\"o}tschel, L{\'a}szl{\'o} Lov{\'a}sz, and Alexander Schrijver.
\newblock The ellipsoid method and its consequences in combinatorial
  optimization.
\newblock \emph{Combinatorica}, 1\penalty0 (2):\penalty0 169--197, 1981.
\newblock \doi{10.1007/BF02579273}.

\bibitem[Holm et~al.(2001)Holm, De~Lichtenberg, and Thorup]{holm2001poly}
Jacob Holm, Kristian De~Lichtenberg, and Mikkel Thorup.
\newblock Poly-logarithmic deterministic fully-dynamic algorithms for
  connectivity, minimum spanning tree, 2-edge, and biconnectivity.
\newblock \emph{Journal of the ACM (JACM)}, 48\penalty0 (4):\penalty0 723--760,
  2001.
\newblock \doi{10.1145/502090.502095}.

\bibitem[Hopcroft and Tarjan(1973)]{hopcroft1973algorithm}
John Hopcroft and Robert Tarjan.
\newblock Algorithm 447: efficient algorithms for graph manipulation.
\newblock \emph{Communications of the ACM}, 16\penalty0 (6):\penalty0 372--378,
  1973.
\newblock \doi{10.1145/362248.362272}.

\bibitem[Hor{\v{n}}{\'a}kov{\'a} et~al.(2017)Hor{\v{n}}{\'a}kov{\'a}, Lange,
  and Andres]{horvnakova2017analysis}
Andrea Hor{\v{n}}{\'a}kov{\'a}, Jan-Hendrik Lange, and Bjoern Andres.
\newblock Analysis and optimization of graph decompositions by lifted
  multicuts.
\newblock In \emph{ICML}, 2017.
\newblock URL \url{https://proceedings.mlr.press/v70/hornakova17a.html}.

\bibitem[Horn{\'{a}}kov{\'{a}} et~al.(2020)Horn{\'{a}}kov{\'{a}}, Henschel,
  Rosenhahn, and Swoboda]{hornakova-2020}
Andrea Horn{\'{a}}kov{\'{a}}, Roberto Henschel, Bodo Rosenhahn, and Paul
  Swoboda.
\newblock Lifted disjoint paths with application in multiple object tracking.
\newblock In \emph{ICML}, 2020.
\newblock URL \url{http://proceedings.mlr.press/v119/hornakova20a.html}.

\bibitem[Kappes et~al.(2011)Kappes, Speth, Andres, Reinelt, and
  Schn{\"o}rr]{kappes-2011-globally}
J{\"o}rg~Hendrik Kappes, Markus Speth, Bjoern Andres, Gerhard Reinelt, and
  Christoph Schn{\"o}rr.
\newblock Globally optimal image partitioning by multicuts.
\newblock In \emph{EMMCVPR}, 2011.
\newblock \doi{10.1007/978-3-642-23094-3_3}.

\bibitem[Kappes et~al.(2016{\natexlab{a}})Kappes, Speth, Reinelt, and
  Schn{\"{o}}rr]{reinelt-2016}
J{\"{o}}rg~Hendrik Kappes, Markus Speth, Gerhard Reinelt, and Christoph
  Schn{\"{o}}rr.
\newblock Higher-order segmentation via multicuts.
\newblock \emph{Computer Vision and Image Understanding}, 143:\penalty0
  104--119, 2016{\natexlab{a}}.
\newblock \doi{10.1016/j.cviu.2015.11.005}.

\bibitem[Kappes et~al.(2016{\natexlab{b}})Kappes, Swoboda, Savchynskyy, Hazan,
  and Schn{\"{o}}rr]{kappes-2016}
J{\"{o}}rg~Hendrik Kappes, Paul Swoboda, Bogdan Savchynskyy, Tamir Hazan, and
  Christoph Schn{\"{o}}rr.
\newblock Multicuts and perturb {\&} {MAP} for probabilistic graph clustering.
\newblock \emph{J. Math. Imaging Vis.}, 56\penalty0 (2):\penalty0 221--237,
  2016{\natexlab{b}}.
\newblock \doi{10.1007/s10851-016-0659-3}.

\bibitem[Kardoost and Keuper(2018)]{kardoost-2018}
Amirhossein Kardoost and Margret Keuper.
\newblock Solving minimum cost lifted multicut problems by node agglomeration.
\newblock In \emph{ACCV}, 2018.
\newblock \doi{10.1007/978-3-030-20870-7_5}.

\bibitem[Kardoost and Keuper(2021)]{kardoost-2021}
Amirhossein Kardoost and Margret Keuper.
\newblock Uncertainty in minimum cost multicuts for image and motion
  segmentation.
\newblock In \emph{UAI}, 2021.
\newblock URL \url{https://proceedings.mlr.press/v161/kardoost21a.html}.

\bibitem[Karp(1972)]{karp1972reducibility}
Richard~Manning Karp.
\newblock Reducibility among combinatorial problems.
\newblock In \emph{Complexity of computer computations}, 1972.
\newblock \doi{10.1007/978-1-4684-2001-2_9}.

\bibitem[Keuper(2017)]{keuper-2017}
Margret Keuper.
\newblock Higher-order minimum cost lifted multicuts for motion segmentation.
\newblock In \emph{ICCV}, 2017.
\newblock \doi{10.1109/ICCV.2017.455}.

\bibitem[Keuper et~al.(2015)Keuper, Levinkov, Bonneel, Lavou{\'e}, Brox, and
  Andres]{keuper2015efficient}
Margret Keuper, Evgeny Levinkov, Nicolas Bonneel, Guillaume Lavou{\'e}, Thomas
  Brox, and Bjoern Andres.
\newblock Efficient decomposition of image and mesh graphs by lifted multicuts.
\newblock In \emph{ICCV}, 2015.
\newblock \doi{10.1109/ICCV.2015.204}.

\bibitem[Kim et~al.(2014)Kim, Yoo, Nowozin, and Kohli]{kim-2014}
Sungwoong Kim, Chang~Dong Yoo, Sebastian Nowozin, and Pushmeet Kohli.
\newblock Image segmentation using higher-order correlation clustering.
\newblock \emph{{IEEE} Trans. Pattern Anal. Mach. Intell.}, 36\penalty0
  (9):\penalty0 1761--1774, 2014.
\newblock \doi{10.1109/TPAMI.2014.2303095}.

\bibitem[Kirillov et~al.(2017)Kirillov, Levinkov, Andres, Savchynskyy, and
  Rother]{kirillov2017instancecut}
Alexander Kirillov, Evgeny Levinkov, Bjoern Andres, Bogdan Savchynskyy, and
  Carsten Rother.
\newblock Instance{C}ut: from edges to instances with multicut.
\newblock In \emph{CVPR}, 2017.
\newblock \doi{10.1109/CVPR.2017.774}.

\bibitem[Klein et~al.(2023)Klein, Mathieu, and Zhou]{klein2023correlation}
Philip~N. Klein, Claire Mathieu, and Hang Zhou.
\newblock Correlation clustering and two-edge-connected augmentation for planar
  graphs.
\newblock \emph{Algorithmica}, pages 1--34, 2023.
\newblock \doi{10.1007/s00453-023-01128-w}.

\bibitem[Lee et~al.(2021)Lee, Lu, Luther, and Seung]{lee2021learning}
Kisuk Lee, Ran Lu, Kyle Luther, and Hyunjune~Sebastian Seung.
\newblock Learning and segmenting dense voxel embeddings for 3d neuron
  reconstruction.
\newblock \emph{IEEE Transactions on Medical Imaging}, 40\penalty0
  (12):\penalty0 3801--3811, 2021.
\newblock \doi{10.1109/TMI.2021.3097826}.

\bibitem[Lipton and Tarjan(1979)]{lipton1979separator}
Richard~Jay Lipton and Robert~Endre Tarjan.
\newblock A separator theorem for planar graphs.
\newblock \emph{SIAM Journal on Applied Mathematics}, 36\penalty0 (2):\penalty0
  177--189, 1979.
\newblock \doi{10.1137/0136016}.

\bibitem[Magnouche et~al.(2021)Magnouche, Mahjoub, and
  Martin]{magnouche2021multi}
Youcef Magnouche, Ali~Ridha Mahjoub, and S{\'e}bastien Martin.
\newblock The multi-terminal vertex separator problem:
  Branch-and-cut-and-price.
\newblock \emph{Discrete Applied Mathematics}, 290:\penalty0 86--111, 2021.
\newblock \doi{10.1016/j.dam.2020.06.021}.

\bibitem[Meil{\u{a}}(2007)]{meila-2007}
Marina Meil{\u{a}}.
\newblock Comparing clusterings—an information based distance.
\newblock \emph{Journal of Multivariate Analysis}, 98\penalty0 (5):\penalty0
  873--895, 2007.
\newblock \doi{10.1016/j.jmva.2006.11.013}.

\bibitem[Menger(1927)]{menger1927allgemeinen}
Karl Menger.
\newblock Zur allgemeinen {K}urventheorie.
\newblock \emph{Fundamenta Mathematicae}, 10\penalty0 (1):\penalty0 96--115,
  1927.
\newblock URL \url{http://eudml.org/doc/211191}.

\bibitem[Meyer(1991)]{meyer-1991}
Fernand Meyer.
\newblock Un algorithme optimal pour la ligne de partage des eaux.
\newblock In \emph{8e congrès de reconnaissance des formes et intelligence
  artificielle}, 1991.

\bibitem[Minaee et~al.(2022)Minaee, Boykov, Porikli, Plaza, Kehtarnavaz, and
  Terzopoulos]{minaee2022image}
Shervin Minaee, Yuri Boykov, Fatih Porikli, Antonio Plaza, Nasser Kehtarnavaz,
  and Demetri Terzopoulos.
\newblock Image segmentation using deep learning: A survey.
\newblock \emph{IEEE Transactions on Pattern Analysis and Machine
  Intelligence}, 44\penalty0 (7):\penalty0 3523--3542, 2022.
\newblock \doi{10.1109/TPAMI.2021.3059968}.

\bibitem[Moss and Rabani(2007)]{moss-2007}
Anna Moss and Yuval Rabani.
\newblock Approximation algorithms for constrained node weighted {S}teiner tree
  problems.
\newblock \emph{SIAM Journal on Computing}, 37\penalty0 (2):\penalty0 460--481,
  2007.
\newblock \doi{10.1137/S0097539702420474}.

\bibitem[Nowozin and Lampert(2010)]{nowozin-2010}
Sebastian Nowozin and Christoph~H. Lampert.
\newblock Global interactions in random field models: {A} potential function
  ensuring connectedness.
\newblock \emph{{SIAM} J. Imaging Sci.}, 3\penalty0 (4):\penalty0 1048--1074,
  2010.
\newblock \doi{10.1137/090752614}.

\bibitem[Rempfler et~al.(2015)Rempfler, Schneider, Ielacqua, Xiao, Stock,
  Klohs, Sz\'ekely, Andres, and Menze]{rempfler-2015}
Markus Rempfler, Matthias Schneider, Giovanna~D. Ielacqua, Xianghui Xiao,
  Stuart~R. Stock, Jan Klohs, G\'bor Sz\'ekely, Bjoern Andres, and Bjoern~H.
  Menze.
\newblock Reconstructing cerebrovascular networks under local physiological
  constraints by integer programming.
\newblock \emph{Medical Image Analysis}, 25\penalty0 (1):\penalty0 86--94,
  2015.
\newblock \doi{10.1016/j.media.2015.03.008}.

\bibitem[Roerdink and Meijster(2000)]{roerdink2000watershed}
Jos~B.T.M. Roerdink and Arnold Meijster.
\newblock The watershed transform: Definitions, algorithms and parallelization
  strategies.
\newblock \emph{Fundamenta informaticae}, 41\penalty0 (1-2):\penalty0 187--228,
  2000.
\newblock \doi{10.3233/FI-2000-411207}.

\bibitem[Shit et~al.(2022)Shit, Koner, Wittmann, Paetzold, Ezhov, Li, Pan,
  Sharifzadeh, Kaissis, Tresp, and Menze]{shit-2022}
Suprosanna Shit, Rajat Koner, Bastian Wittmann, Johannes~C. Paetzold, Ivan
  Ezhov, Hongwei Li, Jiazhen Pan, Sahand Sharifzadeh, Georgios Kaissis, Volker
  Tresp, and Bjoern~H. Menze.
\newblock Relationformer: A unified framework for image-to-graph generation.
\newblock In \emph{ECCV}, 2022.
\newblock \doi{10.1007/978-3-031-19836-6_24}.

\bibitem[Souza and Balas(2005)]{souza2005vertex}
Cid Carvalho~de Souza and Egon Balas.
\newblock The vertex separator problem: algorithms and computations.
\newblock \emph{Mathematical Programming}, 103\penalty0 (3):\penalty0 609--631,
  2005.
\newblock \doi{10.1007/s10107-005-0573-8}.

\bibitem[Tang et~al.(2017)Tang, Andriluka, Andres, and
  Schiele]{tang-2017-multiple}
Siyu Tang, Mykhaylo Andriluka, Bjoern Andres, and Bernt Schiele.
\newblock Multiple people tracking by lifted multicut and person
  re-identification.
\newblock In \emph{CVPR}, 2017.
\newblock \doi{10.1109/CVPR.2017.394}.

\bibitem[T{\"{u}}retken et~al.(2016)T{\"{u}}retken, Benmansour, Andres,
  Glowacki, Pfister, and Fua]{turetken2016reconstructing}
Engin T{\"{u}}retken, Fethallah Benmansour, Bjoern Andres, Przemyslaw Glowacki,
  Hanspeter Pfister, and Pascal Fua.
\newblock Reconstructing curvilinear networks using path classifiers and
  integer programming.
\newblock \emph{{IEEE} Trans. Pattern Anal. Mach. Intell.}, 38\penalty0
  (12):\penalty0 2515--2530, 2016.
\newblock \doi{10.1109/TPAMI.2016.2519025}.

\bibitem[van~der Walt et~al.(2014)van~der Walt, Sch{\"o}nberger,
  {Nunez-Iglesias}, Boulogne, Warner, Yager, Gouillart, Yu, and the
  scikit-image contributors]{scikit-image}
St{\'e}fan van~der Walt, Johannes~L. Sch{\"o}nberger, Juan {Nunez-Iglesias},
  Fran\c{c}ois Boulogne, Joshua~D. Warner, Neil Yager, Emmanuelle Gouillart,
  Tony Yu, and the scikit-image contributors.
\newblock scikit-image: image processing in {P}ython.
\newblock \emph{PeerJ}, 2:\penalty0 e453, 6 2014.
\newblock \doi{10.7717/peerj.453}.

\bibitem[Vincent and Soille(1991)]{soille-1991}
Luc Vincent and Pierre Soille.
\newblock Watersheds in digital spaces: an efficient algorithm based on
  immersion simulations.
\newblock \emph{{IEEE} Trans. Pattern Anal. Mach. Intell.}, 13\penalty0
  (6):\penalty0 583--598, 1991.
\newblock \doi{10.1109/34.87344}.

\bibitem[Voice et~al.(2012)Voice, Polukarov, and Jennings]{voice2012coalition}
Thomas Voice, Maria Polukarov, and Nicholas~R. Jennings.
\newblock Coalition structure generation over graphs.
\newblock \emph{Journal of Artificial Intelligence Research}, 45:\penalty0
  165--196, 2012.
\newblock \doi{10.1613/jair.3715}.

\bibitem[Wolf et~al.(2020)Wolf, Bailoni, Pape, Rahaman, Kreshuk, K{\"o}the, and
  Hamprecht]{wolf2020mutex}
Steffen Wolf, Alberto Bailoni, Constantin Pape, Nasim Rahaman, Anna Kreshuk,
  Ullrich K{\"o}the, and Fred~A. Hamprecht.
\newblock The mutex watershed and its objective: Efficient, parameter-free
  graph partitioning.
\newblock \emph{{IEEE} Trans. Pattern Anal. Mach. Intell.}, 43\penalty0
  (10):\penalty0 3724--3738, 2020.
\newblock \doi{10.1109/TPAMI.2020.2980827}.

\bibitem[Yarkony et~al.(2012)Yarkony, Ihler, and Fowlkes]{yarkony-2012}
Julian Yarkony, Alexander Ihler, and Charless~C. Fowlkes.
\newblock Fast planar correlation clustering for image segmentation.
\newblock In \emph{ECCV}, 2012.
\newblock \doi{10.1007/978-3-642-33783-3_41}.

\bibitem[Zhang et~al.(2014)Zhang, Yarkony, and Hamprecht]{zhang-2014}
Chong Zhang, Julian Yarkony, and Fred~A. Hamprecht.
\newblock Cell detection and segmentation using correlation clustering.
\newblock In \emph{MICCAI}, 2014.
\newblock \doi{10.1007/978-3-319-10404-1_2}.

\end{thebibliography}
\bibliographystyle{plainnat}

\appendix
\section{Appendix}
\subsection{Volume image synthesis (details)}\label{appendix:image-synthesis}

In order to synthesize binary volume images of filaments, we draw $n$ cubic splines randomly so as to intersect the unit cube and such that the minimal distance between any two splines is at least $d_{\text{min}}$. 
Regarding the parameters of the splines, regarding the distributions from which these parameters are drawn, and regarding the rejection of splines that violate the constraints, we refer the reader to the supplementary \cplusplus{} code.
Then, considering the canonical embedding of the 3-dimensional grid graph of $m \times m \times m$ voxels in the unit cube, we label voxels $0$ if and only if their Euclidean distance in the unit cube to the nearest spline is at most $r$.
In our experiments, $m=64$, $d_\text{min} = \tfrac{10}{m}$ and $r = \tfrac{0.75}{m}$.
Thus, the center lines of the filaments are at least $10$ voxels apart, and the filaments have a width of $1.5$ voxels.

In order to synthesize binary volume images of foam cells, we draw $n$ pairwise non-neighboring voxels uniformly at random from the $m \times m \times m$ voxel grid graph. 
These voxels are labeled $1, \dots, n$, respectively. 
All other voxels are labeled $0$. 
Next, we repeat the following randomized region growing procedure until termination: 
We draw a voxel $v$ uniformly at random from the set of voxels $w$ that hold three properties: 
(i) $w$ is labeled $0$, (ii) there is a non-zero label among the labels of the neighbors $w$, and (iii) this non-zero label is unique, called $i_w$.
If no such voxel exists, the procedure terminates.
Otherwise, $v$ is labeled $i_v$.
This results in a voxel labeling such that, for every non-zero label $i$, the set of voxels labeled $i$ induces a component of the voxel grid graph, and these components are separated by the set of voxels labeled $0$.
Next, we erode the components of voxels with a non-zero label with a spherical kernel of radius $\tfrac{d_{\text{min}}}{2} + r$.
In our experiments, $m=64$, $d_\text{min} = 8$ and $r = 0.75$.
If this erosion disconnects voxels with the same non-zero label, we label $0$ all except one maximal component of such voxels.
Next, we again grow components of voxels with a non-zero label, now deterministically by breath-first search, until all voxels labeled $0$ are adjacent to at last two voxels with distinct non-zero labels.
Then, we label $0$ also all voxels with a Euclidean distance of at most $r$ to any voxel labeled $0$.
This ensures that all components of voxels with a non-zero label (to-become foam cells) are separated by a $2\,r$-wide set of voxels labeled $0$ (to-become cell membranes).
Finally, we obtain a binary image by labeling $1$ precisely those voxels whose integer label is $0$.

\begin{figure}
    \centering
    \begin{tikzpicture}[scale=1]
    \def \wmax {0.9};
    \def \r {2};

    \begin{axis}[
        width=8cm, height=4cm,
        axis lines=left, xlabel=$d$, ylabel=$w(d)$, ymax=1.1, ymin=0, 
        xtick={\r}, xticklabels={$r$},
        ytick={0.5,\wmax, 1}, yticklabels={$\tfrac{1}{2}$, $w_{\text{max}}$, $1$}]
        \addplot[domain=0:5] (x,{1/(1 + (\wmax/(1-\wmax))^(x/\r-1))});
        \addplot[domain=0:\r, dashed] (x,0.5);
        \addplot[domain=0:0.5, dashed] (\r,x);
    \end{axis}
\end{tikzpicture}
    \caption{Depicted above its the weight $w$ as a function of the distance $d$.}
    \label{fig:weight-distribution}
\end{figure}
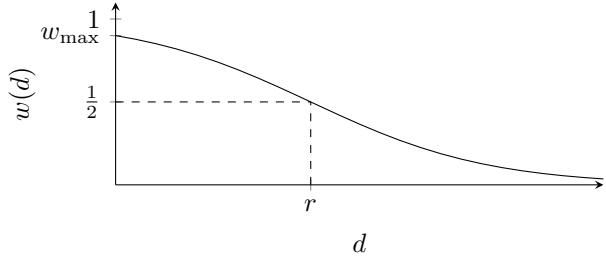

To obtain gray scale images from the binary images described above, we proceed as follows:
For images of filaments, we compute the distance from each voxel to the nearest spline. 
For images of foam cells, we compute the distance from each voxel to the nearest membrane, more specifically, to the nearest voxel labeled 0 during the synthesis of the binary image, at the end of the breadth-first search.
Let $d$ be this distance for a given voxel $v$.
We define the gray value of $v$ as $g = w(d) \, g_1 + (1-w(d)) \, g_2$, clipped to the interval $[0,1]$.
Here, $g_1$ and $g_2$ are samples from two normal distribution with means and standard deviations $\mu_1$, $\sigma_1$ and $\mu_2$, $\sigma_2$,
and $g$ is a weighted sum with the weight $w(d)$ depending on the distance $d$ according to the function written below that is depicted in \Cref{fig:weight-distribution}:
\[
    w(d) = \frac{1}{1 + \left(\frac{w_{\text{max}}}{1-w_{\text{max}}}\right)^{d/r - 1}}  
    \quad
    \text{with}
    \quad
    w_{\text{max}} \in (\tfrac{1}{2},1)
\]
This function decreases monotonically such that $w(0) = w_{\text{max}}$ and $w(r) = \tfrac{1}{2}$, and $w(d) \to 0$ for $d \to \infty$.
In our experiments, $w_{\text{max}} = 0.9$. 
For $\mu_1$, $\sigma_1$, $\mu_2$, and $\sigma_2$, we consider different numbers in order to vary the amount of noise, and thus, the difficulty of reconstructing the binary image from the gray scale image.
More specifically, we define $\mu_1^\text{start}$, $\mu_1^\text{end}$, $\sigma_1^\text{start}$, $\sigma_1^\text{end}$, $\mu_2^\text{start}$, $\mu_2^\text{end}$, $\sigma_2^\text{start}$, $\sigma_2^\text{end}$, and for the parameter $t \in [0, 1]$: $\mu_i = (1-t)\mu_i^\text{start} + t(\mu_i^\text{end})$ and $\sigma_i = (1-t)\sigma_i^\text{start} + t(\sigma_i^\text{end})$ for $i=1,2$.
% The parameter $t$ linearly varies $\mu_1$, $\sigma_1$, $\mu_2$, and $\sigma_2$ between their respective start and end values.
For filaments, 
$
    \mu_1^\text{start} = 0.3, 
    \mu_1^\text{end} = 0.38, 
    \mu_2^\text{start} = 0.7, 
    \mu_2^\text{end} = 0.62, 
    \sigma_1^\text{start} = \sigma_2^\text{start} = 0.05, 
    \sigma_1^\text{end} = \sigma_2^\text{end} = 0.1
$.
For foam cells, 
$
    \mu_1^\text{start} = 0.7, 
    \mu_1^\text{end} = 0.55, 
    \mu_2^\text{start} = 0.3, 
    \mu_2^\text{end} = 0.45, 
    \sigma_1^\text{start} = \sigma_2^\text{start} = 0.05, 
    \sigma_1^\text{end} = \sigma_2^\text{end} = 0.1
$.
These parameters are chosen such that voxels labeled 1 in the binary image will, in expectation, be brighter than voxels labeled 0 in the binary image.
Moreover, the parameters are chosen such that, for $t=0$, all algorithms we consider output near perfect reconstructions of the binary images, and, for $t=1$, all algorithms we consider output binary images with severe mistakes.

\subsection{Detailed numerical results for renderings}\label{appendix:detailed-numerical-results}

\Cref{tab:filament-example} and \Cref{tab:cell-example} report the numerical results corresponding to the examples that are depicted in \Cref{fig:filament-example} and \Cref{fig:cell-example}, respectively.

\begin{table}
    \centering
    \small
    \begin{tabular}{l l l l l l l l l l l}
        \toprule
        $t$ & Alg. & $\theta_\text{start}$ & $\theta_\text{end}$ & $b$ & $\viws$ & $\fc$ & $\fj$ & $\vins$ & $\fcns$ & $\fjns$ \\
        \midrule
        0.00 & WS & 0.580 & 0.580 & & 0.056 & 0.008 & 0.048 & 0.000 & 0.000 & 0.000 \\
        & MS & & & 0.06 & 0.054 & 0.007 & 0.046 & 0.000 & 0.000 & 0.000 \\
        \hline
        0.25 & WS & 0.580 & 0.585 & & 0.104 & 0.022 & 0.082 & 0.000 & 0.000 & 0.000 \\
        & MS & & & 0.09 & 0.079 & 0.013 & 0.065 & 0.000 & 0.000 & 0.000 \\
        \hline
        0.50 & WS & 0.580 & 0.580 & & 0.250 & 0.075 & 0.174 & 0.000 & 0.000 & 0.000 \\
        & MS & & & 0.08 & 0.144 & 0.059 & 0.084 & 0.000 & 0.000 & 0.000 \\
        \hline
        0.75 & WS & 0.570 & 0.570 & & 0.750 & 0.421 & 0.330 & 0.435 & 0.435 & 0.000 \\
        & MS & & & 0.01 & 0.406 & 0.317 & 0.089 & 0.000 & 0.000 & 0.000 \\
        \hline
        1.00 & WS & 0.555 & 0.565 & & 1.447 & 0.781 & 0.666 & 1.045 & 0.907 & 0.138 \\
        & MS & & & -0.11 & 1.345 & 1.274 & 0.071 & 0.749 & 0.749 & 0.000 \\
        \bottomrule
        \end{tabular}
    \caption{This table summarizes the numerical results corresponding to the renderings in \Cref{fig:filament-example} for \textbf{filaments}.
    For the watershed algorithm (WS) the thresholds $\theta_\text{start}$ and $\theta_\text{end}$, and for the multi-separator (MS) algorithm the biases are chosen so as to minimize the average $\viws$ across those images of the data set with the amount of noise $t$.}
    \label{tab:filament-example}
\end{table}

\begin{table}
    \centering
    \small
    \begin{tabular}{l l l l l l l l l l l}
        \toprule
        $t$ & Alg. & $\theta_\text{start}$ & $\theta_\text{end}$ & $b$ & $\viws$ & $\fc$ & $\fj$ & $\vins$ & $\fcns$ & $\fjns$ \\
        \midrule
        0.00 & WS & 0.380 & 0.470 & & 0.566 & 0.421 & 0.144 & 0.023 & 0.023 & 0.000 \\
        & MS & & & -0.10 & 0.480 & 0.354 & 0.126 & 0.000 & 0.000 & 0.000 \\
        \hline
        0.25 & WS & 0.390 & 0.470 & & 1.083 & 0.787 & 0.296 & 0.209 & 0.101 & 0.108 \\
        & MS & & & -0.08 & 0.769 & 0.535 & 0.234 & 0.005 & 0.005 & 0.000 \\
        \hline
        0.50 & WS & 0.120 & 0.480 & & 2.110 & 1.281 & 0.828 & 0.671 & 0.283 & 0.388 \\
        & MS & & & -0.03 & 1.400 & 0.970 & 0.430 & 0.023 & 0.022 & 0.001 \\
        \hline
        0.75 & WS & 0.120 & 0.470 & & 3.273 & 2.540 & 0.734 & 0.657 & 0.472 & 0.185 \\
        & MS & & & 0.06 & 2.552 & 1.565 & 0.987 & 0.085 & 0.080 & 0.005 \\
        \hline
        1.00 & WS & 0.430 & 0.600 & & 4.526 & 2.810 & 1.716 & 3.275 & 1.941 & 1.334 \\
        & MS & & & 0.17 & 4.769 & 2.504 & 2.265 & 1.586 & 0.674 & 0.912 \\
        \bottomrule
        \end{tabular}
    \caption{This table summarizes the numerical results corresponding to the renderings in \Cref{fig:cell-example} for \textbf{foam cells}.
    For the watershed algorithm (WS) the thresholds $\theta_\text{start}$ and $\theta_\text{end}$, and for the multi-separator (MS) algorithm the biases are chosen so as to minimize the average $\viws$ across those images of the data set with the amount of noise $t$.}
    \label{tab:cell-example}
\end{table}

\subsection{Analysis of parameters of watershed algorithm}\label{appendix:watershed-parameters}

\Cref{fig:ws-threshold-filament,fig:ws-threshold-cell} depict the inaccuracy of multi-separators computed by the watershed algorithm for volume images of filaments and foam cells, respectively.

\begin{figure}
    \centering
    \includegraphics[height=4.5cm]{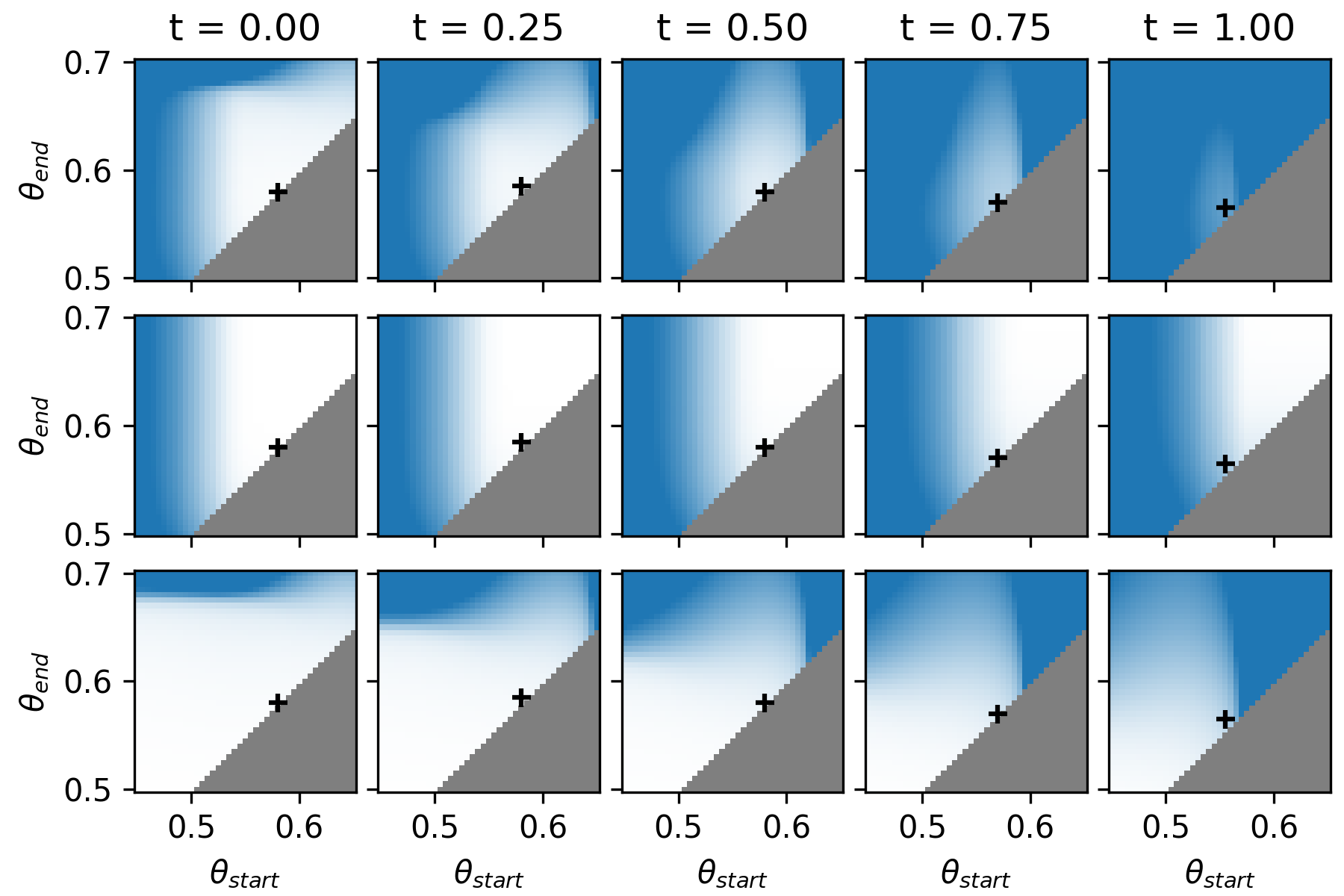}
    \caption{Depicted above is the inaccuracy (distance from truth) of multi-separators computed by the watershed algorithm (\Cref{section:watershed}) for volume images of \textbf{filaments} with five amounts of noise $t$ (columns) for all values of $\theta_\text{start}$ (x-axis) and $\theta_\text{end}$ (y-axis).
    The three rows depict the medians across those images of the data set with the amount of noise $t$ of the $\viws$ (top) and the conditional entropies due to false cuts ($\fc$, middle) and false joins ($\fj$, bottom).
    White indicates a value of $0$ and saturated blue indicates a value of $2$ or greater.
    Infeasible parameters (i.e. $\theta_\text{start} > \theta_\text{end}$) are depicted in gray.
    The + symbol indicates the parameters that archive the best median $\viws$.
    }
    \label{fig:ws-threshold-filament}
\end{figure}

\begin{figure}
    \centering
    \includegraphics[width=\textwidth]{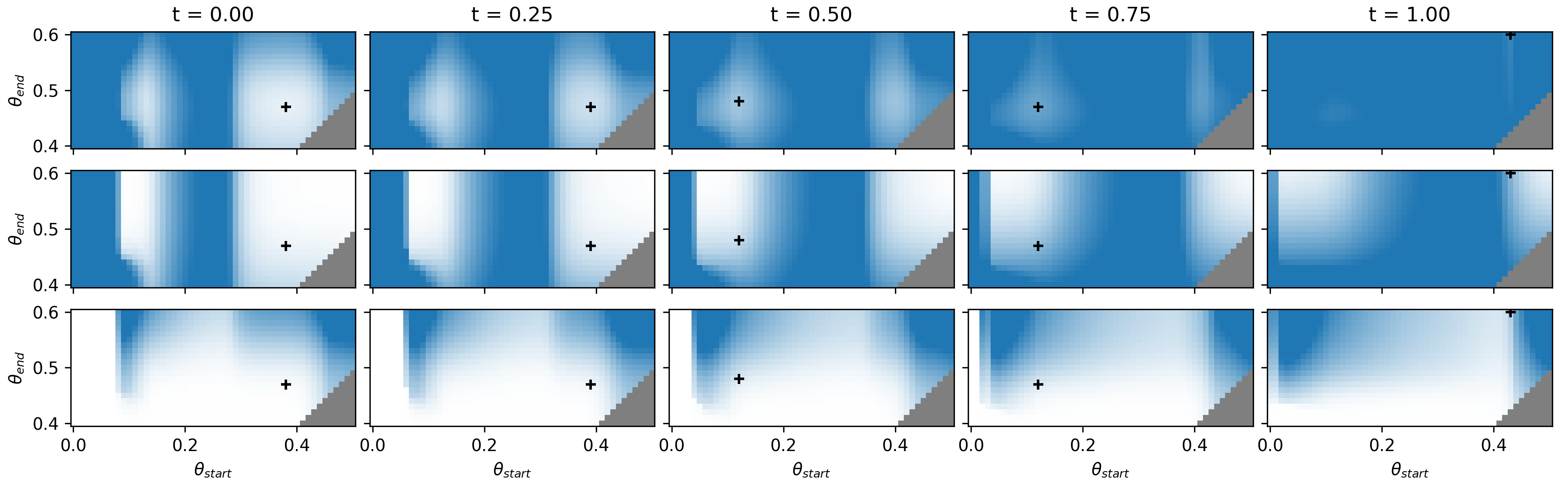}
    \caption{Depicted above is the inaccuracy (distance from truth) of multi-separators computed by the watershed algorithm (\Cref{section:watershed}) for volume images of \textbf{foam cells} completely analogous to \Cref{fig:ws-threshold-filament}.}
    \label{fig:ws-threshold-cell}
\end{figure}

\end{document}